\def\draft{0}

\documentclass[11pt]{article}

\usepackage{etex}
\usepackage{verbatim}
\usepackage{xspace,enumerate}
\usepackage[dvipsnames]{xcolor}
\usepackage[T1]{fontenc}
\usepackage[full]{textcomp}
\usepackage[american]{babel}
\usepackage{mathtools}
\usepackage{amsthm}

\usepackage{fullpage}

\usepackage{tcolorbox}

\usepackage{todonotes}

\usepackage{mathpazo}
\usepackage{multicol}
\usepackage{graphicx}

\usepackage[
letterpaper,
top=1in,
bottom=1in,
left=1in,
right=1in]{geometry}
\usepackage{newpxtext} %
\usepackage{textcomp} %
\usepackage[varg,bigdelims]{newpxmath}
\usepackage[scr=rsfso]{mathalfa}%
\usepackage{bm} %
\let\mathbb\varmathbb
\usepackage{microtype}
\usepackage[pagebackref,colorlinks=true,urlcolor=blue,linkcolor=blue,citecolor=OliveGreen]{hyperref}
\usepackage[capitalise,nameinlink]{cleveref}
\crefname{lemma}{Lemma}{Lemmas}
\crefname{fact}{Fact}{Facts}
\crefname{theorem}{Theorem}{Theorems}
\crefname{corollary}{Corollary}{Corollaries}
\crefname{claim}{Claim}{Claims}
\crefname{example}{Example}{Examples}
\crefname{algorithm}{Algorithm}{Algorithms}
\crefname{problem}{Problem}{Problems}
\crefname{definition}{Definition}{Definitions}
\crefname{exercise}{Exercise}{Exercises}
\usepackage{amsthm}

\newtheorem{theorem}{Theorem}[section]
\newtheorem*{theorem*}{Theorem}
\newtheorem{lemma}[theorem]{Lemma}
\newtheorem*{lemma*}{Lemma}
\newtheorem{fact}[theorem]{Fact}
\newtheorem*{fact*}{Fact}

\newtheorem*{proposition*}{Proposition}
\newtheorem{corollary}[theorem]{Corollary}
\newtheorem*{corollary*}{Corollary}

\newtheorem*{hypothesis*}{Hypothesis}
\newtheorem{conjecture}[theorem]{Conjecture}
\newtheorem*{conjecture*}{Conjecture}
\theoremstyle{definition}
\newtheorem{definition}[theorem]{Definition}
\newtheorem*{definition*}{Definition}

\newtheorem*{construction*}{Construction}

\newtheorem*{example*}{Example}

\newtheorem*{question*}{Question}
\newtheorem{algorithm}[theorem]{Algorithm}
\newtheorem*{algorithm*}{Algorithm}

\newtheorem*{assumption*}{Assumption}

\newtheorem*{problem*}{Problem}

\newtheorem*{openquestion*}{Open Question}
\theoremstyle{remark}

\newtheorem*{claim*}{Claim}

\newtheorem*{remark*}{Remark}

\newtheorem*{observation*}{Observation}
\usepackage{paralist}
\let\originalleft\left
\let\originalright\right
\renewcommand{\left}{\mathopen{}\mathclose\bgroup\originalleft}
\renewcommand{\right}{\aftergroup\egroup\originalright}
\usepackage{turnstile}
\usepackage{mdframed}
\usepackage{tikz}
\usetikzlibrary{positioning}
\usepackage{caption}
\DeclareCaptionType{Algorithm}
\usepackage{newfloat}
\usepackage{array}
\usepackage{subfig}
\usepackage{bbm}
\usepackage{xparse}
\usepackage{amsthm} %
\makeatletter
\let\latexparagraph\paragraph
\RenewDocumentCommand{\paragraph}{som}{%
  \IfBooleanTF{#1}
    {\latexparagraph*{#3}}
    {\IfNoValueTF{#2}
       {\latexparagraph{\maybe@addperiod{#3}}}
       {\latexparagraph[#2]{\maybe@addperiod{#3}}}%
  }%
}
\newcommand{\maybe@addperiod}[1]{%
  #1\@addpunct{.}%
}
\makeatother

\usepackage{boxedminipage}
\newcommand{\paren}[1]{(#1)}
\newcommand{\Paren}[1]{\left(#1\right)}

\newcommand{\Brac}[1]{\left[#1\right]}
\newcommand{\iverson}[1]{\llbracket#1\rrbracket}

\newcommand{\abs}[1]{\lvert#1\rvert}
\newcommand{\Abs}[1]{\left\lvert#1\right\rvert}

\newcommand{\card}[1]{\lvert#1\rvert}

\newcommand{\set}[1]{\{#1\}}
\newcommand{\Set}[1]{\left\{#1\right\}}

\newcommand{\norm}[1]{\lVert#1\rVert}
\newcommand{\Norm}[1]{\left\lVert#1\right\rVert}

\newcommand{\Normt}[1]{\Norm{#1}_2}

\newcommand{\Snormt}[1]{\Norm{#1}^2_2}

\newcommand{\Snorm}[1]{\Norm{#1}^2}

\newcommand{\normo}[1]{\norm{#1}_1}
\newcommand{\Normo}[1]{\Norm{#1}_1}

\newcommand{\normi}[1]{\norm{#1}_\infty}
\newcommand{\Normi}[1]{\Norm{#1}_\infty}

\newcommand{\normcol}[1]{\lVert#1\rVert_{2\to\infty}}

\newcommand{\iprod}[1]{\langle#1\rangle}
\newcommand{\Iprod}[1]{\left\langle#1\right\rangle}

\newcommand{\Esymb}{\mathbb{E}}
\newcommand{\Psymb}{\mathbb{P}}

\DeclareMathOperator*{\E}{\Esymb}

\DeclareMathOperator*{\ProbOp}{\Psymb}
\renewcommand{\Pr}{\ProbOp}

\newcommand{\tensor}{\otimes}

\newcommand{\sge}{\succeq}
\newcommand{\sle}{\preceq}

\renewcommand{\ij}{{ij}}

\newcommand{\defeq}{\stackrel{\mathrm{def}}=}

\newcommand{\from}{\colon}

\newcommand{\mper}{\,.}

\newcommand\bdot\bullet

\DeclareMathOperator{\Tr}{Tr}

\DeclareMathOperator{\OPT}{OPT}

\DeclareMathOperator{\poly}{poly}

\DeclareMathOperator{\polylog}{polylog}
\DeclareMathOperator{\supp}{supp}

\DeclareMathOperator{\rank}{rank}

\newcommand{\Erdos}{Erd\H{o}s\xspace}
\newcommand{\Renyi}{R\'enyi\xspace}

\newcommand{\N}{\mathbb N}
\newcommand{\R}{\mathbb R}

\newcommand{\problemmacro}[1]{\texorpdfstring{\textup{\textsc{#1}}}{#1}\xspace}

\newcommand{\maxcut}{\problemmacro{max cut}}
\newcommand{\maxbisection}{\problemmacro{max bisection}}
\newcommand{\twocsp}{\problemmacro{2-csp}}

\newcommand{\cA}{\mathcal A}
\newcommand{\cB}{\mathcal B}

\newcommand{\cD}{\mathcal D}
\newcommand{\cE}{\mathcal E}

\newcommand{\cI}{\mathcal I}

\newcommand{\cM}{\mathcal M}

\newcommand{\cO}{\mathcal O}

\newcommand{\cS}{\mathcal S}

\newcommand{\cY}{\mathcal Y}
\newcommand{\cZ}{\mathcal Z}

\newcommand{\bbP}{\mathbb P}

\renewcommand{\leq}{\leqslant}
\renewcommand{\le}{\leqslant}
\renewcommand{\geq}{\geqslant}
\renewcommand{\ge}{\geqslant}

\let\epsilon=\varepsilon
\numberwithin{equation}{section}
\newcommand\MYcurrentlabel{xxx}
\newcommand{\MYstore}[2]{%
  \global\expandafter \def \csname MYMEMORY #1 \endcsname{#2}%
}
\newcommand{\MYload}[1]{%
  \csname MYMEMORY #1 \endcsname%
}
\newcommand{\MYnewlabel}[1]{%
  \renewcommand\MYcurrentlabel{#1}%
  \MYoldlabel{#1}%
}
\newcommand{\MYdummylabel}[1]{}
\newcommand{\torestate}[1]{%
  \let\MYoldlabel\label%
  \let\label\MYnewlabel%
  #1%
  \MYstore{\MYcurrentlabel}{#1}%
  \let\label\MYoldlabel%
}
\newcommand{\restatedef}[1]{%
  \let\MYoldlabel\label
  \let\label\MYdummylabel
  \begin{definition*}[Restatement of \cref{#1}]
    \MYload{#1}
  \end{definition*}
  \let\label\MYoldlabel
}
\newcommand{\restatetheorem}[1]{%
  \let\MYoldlabel\label
  \let\label\MYdummylabel
  \begin{theorem*}[Restatement of \cref{#1}]
    \MYload{#1}
  \end{theorem*}
  \let\label\MYoldlabel
}
\newcommand{\restatelemma}[1]{%
  \let\MYoldlabel\label
  \let\label\MYdummylabel
  \begin{lemma*}[Restatement of \cref{#1}]
    \MYload{#1}
  \end{lemma*}
  \let\label\MYoldlabel
}
\newcommand{\restateprop}[1]{%
  \let\MYoldlabel\label
  \let\label\MYdummylabel
  \begin{proposition*}[Restatement of \cref{#1}]
    \MYload{#1}
  \end{proposition*}
  \let\label\MYoldlabel
}
\newcommand{\restatefact}[1]{%
  \let\MYoldlabel\label
  \let\label\MYdummylabel
  \begin{fact*}[Restatement of \cref{#1}]
    \MYload{#1}
  \end{fact*}
  \let\label\MYoldlabel
}
\newcommand{\restate}[1]{%
  \let\MYoldlabel\label
  \let\label\MYdummylabel
  \MYload{#1}
  \let\label\MYoldlabel
}

\newcommand{\e}{\epsilon}
\newcommand{\eps}{\epsilon}

\allowdisplaybreaks
\newcommand*{\Id}{\mathrm{I}}

\newcommand*{\normf}[1]{\norm{#1}_{\mathrm{F}}}
\newcommand*{\Normf}[1]{\Norm{#1}_{\mathrm{F}}}

\newenvironment{algorithmbox}{\begin{mdframed}[nobreak=true]\begin{algorithm}}{\end{algorithm}\end{mdframed}}
\newcommand{\normn}[1]{\norm{#1}_\textnormal{nuc}}

\renewcommand{\normo}[1]{\norm{#1}_{\textnormal{1}}}
\renewcommand{\Normo}[1]{\Norm{#1}_{\textnormal{1}}}

\renewcommand{\normi}[1]{\norm{#1}_{\max}}
\renewcommand{\Normi}[1]{\Norm{#1}_{\max}}

\providecommand{\Lap}{\text{Lap}}

\newcommand*{\transpose}[1]{{#1}{}^{\mkern-1.5mu\mathsf{T}}}
\newcommand*{\dyad}[1]{#1#1{}^{\mkern-1.5mu\mathsf{T}}}

\newcommand{\1}{{\mathbb{1}}}

\definecolor{niceish}{HTML}{74b807} 
\ifnum\draft=1
\newcommand{\siuon}[1]{\textcolor{teal}{[Siu On: #1]}}
\newcommand{\lucas}[1]{\textcolor{violet}{[Lucas: #1]}}
\newcommand{\tom}[1]{\textcolor{WildStrawberry}{[Tommaso: #1]}}
\newcommand{\gleb}[1]{\textcolor{WildStrawberry}{[Gleb: #1]}}
\else
\newcommand{\siuon}[1]{}
\newcommand{\lucas}[1]{}
\newcommand{\gleb}[1]{}
\newcommand{\tom}[1]{}
\fi

\newcommand{\om}{\om}

\DeclareMathOperator*{\Cov}{Cov}

\RequirePackage[outline]{contour} % used to define a less-bold \xsos
\contourlength{0.065pt}
\contournumber{10}%

\newcommand{\mul}{\textnormal{\textsc{mul}}\xspace}
\newcommand{\simul}{\textnormal{\textsc{mul}}\xspace}

\newcommand{\valI}[2]{\textnormal{Val}_{#1}\Paren{#2}}
\newcommand{\coh}[2]{\mu_{#1}\paren{#2}}

\newcommand{\GC}{\textnormal{GC}}
\newcommand{\LC}{\textnormal{LC}}
\newcommand{\simiid}{\stackrel{\text{iid}}\sim}

\begin{document}

\title{Tight Differentially Private PCA via Matrix Coherence \gleb{IF YOU SEE IT, IT IS A DRAFT VERSION.}
}

\author{
  Tommaso d'Orsi%
  \thanks{Bocconi University, Italy.}
  %\texttt{tommaso.dorsi@unibocconi.it}}%
  \and
  Gleb Novikov%
  \thanks{Lucerne School of Computer Science and Information Technology, Switzerland.}
  %\texttt{gleb.novikov@hslu.ch}}
}

\date{}

\pagestyle{empty}
\maketitle
\thispagestyle{empty} 
\begin{abstract}
We revisit the task of computing the span of the top \( r \) singular vectors \( u_1, \ldots, u_r \) of a matrix under differential privacy.  
We show that a simple and efficient algorithm—based on singular value decomposition and standard perturbation mechanisms—returns a private rank-\( r \) approximation whose error depends only on the \emph{rank-\( r \) coherence} of \( u_1, \ldots, u_r \) and the spectral gap \( \sigma_r - \sigma_{r+1} \). This resolves a question posed by Hardt and Roth~\cite{hardt2013beyond}.  
Our estimator outperforms the state of the art—significantly so in some regimes. In particular, we show that in the dense setting, it achieves the same guarantees for single-spike PCA in the Wishart model as those attained by {optimal non-private algorithms}, whereas prior private algorithms failed to do so.

In addition, we prove that (rank-\( r \)) coherence does not increase under Gaussian perturbations. This implies that any estimator based on the Gaussian mechanism—including ours—preserves the coherence of the input. We conjecture that similar behavior holds for other structured models, including planted problems in graphs. 

We also explore applications of coherence to graph problems. In particular, we present a differentially private algorithm for Max-Cut and other constraint satisfaction problems under low coherence assumptions.

\end{abstract}

% TABLE OF CONTENT
\clearpage
\microtypesetup{protrusion=false}
\tableofcontents{}
\microtypesetup{protrusion=true}
\clearpage

\pagestyle{plain}
\setcounter{page}{1}

\section{Introduction}\label{sec:introduction}
For a matrix $M\in\R^{n\times  m},$ consider the basic task of finding its $r$ left (or right) leading singular vectors $U_{(r)} \in \R^{n\times r}$.
%As stated, $\hat{M}_{(r)}$ can be compute via singular value decomposition. 
Because of its ubiquity in data mining applications, there has been ongoing effort to efficiently construct accurate yet sanitized versions of $\hat{U}_{(r)}$ that do not reveal sensitive information about $M$. The privacy notion considered is typically that of $(\eps,\delta)$-differential privacy \cite{dwork2006calibrating}, which we also adopt here with respect to matrices that differ by at most $1$ in a single entry.\footnote{In fact, we use a more general notion of adjacency, 
\cref{def:matrix-adjacency}
though this distinction is inconsequential for the scope of this discussion.}.%\footnote{Different notions of adjacency based on the spectral norm the Frobenius norm, or other distance metrics are also used. We discuss these in \tom{related work}.}

A large body of work has focused on the problem of computing a private low-rank approximation of a matrix~\cite{blum2005practical, blocki2012johnson, dwork2006calibrating, dwork2014analyze, hardt2012beating, upadhyay2018price, mangoubi2022re, mangoubi2023private, mangoubi2025private, he2025differentially}, largely motivated by the fact that in many applications, the top singular vectors are significantly more important than the rest.
A common example arises in graph partitioning, where the goal is to privately compute a cut; since cuts can be expressed as quadratic forms over the adjacency matrix, this naturally motivates the need to estimate the top singular vectors~\cite{blocki2012johnson, gupta2012iterative, borgs2015private, borgs2018revealing, arora2019differentially, mohamed2022differentially, dalirrooyfard2023nearly, chandra2024differentially, chen2023private, chen2024private}.
A fruitful line of work has pursued this direction~\cite{chaudhuri2012near, dwork2014analyze, hardt2013beyond, kapralov2013differentially, hardt2014noisy, gonem2018smooth, singhal2021privately, liu2022dp, nicolas2024differentially}, often arriving at a seemingly discouraging conclusion: in the worst case, the utility error must inherently depend on the ambient dimension.

 Notably, a sequence of works \cite{hardt2013beyond, hardt2014noisy, balcan2016improved, nicolas2024differentially} showed that utility guarantees need not degrade with the ambient dimension of the data, but instead depend only on the \emph{coherence} of the input matrix $M$. The coherence  $\mu(M)$ of a matrix $M$ takes values between $1$ and $\max\set{n,m}$ and, roughly speaking, measures the sparsity of its singular vectors. Since real-world matrices typically exhibit low coherence, the concerted message of these results is that the worst-case scenario is rare --arising only from peculiar matrices-- and that one can typically expect a \emph{dimension-free} accuracy bound.

To distinguish this notion of coherence from other definitions, we henceforth refer to it as \emph{basic coherence}. %For simplicity, we (as well as prior works) now assume  that $m=n$ and that $M$ is symmetric. The general case can be reduced to this one via standard techniques.

\begin{definition}[Basic coherence]\label{def:basic-coherence}
    The basic coherence of a matrix $M\in \R^{n\times m}$ with singular value decomposition $\sum_{i=1}^{\rank(M)} \sigma_i u_i \transpose{v_i}$ is 
    \begin{align*}
        \bar{\mu}(M):=\max\Set{n\cdot \max_{i\in [n]} \Normi{u_i}^2, \; m\cdot \max_{i\in [m]} \Normi{v_i}^2}.
    \end{align*}
    where $\Normi{\cdot}$ denotes the largest entry in absolute value.
\end{definition}
\noindent The differentially private algorithm with the best known utility guarantees for low-coherence matrices is from \cite{hardt2014noisy}. To state their result --and to discuss utility more generally-- we introduce a standard notion of closeness of subspaces of (possibly) different dimensions:
\begin{definition}[Closeness of subspaces]
Let $r, r'\in [n]$ such that $r' \ge r$.
Let $S_1, S_2 \subset \R^n$ be vector subspaces of $\R^n$ of dimensions $r$ and $r'$ respectively. Let $U_1 \in \R^{n\times r}$ be a matrix whose columns form an orthonormal basis in $S_1$, and let $P_2$ be an orthogonal projector onto $S_2$. The closeness of $S_2$ to $S_1$ is
\[
%\sdist\Paren{S_2, S_1} = 
\Norm{\Paren{\Id_n - P_2}U_1}\,.
\]
\end{definition}

An equivalent geometric definition of the closeness is as follows: The maximum distance from unit vectors in $S_1$ to $S_2$.  This equivalent formulation shows that the definition does not depend on the choice of basis in $S_1$.
When $\dim(S_1)=\dim(S_2)$, the closeness is equal to the sine of the angle between the subspaces.
Now we are ready to state the main result of \cite{hardt2014noisy}.

\begin{theorem}[\cite{hardt2014noisy}\footnote{\cite{balcan2016improved} showed that the algorithm \cite{hardt2014noisy} achieves slightly better guarantees for the restriced case of positive semidefinite matrices.}]\label{thm:hardt-price}
Let $M\in\R^{n\times m}$, and let $r, r'\in [\rank(M)]$ such that $r'\ge r$. Let $U_{(r)} \in \R^{n\times r}$ be the matrix whose columns are the top $r$ left singular vectors of $M$.
    There exists an efficient, $(\eps,\delta)$-differentially private algorithm that, given $M, r'$  returns a projector $\hat{\mathbf P}_{(r')}\in \R^{n\times n}$ onto an $r'$-dimensional space such that, with probability $0.99$,
    \begin{align*}
       \Norm{\Paren{\Id_n - \hat{\mathbf P}_{(r')}} U_{(r)}}
       \leq O\Paren{\frac{\sqrt{r' \cdot \bar{\mu}(M) \cdot \log \paren{n+m}}}{\sigma_r - \sigma_{r+1}} 
       \cdot \sqrt{L\log L}
       \cdot \frac{\sqrt{r'}}{\sqrt{r'}-\sqrt{r-1}}\cdot\frac{\sqrt{\log(1/\delta)}}{\eps}
       }\,,
   \end{align*}
    where $L = \frac{\sigma_r \log \paren{n+m} }{\sigma_r - \sigma_{r+1}}$\,.
\end{theorem}

In other words, the algorithm from \cref{thm:hardt-price} outputs a projector onto an $r'$-dimensional subspace of $\R^n$ that is close to the $r$-dimensional subspace spanned by the $r$ leading singular vectors of $M$.
Importantly, the subspace closeness depends on $r,r'$ and the coherence of the input matrix, but almost does not depend on the ambient dimension.
Despite this remarkable property, the algorithm has several fundamental limitations.
First, it requires \textit{all} singular vectors of $M$ to be dense, despite aiming to privatize only the first $r.$  
This  limitation can be  significant. 
For instance, the matrix $M:=\dyad{\mathbbm{1}_n}+\tfrac{1}{n}\Id_n$ (where ${\mathbbm{1}_n}$ is the vector with all entries equal to $1$) has an incoherent spike we might be interested in to find, but it has $\bar{\mu}(M)\geq \Omega(n)$, and hence \cref{thm:hardt-price} does not provide any non-trivial guarantees.
Similarly, while the adjacency matrix of an \Erdos-\Renyi graph with average degree $\tfrac{n}{2}$ satisfies $\bar{\mu}\leq O(\polylog n)$ with high probability, adding an isolated clique of size $O(1)$ suffices to raise the basic coherence to $\Omega(n).$
Second, the approximation with $r'=r$ is worse than the approximation with $r'=2r$ by a factor $O(r)$. 
For many natural applications one may be interested in approximating the singular space by a space of \textit{exactly} the same dimension.
Finally, the approximation error depends not only on the gap  $\sigma_r-\sigma_{r+1}$, but also on the ratio $\sigma_r / \paren{\sigma_r - \sigma_{r+1}},$ which may be disproportionally larger.
This scenario can arise in many natural problems, as we discuss in more detail later.
%This scenario may arise when the best rank-$(r+1)$ approximation of the input  $M$ contains more relevant information than the best rank-$r$ approximation for the task at hand. In such cases, the algorithm designer may be unable to tune the parameters appropriately --e.g., due to privacy constraints-- despite the existence of a better trade-off.

%Let us illustrate the significance of the first limitation in more detail. For instance, the matrix $M:=\dyad{\mathbbm{1}_n}+\Id_n$ (where $\dyad{\mathbbm{1}_n}$ is the vectors with all entries equal to $1$) has very small ($1/\sqrt{n}$) entries of its top eigenvector, its coherence is very high: $\bar{\mu}(M)\geq \Omega(n)$. Similarly, while the adjacency matrix of an \Erdos-\Renyi graph with average degree $\tfrac{n}{2}$ satisfies $\bar{\mu}\leq O(\log n)$ with high probability, adding an isolated clique (not connected to the rest of the graph) of size $O(1)$ is enough to increase the basic coherence to $\Omega(n)$.

In fact, in \cite{hardt2013beyond} the authors themselves conjectured that the first limitation could be overcome. In this work, we answer that question in the affirmative, showing that a significantly weaker notion of coherence\footnote{The notion of coherence that we use is even weaker than the one  conjectured in \cite{hardt2013beyond}.} suffices to obtain strong utility guarantees. 
Furthermore, our algorithm --which is based on different ideas-- avoids all of the other limitations discussed above.
%Building on this improvement, we introduce new differentially private algorithms for graph partitioning and related tasks.
%\footnote{Our algorithm is different from the algorithm from \cite{hardt2014noisy} and is based on a different idea.}

\subsection{Main result}
To state our results we consider a more general notion of matrix coherence introduced in \cite{candes2012exact}.

\begin{definition}[$r$-Coherence]\label{def:coherence}
    Let $M\in \R^{n\times n}$ be a matrix, and let $M = \sum_{i=1}^{\rank(M)} \sigma_{i}u_iv_i^\top$ be its singular value decomposition such that $\sigma_1 \ge \ldots \ge \sigma_{\rank(M)}$. Let $r \in [\rank(M)]$. The \emph{rank-$r$ coherence} of $M$ is
    \[
    \coh{r}{M} := \max\Set{\frac{n}{r}\Normi{\sum_{i=1}^{r} u_iu_i^\top},\; \frac{m}{r}\Normi{\sum_{i=1}^{r} v_iv_i^\top}} = \max\Set{\frac{n}{r}\normi{P_{(r)}}, \;\frac{m}{r}\normi{Q_{(r)}}}\,,
    \]
    where $P_{(r)}$ and $Q_{(r)}$ are the orthogonal projectors onto the spaces spanned by $r$ leading left and right singular vectors of $M$ respectively.
\end{definition}
\noindent
By design $\coh{r}{M}$ takes values between $1$ and $\max\set{n,m}/r$ and satisfies $\coh{r}{M}\leq \bar{\mu}(M)$ for all $r$. 
Importantly, in the context of the examples from the previous paragraph, \cref{def:coherence} correctly captures the fact that the best low-rank approximation of the matrix has low coherence.
Specifically, for the first example $M=\dyad{\mathbbm{1}_n}+\tfrac{1}{n}\Id_n$--where $\bar{\mu}(M)\geq \Omega(n)$, we have $\coh{1}{M}\leq O(1).$ 
For the adjacency matrix of the aforementioned random graph with a small isolated clique, $\coh{r}{M}\leq O(\polylog n)$ for all $r \lesssim n$ with high probability.
%A second significant difference with respect to \cref{def:basic-coherence} is that even for a rank-$1$ asymmetric matrix $M,$  $\coh{\rank(M)}{M}=\coh{1}{M}$ may be vanishingly small compared to $\bar{\mu}(M).$ For instance, for $M=e_1\transpose{\mathbb{1}_n}$, where $e_1$ denotes the first canonical basis vector in $\R^n$, we have $\coh{\rank(M)}{M}=\tfrac{\sqrt{n}}{r}$ while $\bar{\mu}(M)=n.$

Under this definition, we obtain the following result. 
\begin{theorem}[Private singular subspace estimator]\label{thm:main}
Let $M\in\R^{n\times m}$ and $r\in [\rank(M)]$.
Let $U_{(r)} \in \R^{n\times r}$ be the matrix whose columns are the top $r$ left singular vectors of $M$.
    There exists an efficient, $(\eps,\delta)$-differentially private algorithm that, given $M, r$  returns a projector $\hat{\mathbf P}_{(r)}\in \R^{n\times n}$ onto an $r$-dimensional space such that, with probability $0.99$,
    \begin{align*}
       \Norm{\Paren{\Id_n - \hat{\mathbf P}_{(r)}} U_{(r)}}
       \leq O\Paren{\frac{\sqrt{r \cdot \mu_r(M) + \log(1/\delta)} }{\sigma_r - \sigma_{r+1}} 
       \cdot \frac{\sqrt{\log(1/\delta)}}{\eps}}\,.
   \end{align*}
\end{theorem}
In addition to relying on a weaker notion of coherence, our algorithm offers several improvements over \cite{hardt2014noisy}. First, it returns the projector onto a space of dimension \emph{exactly} $r$. If one is interested in privately estimating the $r$-dimensional singular space of $M$ using subspaces of exactly the same dimension, \cref{thm:main} improves over \cite{hardt2014noisy} by a factor $O(r).$

Second, in the standard regime $\delta \ge 1/\poly(n)$, regardless of the singular values or the coherence of the input, the error of the estimator from \cite{hardt2014noisy} is always at least a $\sqrt{\log n}$ factor larger than that of our estimator. Furthermore, the term $r \cdot \mu_r(M) $ almost always dominates the term $\log(1/\delta)$. Indeed, if $\;r \ge \Omega(\log(n))$, it clearly dominates. For$\;r < o(\log(n))\;$(e.g.$\;r=1$), even incoherent random matrices satisfy$\;\mu_r(M) \ge \Omega(\log(n))\;$with high probability, and hence even for such matrices (and, of course, for matrices with higher coherence) this term also dominates. Hence in almost all regimes our bound is at least a $\log n$ factor larger than that of the estimator of \cite{hardt2014noisy}. 

Third, \cref{thm:main} depends only on the spectral gap $\sigma_r - \sigma_{r+1}$, and not on the ratio $\sigma_r / \paren{\sigma_r - \sigma_{r+1}}$.  
A concrete example where this distinction becomes crucial is the classical problem of single-spike PCA in the Wishart model, which has been extensively studied in the literature (e.g., \cite{johnstone2001distribution, johnstone2009consistency, berthet2013optimal, deshpande2016sparse, d2020sparse, novikov2023sparse}). In this model, $M \in \R^{n\times m}$ is a matrix whose columns are iid samples from the Gaussian distribution with spiked covariance. Concretely, $M$ can be represented as follows:
\[
M = \sqrt{\beta} \cdot u \mathbf{g}^\top + \mathbf{W},
\]
where $u \in \mathbb{R}^n$ is a unit signal vector, $\mathbf{g} \sim N(0,1)^m$, and $\mathbf{W} \sim N(0,1)^{n \times m}$ are independent. We assume that $u$ is delocalized (i.e., its entries are at most $\tilde{O}(\sqrt{1/n})$), which ensures that $M$ is incoherent.

It is well known that in the large-sample regime $m \gg n$, if $\beta = C\sqrt{n/m}$ with a sufficiently large constant $C$, the top left singular vector of $M$ is highly correlated with $u$ with high probability. Moreover, it can be shown\footnote{We formally prove this in \cref{sec:wishart}.} 
the spectral gap is $\sigma_1 - \sigma_2 = \Theta(\beta \sqrt{m}) = \Theta(\sqrt{n})$, and the coherence satisfies $\coh{1}{M} \le \tilde{O}(1)$. Since a typical entry of $M$ is $\Theta(1)$, our notion of adjacent inputs is adequate in this setting.

Hence, \cref{thm:main} yields an error of $\tilde{O}(1/\sqrt{n})$. In particular, for delocalized signals, our private algorithm succeeds in exactly the same regime as classical (non-private) PCA. Furthermore, if $\beta \lesssim \sqrt{n/m}$, recovering $u$ becomes information-theoretically impossible.

In contrast, the algorithm from~\cite{hardt2014noisy} yields error
\[
\tilde{O}\left(\sqrt{\frac{\sigma_1}{\sigma_1 - \sigma_2}} \cdot \frac{1}{\sigma_1 - \sigma_2} \right) = \tilde{O}\left(\sqrt{\frac{\sqrt{m}}{\sqrt{n}}} \cdot \frac{1}{\sqrt{n}} \right) = \tilde{O}\left(\left(\frac{m}{n^3}\right)^{1/4}\right),
\]
which is significantly worse. In particular, when $m \gg n^3$, the output of their estimator is not correlated with $u$, so the algorithm from~\cite{hardt2014noisy} fails to solve the problem in the large-sample regime—even for delocalized signals, where the input matrix is incoherent.
%We remark that since this example relied only on the spectral properties of the problem, both estimators behave similarly on other problems with comparable spectral structure --including classic statistical primitives such as sparse PCA in the Wigner model.

Another important consequence of the fact that \cref{thm:main} depends only on the spectral gap --and not on $\sigma_r$-- is the \emph{shift invariance} of the estimator.
Specifically, let $M\in \R^{n\times n}$ be symmetric. Given a differentially private upper bound $b$ on the spectral norm $\norm{M}$ (which has low sensitivity due to the triangle inequality and can thus be privatized using standard mechanisms), we can always work with the positive definite matrix $M + b\cdot \Id_n$. 
As  the $r$ leading \emph{eigenvectors} of $M$ are the leading \emph{singular vectors} of $M + b\cdot \Id_n$, the utility  guarantees of \cref{thm:main} extend directly to eigenvector estimation.
\begin{corollary}[Private eigenspace estimator]
    Let $M\in\R^{n\times n}$ be a symmetric matrix with eigenvalues $\lambda_1 \ge \lambda_2 \ge\cdots \ge \lambda_n$, and let $r\in [n]$. Let $U_{(r)} \in \R^{n\times r}$ be the matrix whose columns are the top $r$ eigenvectors of $M$.
    There exists an efficient, $(\eps,\delta)$-differentially private algorithm that, given $M, r$  returns a projector $\hat{\mathbf P}_{(r)}\in \R^{n\times n}$ onto an $r$-dimensional space such that, with probability $0.99$,
    \begin{align*}
       \Norm{\Paren{\Id_n - \hat{\mathbf P}_{(r)}} U_{(r)}}
       \leq O\Paren{\frac{\sqrt{r \cdot \mu_r(M) + \log(1/\delta) } }{\lambda_r - \lambda_{r+1}} 
       \cdot \frac{\sqrt{\log(1/\delta)}}{\eps}}\,.
   \end{align*}
\end{corollary}
In particular, if $\lambda_1 \gg \lambda_2$, we can accurately and privately estimate the leading eigenvector $v_1$ (up to sign), even when $\sigma_r \gg \abs{\lambda_1}$ for $r \ge \Omega(n)$. Note that since the error of the estimator from \cite{hardt2014noisy} includes a multiplicative factor of $\sigma_r$, it is \emph{not} shift invariant:  adding $b \cdot \Id_n$ to $M$ will proportionally increase the error of their estimator.

Finally, we emphasize that the algorithm underlying \cref{thm:main} consists of a sequence of elementary operations, such as computing the top-$r$ singular vectors and singular values and adding noise entry-wise. As such, we believe it may be of immediate practical interest.
Moreover, we note that our result holds under a more general notion of adjacency 
(see \cref{def:matrix-adjacency}) 
than the one introduced in the beginning of the paper, and the other notions of adjacency used for this problem in prior work \cite{hardt2013beyond,hardt2014noisy,balcan2016improved,nicolas2024differentially}.

\paragraph{Coherence of our estimator.} For certain applications the coherence of the estimator might be important. Therefore, it is desirable that coherence does not increase significantly after privatizing the eigenvectors. We show that $\coh{r}{\hat{\mathbf P}_{(r)}} \le O\Paren{\coh{r}{M} + {\frac{\log \paren{n+m}}{r}}}$ with high probability, where $\hat{\mathbf P}_{(r)}$ is our estimator from \cref{thm:main}. Note that even for highly incoherent matrices, for example, random Gaussian matrices, the coherence is $\Theta(\log n)$, and hence for such matrices the coherence of our estimator can be larger at most by a constant factor than the coherence of the input.
We remark that the estimator $\hat{P}_{\text{HP}}$ from \cite{hardt2014noisy} has $\bar{\mu}(\hat{P}_{\text{HP}})\le \bar{\mu}(M)\cdot \log \paren{n+m}$, which is in most cases by a $\log$ factor larger than the guarantees of our estimator. Note also that their bound is only valid for \emph{basic} coherence.

\subsection{Coherence of graphs under random perturbations}\label{sec:conjecture}
To show the  bound on the coherence of our estimator, we prove the following statement:
\footnote{See \cref{thm:coherence-gaussian} for the formal version.}
if $A \in \mathbb{R}^{n \times m}$ is a rank-$r$ matrix and $\mathbf{W} \in \mathbb{R}^{n \times m}$ has i.i.d.\ standard Gaussian entries, then with high probability
\[
\coh{r}{A + \mathbf{W}} \le O\Paren{\coh{r}{A} + \frac{\log(n + m)}{r}}.
\]
While our proof heavily relies on properties specific to the Gaussian distribution (rotational symmetry) we believe that similar coherence bounds should hold for other random matrices with comparable structural features.
To formalize this belief that coherence should be stable under random perturbations,  we conjecture that adding a $G(n, p)$ graph to a low-rank graph does not significantly increase the coherence of the resulting adjacency matrix. Concretely, we propose the following:

\begin{conjecture}\label{conj:coherence-simple}
    Let $K$ be a graph whose adjacency matrix $A_K$ has rank $r$, 
    and let $\mathbf{G} \sim G(n,p)$ with $1/2 \ge p \ge \polylog(n)$.
    Let $\hat{\mathbf{K}}$ denote the union of $K$ and $\mathbf{G}$. Then, with high probability,
    \[
    \coh{r}{A_{\hat{\mathbf{K}}}} \le O\Paren{\coh{r}{A_K}} + \polylog(n),
    \]
    where $A_{\hat{\mathbf{K}}}$ is the adjacency matrix of $\hat{\mathbf{K}}$.
\end{conjecture}

We remark that even for the Erd\H{o}s--R\'enyi model $G(n,p)$, proving  bounds on coherence has required multiple papers and nontrivial techniques \cite{dekel2011eigenvectors, tao2010random, erdHos2013spectral}, and the precise bound is conjectured to be $\log(n)$ (see, e.g.,  \cite{vu2015random}), and, to the best of our knowledge, remains an open problem. Currently, it is known \cite{erdHos2013spectral} that for all $r$, the adjacency matrix $A_{\mathbf{G}}$ of a random graph $\mathbf{G} \sim G(n,p)$ satisfies
\[
\coh{r}{A_{\mathbf{G}}} \le {\polylog(n)},
\]
as long as $p \ge \polylog(n)$. This suggests that \cref{conj:coherence-simple} may be difficult to prove in general. On the other hand, the existing incoherence bound for $G(n,p)$ may offer a promising starting point for establishing the conjecture, if it holds.
For many applications, it is convenient to consider the normalized adjacency matrix of a graph. (see \cref{sec:preliminaries})
We therefore formulate a corresponding conjecture for the normalized adjacency matrix.

\begin{conjecture}\label{conj:coherence-normalized}
    Let $K$ be a graph whose {normalized} adjacency matrix $\bar{A}_K$ has rank $r$, and let $\mathbf{G} \sim G(n,p)$ with $1/2 \ge p \ge \polylog(n)$. Let $\hat{\mathbf{K}}$ denote the union of $K$ and $\mathbf{G}$. Then with high probability,
    \[
    \coh{r}{\bar{A}_{\hat{\mathbf{K}}}} \le O\Paren{\coh{r}{\bar{A}_K}} + \polylog(n)\,,
    \]
    where $\bar{A}_{\hat{\mathbf{K}}}$ is the normalized adjacency matrix of $\hat{\mathbf{K}}$.
\end{conjecture}

%Compared to \cref{thm:hardt-price}, \cref{thm:main} improves the utility guarantees by a multiplicative factor $\tilde{O}\Paren{\frac{\coh{r}{M}\cdot\Lambda}{\bar{\mu}(M)\cdot\sigma_r}}\leq \tilde{O}\Paren{\frac{\coh{r}{M}}{\bar{\mu}(M)}}$ (where we used $\tilde{O}(\cdot)$ to hide polylogarithmic factors). That is, even when $\coh{r}{M}\geq \Omega(\bar{\mu}(M))$ the Theorem provides  sharper utility guarantees. Moreover, \cref{thm:main} holds under a strictly more general notion of adjacency (see \tom{pointer to definition}).

%We remark that, as in \cref{thm:hardt-price}, both parameters $\Lambda, \bar{\mu}$ can be appropriately chosen paying only a degradation in the privacy parameters. We defer this discussion to \tom{link to remark}.

\subsection{Differentially private CSP solvers}
To motivate our conjectures, we demonstrate that low coherence can have meaningful algorithmic consequences in the context of privacy. We introduce novel differentially private algorithms for $2$-CSPs. We defer the general statement to \cref{sec:applications} and present here the special case of \maxcut under edge-differential privacy (see \cref{sec:preliminaries} for the definition). In the \maxcut problem, the goal is to find a bipartition of the vertex set that maximizes the number of edges crossing the cut in a given graph \( G \). Let $\sigma_{r}$ denote the $r$-th largest singular value of the normalized adjacency matrix of $G.$  We prove the following theorem:
\begin{theorem}[Edge-DP \maxcut for low coherence graphs, simplified]\label{thm:dp-max-cut-easy}
    Let $C>0$ be a universal constant.
    There exists an $(\eps,\delta)$-DP algorithm that, given a graph $G$, and an integer $r>0,$ with high probability returns a bipartition such that the number of cut edges is at least
    \begin{align*}
        %\OPT -O\Paren{\frac{n\log n}{\eps}} - 
        (0.99-\sigma_{r+1})\cdot\OPT
    \end{align*}
    whenever $G$ has 
    \begin{align}\label{eq:intro-requirements}
        d_{\min}\geq C\Paren{\frac{\sqrt{\log(1/\delta)}}{\eps}\cdot \frac{\sqrt{r\cdot\mu_r+\log n}}{(\sigma_r -\sigma_{r+1})}},\qquad \sigma_r\geq 0.01.
    \end{align}
    Moreover, the algorithm runs in randomized time $n^{ O\Paren{1}}\cdot \exp\Set{O\Paren{{r}}
}.$
\end{theorem}
\cref{thm:dp-max-cut-easy} states that, whenever the minimum degree is sufficiently large relative to the coherence of the graph, it is possible to efficiently and privately recover the maximum cut up to accuracy $(0.99-\sigma_{r+1}).$ 
Observe that one can always satisfy this minimum degree requirement by adding a random graph of comparable expected degree on top of the input. 
Since the maximum cut always has value at least  $\card{E}/2,$ this modification does not significantly affect the maximum cut, provided the average degree is at least a constant factor larger. 
Under \cref{conj:coherence-normalized}, this perturbation cannot significantly change the coherence of the input. 
We further remark that the approximation factor --as well as the requirement on $\sigma_{r}$-- can be improved at the cost of increased running time and a stronger minimum degree assumption. %(see \cref{thm:dp-max-cut}).

A  remarkable sequence of works \cite{blocki2012johnson,gupta2012iterative,arora2019differentially, eliavs2020differentially, liu2024optimal} has introduced  polynomial-time differentially private algorithms that, given a graph $G$, return a synthetic graph $G'$ in which every cut is preserved up to an additive error of $O(\sqrt{\card{E}\cdot n}\cdot\log^2 (n)/\eps).$
This additive error  becomes negligible when $\card{E}\geq \omega\Paren{n\log^4 n}.$
On the other hand, for the same reason,  for sparser graphs these algorithms provide \emph{no} guarantee on the relationship between the maximum cut of the synthetic graph and that of the original graph $G.$ 
Assuming \cref{conj:coherence-normalized}, \cref{thm:dp-max-cut-easy} would allow us to go beyond this limitation on graphs with coherence $O(\log n).$

\subsection{Organization}
The rest of the paper is organized as follows. In \cref{sec:techniques} we present the main technical ideas behind our results. \cref{sec:preliminaries} introduces the notation used throughout subsequent sections. \cref{sec:future-work} contains a discussion of related open problems. In \cref{sec:privatization-mechanism} we present the proof of \cref{thm:main}. In \cref{sec:threshold-rank} we introduce a mechanism to privatize the normalized adjacency matrix of a graph. \tom{Reminder: check if you change the organization}In \cref{sec:correlation-rounding} we extend the global correlation rounding framework, making it amenable to differential privacy.  Finally in \cref{sec:applications} we combine the results of the previous sections to prove \cref{thm:dp-max-cut-easy} and its generalizations to $2$-CSPs. The appendices contains deferred proofs, discussions and background that flesh out the exposition.

\section{Techniques}\label{sec:techniques}

In this section, we present the main ideas behind our results. Throughout the remainder of the paper (unless stated otherwise), we assume that \( M \in \mathbb{R}^{n \times n} \) is \emph{symmetric}. The general case of rectangular matrices can be reduced to this setting using standard techniques. Specifically, given a non-symmetric matrix \( B \in \mathbb{R}^{m \times n} \), we apply the standard symmetrization trick by embedding it into a larger symmetric matrix
\[
A = \begin{bmatrix}
0 & B \\
B^\top & 0
\end{bmatrix} \in \mathbb{R}^{(m+n) \times (m+n)}.
\]
This transformation preserves the key structural properties relevant to our analysis; see~\cite{hardt2013beyond} for further details.

\paragraph{Approximating the top singular space} The algorithm behind \cref{thm:main} is based on a direct intuition: since we want to privately estimate the projector $P_{(r)}$ onto the span of the top $r$ singular vectors of $M$, we may consider applying the Gaussian mechanism on the projector. While the resulting matrix $\hat{\mathbf Y}_{(r)} = P_{(r)} + \mathbf W$ is far from a projector, we can hope that the projector $\hat{\mathbf P}_{(r)}$ onto the space spanned by its top $r$ singular vectors is sufficiently close to $P_{(r)}$. 
%Let us investigate what scale we should choose. 
Let us determine the appropriate noise scale.
To do so, we need to analyze the sensitivity of the projector. Let $P'_{(r)}$ be the projector onto the space of $r$ top singular vectors of $M' = M + E$ such that $E = e_{i} e_j^\top + e_j e_i^\top$ for $i \neq j$ (We use this simple form of $E$ for illustration; the full analysis extends to far more general perturbations.)
For this, we invoke Wedin’s theorem\footnote{Wedin's theorem is an analogue of the well-known Davis–Kahan theorem for singular value decomposition.}:
\[
\normf{P'_{(r)} - P_{(r)}} \le \frac{2\normf{EU_{(r)}}}{\sigma_r(M)-\sigma_{r+1}(M+E)}\,,
\]
where $U_{(r)}\in \R^{n\times r}$ is the matrix whose columns are the top $r$ singular vectors of $M$.
By Weyl's theorem, $\sigma_{r+1}(M+E) \le \sigma_{r+1}(M) + 1$, so if $\sigma_{r}(M) - \sigma_{r+1}(M) \ge 2$, we obtain
\[
\normf{P'_{(r)} - P_{(r)}} \le \frac{4\normf{EU_{(r)}}}{\sigma_r(M)-\sigma_{r+1}(M)}\,.
\]
Since $E = e_{i} e_j^\top + e_j e_i^\top$, it follows that  $\normf{EU_{(r)}} \le 2\normcol{U_{(r)}}$, where $\normcol{U_{(r)}}$ denotes the maximum row norm of $U_{(r)}$. Note that $\coh{r}{M} = \frac{n}{k}\normcol{U_{(r)}}^2$. Hence, substituting into the bound, we obtain:
\[
\normf{P'_{(r)} - P_{(r)}} \le \frac{8\sqrt{r \coh{r}{M}}}{\sqrt{n}\paren{\sigma_r(M)-\sigma_{r+1}(M)}}\,.
\]

Now, if $\sigma_r(M)-\sigma_{r+1}(M)$ and $\coh{r}{M}$ were \emph{not} private, we could apply the Gaussian mechanism with standard deviation $O\Paren{\frac{\sqrt{r \coh{r}{M}}}{\sqrt{n}\paren{\sigma_r(M)-\sigma_{r+1}(M)}} \cdot \frac{\sqrt{\log(1/\delta)}}{\eps}}$ thus yielding  an $(\eps, \delta)$-differentially private algorithm. 
 By applying Wedin’s theorem in spectral norm and standard concentration bounds for Gaussian matrices, we would recover the desired guarantee from \cref{thm:main} with high probability.

Unfortunately, both the spectral gap and the coherence are part of the input and therefore cannot be used without first privatizing them.
To make the algorithm fully private, we therefore turn to finding good private estimators for $\sigma_r(M)-\sigma_{r+1}(M)$ and $\coh{r}{M}$. 
Estimating the singular value gap $\sigma_r(M)-\sigma_{r+1}(M)$ is straightforward: its sensitivity is bounded by $2$ by Wedin's theorem, so we can apply the Gaussian mechanism with appropriate scaling to obtain a private estimator.
Since the value of this gap ultimately impacts the quality of the output, we can privately verify whether it is sufficiently large, and return $\bot$ otherwise.
%Since later we divide by this value, we need to make sure that the value we get is not too small (and, in particular, not negative). Note that if $\sigma_r(M)-\sigma_{r+1}(M)$ is large enough, it can only happen with small probability, and if it is not large enough (for example, if it is some not very large constant), then our bound in \cref{thm:main} is trivial. So we can simply threshold this value and require it to be not very small (if it is smaller than the threshold, we return $\bot$). Hence from now on we can assume that the value $\sigma_r(M)-\sigma_{r+1}(M)$ is large enough and can be used by our private algorithm, and further we drop $M$ and denote this value $\sigma_r-\sigma_{r+1}$. \gleb{not sure how clear that paragraph was...}

Privatizing $\coh{r}{M}$ proves  more challenging. First, note that in general, coherence exhibits unpredictable behavior under matrix addition.
 For instance, a random Gaussian matrix with i.i.d. entries has small coherence (on the order of $\log n$), and flipping the sign of a diagonal entry does not significantly change it. 
However, subtracting one such matrix from another yields a matrix with a single nonzero entry, resulting in maximal coherence.
Fortunately, it is possible to show that the coherence values of adjacent inputs cannot differ significantly. Concretely, we show that
   \[
   \coh{r}{M + E} \le \Paren{1 + O\Paren{\frac{1}{\sigma_{r} - \sigma_{r+1}}}}\coh{r}{M}\,.
   \]
Note that with this multiplicative bound, one cannot directly apply the Gaussian mechanism to $\coh{r}{M}$ as its sensitivity depends on the value of $\coh{r}{M}$ itself and can be very large for certain inputs.
Nevertheless, this bound implies that $\log\paren{\coh{r}{M}}$ has sensitivity at most $\log\paren{1 + O\paren{1/(\sigma_{r} - \sigma_{r+1}}}$.
Applying the Gaussian mechanism with scale $O\Paren{\log\Paren{1 + O\Paren{\frac{1}{\sigma_{r} - \sigma_{r+1}}}}\cdot \frac{\sqrt{\log(1/\delta)}}{\eps}}$ to $\log\paren{\coh{r}{M}}$, yields a private estimator $\bm \ell$ of $\log\paren{\coh{r}{M}}$. That is, $\exp\paren{\bm \ell}$ serves as a private estimator of $\coh{r}{M}$ itself. 
At first glance, the error of this estimator may appear problematic, since both $\e$ and $\sqrt{\log\paren{1/\delta}}$ appear in the exponent. 
However, recall that we required $\sigma_{r} - \sigma_{r+1}$ to be small. 
In particular, this implies that
\[
\log\Paren{1 + O\Paren{\frac{1}{\sigma_{r} - \sigma_{r+1}}}}\cdot \frac{\sqrt{\log(1/\delta)}}{\eps} 
\le 
O\Paren{\frac{1}{\sigma_{r} - \sigma_{r+1}}}\cdot \frac{\sqrt{\log(1/\delta)}}{\eps} 
\le 
\log\Paren{1 +O\Paren{\frac{\sqrt{\log(1/\delta)}}{\eps \paren{\sigma_{r} - \sigma_{r+1}} }}}\,.
\]
The value $O(\sqrt{\log(1/\delta)}/\eps \paren{\sigma_{r} - \sigma_{r+1}})$ is small --otherwise the  procedure for privatizing $\sigma_{r} - \sigma_{r+1}$ would have returned $\bot$. Hence, after exponentiating, we get 
$\tfrac{1}{2}\coh{r}{M} \le \exp\paren{\bm \ell} \le 2\coh{r}{M}$ with high probability. 

Finally, observe that the standard composition theorem does not work here, since if the inner algorithms fail, then the privacy of the outer algorithm (Gaussian mechanism) is not guaranteed. Fortunately, a variant of the Propose–Test–Release paradigm \cite{dwork2009differential} shows that the composition is still differentially private if we include the failure probabilities in the privacy budget (this results in the $\log(1/\delta)$ term added to $r\cdot \coh{r}{M}$ in our error bound). 

\paragraph{Coherence of our estimator.}
As mentioned earlier, bounding the coherence of adjacency matrices is a challenging task—even for purely random graphs—and becomes significantly harder for graphs with more intricate structure, such as random graphs with planted cliques or bicliques. Fortunately, a continuous analogue of this problem is more tractable: namely, the model $A + \mathbf{W}$, where $A \in \mathbb{R}^{n \times n}$ is a rank-$r$ matrix and $\mathbf{W} \sim N(0,1)^{n \times n}$. In this setting, we can derive a strong upper bound on the growth of coherence. Specifically, we show that
\[
\coh{r}{A + \mathbf{W}} \le O\Paren{\coh{r}{A} + \frac{\log(n)}{r}}
\]
with high probability.
We now briefly describe the main idea behind our analysis. It can be shown that for any matrix $E \in \mathbb{R}^{n \times n}$,
\[
\coh{r}{A + E} \le O\Paren{\coh{r}{A} + \frac{n}{r} \normcol{(\Id_n - P_{(r)}) \hat{U}_{(r)}}^2},
\]
where $P_{(r)}$ is the projector onto the column space of $A$, $\hat{U}_{(r)}$ is the matrix of the top $r$ left singular vectors of $A + E$, and $\normcol{B}$ denotes the maximum $\ell_2$ norm of the rows of $B$. Thus, it suffices to bound $\normcol{(\Id_n - P_{(r)}) \hat{U}_{(r)}}$ when $E = \mathbf{W}$.

To this end, we leverage the rotational symmetry of the Gaussian matrix. Consider a random rotation matrix $\mathbf{R} \in \mathbb{R}^{(n - r) \times (n - r)}$ acting on the orthogonal complement of the column space of $A$ and independent of $\mathbf{W}$. The distribution of $A + \mathbf{W}$ remains unchanged under this transformation, and hence so does the distribution of $\hat{U}_{(r)}$. In other words, $\hat{U}_{(r)}$ and $\mathbf{R} \hat{U}_{(r)}$ are identically distributed. Since $\mathbf{R}$ and $\hat{U}_{(r)}$ are independent, we may condition on $\hat{U}_{(r)}$ and treat it as fixed in the subsequent analysis. Therefore, it suffices to bound
\[
\normcol{(\Id_n - P_{(r)}) \mathbf{R} \hat{U}_{(r)}}.
\]
We show that the rows of $(\Id_n - P_{(r)}) \mathbf{R} \hat{U}_{(r)}$ satisfy strong Hanson–Wright-type concentration bounds, which yield sharp high-probability control of their norms.

In particular, when $A$ is the projector onto the top-$r$ singular subspace of an input matrix $M \in \mathbb{R}^{n \times n}$, our result implies that the estimator produced by our algorithm (as stated in \cref{thm:main}) is incoherent, provided that $M$ itself is incoherent. This holds because the estimator is obtained via the Gaussian mechanism.

Since this analysis relies heavily on the rotational symmetry of Gaussian matrices, it is not clear how to extend it to more structured or discrete settings—such as planted graph problems. Nonetheless, we conjecture that an analogous statement should hold in those cases as well.

\subsection{Private CSP solvers via differentially private PCA}
Our starting point towards \cref{thm:dp-max-cut-easy} --and its extensions to \twocsp and \maxbisection-- is the classic algorithm of Barak, Raghavendra and Steurer \cite{barak2011rounding} based on the sum-of-squares framework (we direct the unfamiliar reader to \cref{sec:sos-background}).
Given a graph $G$ with adjacency matrix $A$, normalized adjacency matrix $\bar{A}$ and maximum cut of value $\OPT,$ this allows one to find a cut with value $\OPT(1-\eta)$ in time $n^{O(1)}\cdot\exp\set{O(\simul_{\eta}(\bar{A})/\eta^2)}$ where $n$ is the number of vertices and $\simul_{\eta}(\bar{A})$ is the number of singular values of $\bar{A}$ larger than $\eta.$\footnote{More precisely, the running time of the algorithm in\cite{barak2011rounding} is parametrized by the number of eigenvalues larger than $\eta.$ For consistency, throughout the paper we limit our discussion to singular values. Nevertheless, the running time of our differentially private CSP solvers can also be parametrized by the number of eigenvalues larger than $\eta.$}

To make this algorithm differentially private, our plan is as follows: (1) construct a private estimate $\hat{\mathbf{A}}$ of the normalized adjacency matrix of $G;$ (2) construct a synthetic graph $\hat{\mathbf{G}}$ from $\hat{\mathbf{A}};$ and (3) run the aforementioned non-private algorithm on this synthetic graph.
For this strategy to succeed, we need  the normalized adjacency matrix of the resulting synthetic graph $\bar{A}(\hat{\mathbf{G}})$ to satisfy $\simul_{\eta}(\bar{A}(\hat{\mathbf{G}})\approx\simul_\eta(\bar{A})$ and the graph itself to remain similar to the original one in the sense: $\Norm{A(\hat{\mathbf{G}})-A}\leq \gamma$, for some small $\gamma\in(0,1)$.
Indeed, the first property ensures the running time remains $n^{O(1)}\exp\set{\simul_{\eta}(\bar{A})/\eta^2},$ and the second that the optimal cut in $G$ remains close to the optimal cut in $\hat{\mathbf{G}}$ as
\begin{align*}
    \forall x\in\set{-1,1}^n,\qquad \tfrac{1}{n}\abs{\iprod{x,Ax}-\iprod{x,A(\hat{\mathbf{G}})x}}\leq  \gamma.
\end{align*}

To address these requirements, we apply \cref{thm:main} to the normalized adjacency matrix of $G.$  However, note that the sensitivity of $\bar{A}$ is proportional to $1/d_{\min}(G),$ and thus varies accross adjacent inputs. 
To work around this, we privatize the degree  matrix $D$ of $G$ by adding Gaussian noise with variance $\log(1/\delta)/\eps^2,$ yielding $\hat{\mathbf{ D}}$.
If the degrees are at least of order $\tau:=\sqrt{\log (n)\log (1/\delta)}/\eps,$ then this perturbation will only alters them by at most a constant factor, in the sense that:
\begin{align}\label{eq:techniques-infty}
    \normi{D-\hat{\mathbf{D}}}\leq O(\sqrt{\log (n)\log (1/\delta)}/\eps) \leq \tfrac{1}{2}\normi{D}.
\end{align}
This, in turn, allows us to estimate the sensitivity of $\bar{A}$ and  construct a privatization mechanism based on that estimate. Moreover, if $d_{\min}(G)\geq\tau/\gamma'$ for some small $\gamma',$ then \cref{thm:main} produces a rank-$r$ matrix $\hat{\mathbf{A}}'$ satisfying
$\Norm{\hat{\mathbf{A}}' - \bar{A}}\leq \sigma_{r+1} + \gamma',$ which implies by Weyl's inequality that $\simul_{\eta-\gamma'}(\hat{\mathbf{A}}')\leq \simul_\eta(\bar{A}).$ That is, roughly the required bounds.

While this represents significant progress, the matrix $\hat{\mathbf{A}}'$ may contain negative entries and thus cannot be directly used to construct a graph.
Let $\hat{\mathbf{A}}$ be the projection of $\hat{\mathbf{A}}'$ onto the intersection of the spectral norm ball of radius $\Norm{\hat{\mathbf{A}}'}+\sigma_{r+1} + \gamma'$ and the set of matrices with non-negative entries, which contains $\bar{A}$ by construction. Then, by triangle inequality,
\begin{align}\label{eq:techniques-spectral}
    \Norm{\hat{\mathbf{A}}'-\bar{A}}\leq \Norm{\hat{\mathbf{A}}'-\hat{\mathbf{A}}}+\Norm{\hat{\mathbf{A}}-\bar{A}}\leq 2\Norm{\hat{\mathbf{A}}-\bar{A}}\leq (\sigma_{r+1} + \gamma').
\end{align}
Because $\Norm{\hat{\mathbf{A}}'-\bar{A}}$ is small, it follows that with high probability the maximum cut of $G$ and the graph $\hat{\mathbf{G}},$ whose adjacency matrix  is $\hat{\mathbf{D}}^{1/2}\hat{\mathbf{A}}\hat{\mathbf{D}}^{1/2},$ are closely related. Specifically, for any $x\in\set{\pm 1}^n,$ we have
\begin{align*}
    \iprod{x, (A-\hat{\mathbf{D}}^{1/2}\hat{\mathbf{A}}\hat{\mathbf{D}}^{1/2})x} &=\iprod{x, (D^{1/2}\bar{A}D^{1/2} -(D^{1/2}-D^{1/2} + \hat{\mathbf{D}}^{1/2})\hat{\mathbf{A}}(D^{1/2}-D^{1/2} + \hat{\mathbf{D}}^{1/2}))x}\\
    &\leq \iprod{x, D^{1/2}(A-\hat{\mathbf{A}})D^{1/2}x} + \iprod{x,(\hat{\mathbf{D}}^{1/2}-D^{1/2})\hat{\mathbf{A}}(\hat{\mathbf{D}}^{1/2}-D^{1/2})x}
\end{align*}
Here the first term is small by \cref{eq:techniques-spectral} and the second by \cref{eq:techniques-infty}. 
%\tom{FIx the next paragraph}
%Unfortunately, however, such projection may unpredictably increase the coherence, thereby inflating the runtime of our algorithm.
%To work around this obstacle our approach is to show that the local correlation on $\hat{\mathbf{A}}$ is close to the local correlation on $\hat{\mathbf{A}}'.$ This allows us to work directly with $\hat{\mathbf{A}}$ and $\hat{\mathbf{G}}$, thus resolving the issue and completing the algorithm.

\section{Future work}\label{sec:future-work}
Apart from the challenging graph-related problems discussed in \cref{sec:conjecture}, there are a few other interesting unresolved directions.
\cite{balcan2016improved} showed that the algorithm of Hardt and Price~\cite{hardt2014noisy} enjoys stronger guarantees than those of \cref{thm:hardt-price} when \( M \in \mathbb{R}^{n \times n} \) is positive semidefinite. Specifically, in the error bound, the denominator \( \sigma_r - \sigma_{r+1} \) can be replaced by \( \sigma_r - \sigma_{r'+1} \). It would be interesting to investigate whether a similar improvement is possible in the general case—or at least in the PSD case—while still maintaining the strong guarantees of our \cref{thm:main}.

A second direction concerns an even weaker notion of coherence studied in~\cite{candes2012exact}. This version exploits the fact that, for random asymmetric matrices, the left and right singular vectors are uncorrelated. Our approach—as well as any other method based on symmetrization—fails to capture this phenomenon. It would be interesting to explore whether the guarantees of \cref{thm:main} can be extended to this weaker notion of coherence.

\section{Preliminaries}\label{sec:preliminaries}
We denote random variables in \textbf{boldface}.  For simplicity of the exposition, we ignore bit complexity issues and assume all quantities to be polynomially bounded in the ambient dimension. We write $\Tilde{O}$ to hide polylogarithmic factors. 
For a vector $v\in\R^{n}$ we write $\Norm{v}$ for its Euclidean norm and $\Normo{v}$ for its $\ell_1$-norm. We write $\mathbb{1}_n$ for the $n$-dimensional all ones vector and $\Id_n$ for the $n$-by-$n$ identity matrix.
For a matrix $A\in\R^{n\times n}$, let %$\lambda_1(A)\geq \ldots\geq \lambda_n(A)$ be its eigenvalues and 
be its $\sigma_1(A)\geq\ldots\geq \sigma_n(A)$ its singular values. We write $A\sge 0$ to denote that the matrix is positive semidefinite. 
%We denote the corresponding eigenvectors as $v_1(A),\ldots, v_n(A)\,.$ 
We denote by $\Norm{A}$ the spectral norm of $A$, by $\Normf{A}$ the Frobenius norm, by\footnote{Sometimes the notation $\normo{A}$ is used for another matrix norm. In this paper this notation always means the sum of absolute values of all entries of $A$.} $\Normo{A}=\sum_\ij \Abs{A_\ij}$, by $\Normi{A}=\max_{\ij}\Abs{A_\ij}$, and by $\normcol{A}$ the maximal $\ell_2$ norm of the rows of $A$.
We let $A_{(k)}$ be the rank-$k$ matrix minimizing $\Normf{A-A_{(k)}}.$
We define $D(A)\in\R^{n\times n}$ to be the diagonal matrix with entries $D(A)_{ii}=\Normo{A_i}.$ We write $D(A)^{-1/2}\in \R^{n\times n}$ for  the matrix entries
\begin{align*}
    D(A)^{-1/2}_{ii}=
    \begin{cases}
        0&\text{ if }D(A)_{ii}=0,\\
        1/\sqrt{D(A)_{ii}}& \text{otherwise.}    
    \end{cases}
\end{align*}
We write $\bar{A}$ for the matrix $D(A)^{-1/2}AD(A)^{-1/2}.$

%We use $\mathbb Z_{+}$ for the set of positive integers.  When the context is clear we do not specify the matrix. We write $A\tensor B$ for the tensor product of $A$ and $B.$
\begin{definition}[$\tau$-threshold  rank]
    For a  matrix $A\in\R^{n\times n}$, the $\tau$-threshold rank $\simul_{\tau}(A)$ is the number of singular values of value at least $\tau.$% in absolute value. For a matrix $M\in\R^{n\times n},$ $\simul_\tau(M)$ is the number of singular values larger or equal to $\tau.$
\end{definition}

We denote by $N(\mu,\Sigma)$ the multivariate Gaussian distribution with mean $\mu\in\R^n$ and covariance $\Sigma\in\R^{n\times n}.$ We say that an event holds with high probability if it holds with probability $1-o_n(1).$ We do not specify the subscripts when the context is clear.
A weighted graph is a triplet $G=(V,E,w)$ where $w:V\times V\to\R_{\ge 0}$ is the weight function. We denote by $w(G)$ the total weight of the edges in $G.$
For  graph $G$, the degree of vertex $i\in V(G)$, denoted by $d_G(i)$ is the sum of the weights of its edges. We define $d_{\min}(G)=\min_{v\in V(G)}d_G(v).$ %We will typically use $n$ to denote the number of vertices and equate $V(G)$ to $[n].$
We write $A(G)\in\R^{n\times n}$ for the  adjacency matrix of $G$ and  $D(G)\in\R^{n\times n}$ for the diagonal matrix with entries $D_{ii}=d_G(i).$
We denote the neighborhood of $i\in V(G)$ by $N_G(i).$
%We write $d_{\min}(G)$ %, %$ d_{\max}(G)$ and $d_{\text{avg}}(G)$ 
%for the minimum %, the maximum and the average 
%degree in $G.$ We also write $\vol_G(S)$ for the sum of the degrees of vertices in $S\subseteq V(G)$ and $\partial_G(S)$ for the set of edges in $G$ with exactly one endpoint in $S.$
%When the context is clear we do not specify the graph.
The normalized adjacency matrix of $G$ is then $\Bar{A}:=D^{-1/2}AD^{-1/2}.$
Note that by construction $\Norm{\bar{A}}\leq 1.$ 
We often refer to the singular values/vectors of $\bar{A}$ as the singular values/vectors of $G$ and denote them by %$\bar{\lambda}_1(G)\geq\ldots\geq \bar{\lambda}_n(G).$ Similarly, we write 
$\sigma_1(G),\ldots,\sigma_n(G).$ %for the singular values of $\bar{A}.$ %In particular, we say $G$ has $(r,\Lambda)$-spectral gap if $\bar{A}(G)$ has $(r,\Lambda)$-spectral gap and 
We write $\coh{r}{G}$ for the coherence of its normalized adjacency matrix.  %The eigenvalues of $A(G)$ and $\bar{A}(G)$ exhibit the following relationship.\tom{Not sure this is used anywhere right now}
%\begin{fact}\label{fact:normalied-eigenvalues}
%    Let $A\in\R^{n\times n}$ be a symmetric matrix with non-negative entries. Let $d_{\min} :=\min_i D(A)_{ii}$ and $d_{\max} :=\max_i D(A)_{ii}.$ For any $r\geq 1,$ 
%    \begin{align*}
%        \frac{\lambda_r(A)}{d_{\max}}\leq \lambda_r(\bar{A})\leq \frac{\lambda_r(A)}{d_{\min}}.
%    \end{align*}
%\end{fact}
%We consider both simple graphs, as well as graphs with non-negative edge weights and directed edges. Whenever unspecified we assume the graph at hand to be unweighted and undirected. A weighted graph is a triplet $G=(V,E,w)$ where $w:V\times V\to\R_{\ge 0}$ is the weight function. We denote undirected edges with the notation $\set{ij}$ or simply $\ij$ and directed edges using the notation $(i,j).$
We consider both simple graphs, as well as graphs with non-negative edge weights and self-loops. %Whenever unspecified we assume the graph at hand to be unweighted.  %We denote undirected edges with the notation $\set{ij}$ or simply $\ij$ and directed edges using the notation $(i,j).$

\paragraph{Differential privacy} We use the following definitions of adjacency:
\begin{definition}[Matrix adjacency]\label{def:matrix-adjacency}
    Two matrices $A,A'\in \R^{n\times n}$ are $\Delta$-adjacent if 
    \[
    \sqrt{\normo{EE^\top}} :=\sqrt{\sum\limits_{1\le i,j\le n}\abs{\paren{EE^\top}_{ij}}}\le \Delta\,,
    \]
    where $E = A' - A$. This definition generalizes the standard notion one entry adjacency. See \cref{sec:adjacency} for comparison of different notions of matrix adjacency in the context of private low rank approximation.
\end{definition}
\begin{definition}[Graph adjacency]\label{def:graph-adjacency}
    Two $n$-vertices graphs $G,G'$ are adjacent if they differ in at most one edge. %\gleb{Should we change it?}\tom{No because if we have this then we have the others for the normalized adj. matrix}
\end{definition}
\noindent Note that this implies the corresponding adjacency matrix are $2$-adjacent. We introduce standard differential privacy definitions and mechanisms in \cref{sec:dp-background} and the necessary sum-of-squares background in \cref{sec:sos-background}.

% To add if we end up using Laplace
\iffalse

For function with bounded $\ell_1$-sensitivity the Laplace mechanism is  often the tool of choice to achieve privacy.

It is also useful to consider the "truncated" version of the Laplace distribution where the noise distribution is shifted and truncated to be non-positive.

\begin{definition}[Truncated Laplace distribution]\label[definition]{definition:truncated-laplace-distribution}
	The (negatively) truncated Laplace distribution w with mean $\mu$ and parameter $b$ on $\R$, denoted by $\text{tLap}(\mu, b)$, is defined as $\text{Lap}(\mu, b)$ conditioned on the value being non-positive.
\end{definition}

\begin{lemma}[Truncated Laplace mechanism]\label{lemma:truncated-laplace-mechanism}
	Let $f:\cY\rightarrow \R$ be any function with $\ell_1$-sensitivity at most $\Delta_{f,1}$. Then the algorithm that adds $\text{tLap}\Paren{-\Delta_{f, 1}\Paren{1+\frac{\log(1/\delta)}{\eps}}, \Delta_{f, 1}/\eps}$  to $f$ is $(\eps, \delta)$-DP.
\end{lemma}

The following tail bound is useful when reasoning about truncated Laplace random variables.

\begin{lemma}[Tail bound truncated Laplace]\label{lemma:tail-bound-truncated-laplace}
	Suppose $\mu <0$ and $b> 0$. Let $\bm x \sim \text{tLap}(\mu, b)$. Then,for $y < \mu$ we have that
	\begin{align*}
		\bbP \Brac{\bm x < y}\leq \frac{e^{(y-\mu/b)}}{2-e^{\mu/b}}\,.
	\end{align*}
\end{lemma}

In constrast, when the function has bounded $\ell_2$-sensitivity, the Gaussian mechanism provides privacy.
\fi
\section{Differentialy private low rank matrix estimation}\label{sec:privatization-mechanism}

In this section we prove \cref{thm:main}. We restate a more general version in this section.

\begin{theorem}\label{thm:main-coherence}
    Let $M \in \R^{n\times n}$ be a symmetric matrix, and $p \le \delta/10$. There exists an efficient, $(\eps,\delta)$-differentially private algorithm (with respect to \cref{def:matrix-adjacency} of $\Delta$-adjacency) that, given $M\in\R^{n\times n},r, \Delta$  returns a rank-$r$ symmetric matrix $\hat{\mathbf{M}}_{(r)}\in \R^{n\times n}$ such that with probability at least $1-p-2^{-n}$,
    \begin{align*}
        \Norm{M-\hat{\mathbf{M}}_{(r)}}\leq \sigma_{r+1}+O\Paren{\sigma_1\cdot\frac{\Delta\sqrt{r \cdot \coh{r}{M} + \log\paren{1/p} }}{\sigma_r-\sigma_{r+1}} \cdot \frac{\sqrt{\log(1/\delta)}}{\eps}}\,,
    \end{align*}
    and 
    \[
           \Norm{\paren{\Id_n - \hat{\mathbf P}_{(r)}} U_{(r)}}
       \leq O\Paren{\frac{\Delta\sqrt{r \cdot \mu_r(M) + \log\paren{1/p} } }{\sigma_r - \sigma_{r+1}} 
       \cdot \frac{\sqrt{\log(1/\delta)}}{\eps}}\,,
    \]
    where $\hat{\mathbf P}_{(r)}$ is the projector onto the column span of $\hat{\mathbf{M}}_{(r)}$, and $U_{(r)}$ is a matrix whose columns are $r$ leading singular vectors of $M$. In addition, 
    \[
     { \coh{r}{\hat{\mathbf{M}}_{(r)}}} 
    \le O\Paren{{\coh{r}{M}} +{\frac{\log\paren{n/p}}{r}} } \,.
    \]
    
    Furthermore, there exist absolute constants $C, C'$, such that if
    \[
    \frac{\Delta\sqrt{r \cdot \coh{r}{M} + \log\paren{1/p} }}{\sigma_r-\sigma_{r+1}} \cdot \frac{\sqrt{\log(1/\delta)}}{\eps} \leq 1/C'\,,
    \]
    then
        \[
    { \coh{r}{\hat{\mathbf{M}}_{(r)}}}  \ge \frac{1}{C}\sqrt{\coh{r}{M}} - 
    C{\frac{\log\paren{n/p}}{r}}\,.
    \]
\end{theorem}

We will prove this theorem in several steps. First, we need to estimate the spectral gap.

\subsection{Differentially private spectral gap estimation}

By Weyl's theorem, the sensitivity of the spectral gap is bounded by $2\Delta$, and hence the gap can be estimated via Gaussian mechanism. Conceretely, we use the following algorithm:

\begin{algorithmbox}\label{alg:private-gap}
    \textsc{PrivateGapEstimator}\\
    %\mbox{}\\
    \textbf{Input:} $\e,\delta$,  failure probability $0 < p < 1/2$, symmetric matrix $M \in \mathbb{R}^{n \times n}$, $r\in [n]$, $\Delta > 0$.\\
    \textbf{Hyperparameter:} Absolute constant $C'$.\\
    \textbf{Output:} Gap estimator $\hat{\bm\gamma}$.
    \begin{enumerate}
    \item Compute $\sigma_r$ and $\sigma_{r+1}$.
        \item $\hat{\bm \gamma} \gets \sigma_r - \sigma_{r+1} + \mathbf g$, 
        where $ \mathbf g\sim N\Paren{0, C'\cdot \frac{\Delta^2{\log(1/\delta)}}{\eps^2}}$.
        \item If 
        $\hat{\bm \gamma} < C'\cdot \frac{\Delta\sqrt{\log(1/\delta)}}{\eps}\cdot\sqrt{\log\paren{1/p}}$,
        return $\bot$.
        \item return $\hat{\bm \gamma}$.
    \end{enumerate}
\end{algorithmbox}

\begin{lemma}\label{lem:private-gap}
If $C'$ is large enough, then
    \cref{alg:private-gap} is $(\eps,\delta)$-differentially private, and 
    if 
    \[
    \frac{\Delta}{\sigma_r-\sigma_{r+1}} \cdot \frac{\sqrt{{\log\paren{1/p} }\cdot \log(1/\delta)}}{\eps} \le 1/C'\,,
    \]
    then
    with probability $1-p$ its output $\hat{\bm \gamma}$ satisfies
    \[
    \frac{1}{2} \Paren{\sigma_r - \sigma_{r+1}} \le \hat{\bm \gamma} \le 2\Paren{\sigma_r - \sigma_{r+1}}\,.
    \]
\end{lemma}
\begin{proof}
    By Weyl's theorem \cref{thm:weyl}, $\sigma_r - \sigma_{r+1}$ has $\ell_2$ sensitivity $2\Delta$. Hence by \cref{thm:gaussian-mechanism}, the algorithm is $(\eps,\delta)$-differentially private. Since with probability $1-p$, 
    \[
    \abs{\mathbf g} \le O\Paren{\sqrt{C'} \cdot \frac{\Delta\sqrt{\log(1/\delta)}}{\eps}\cdot\sqrt{\log\paren{1/p}}}\,,
    \]
    we get the desired bound.
\end{proof}
%Next, we privatize the coherence.

\subsection{Differentially private coherence estimation}

Unlike the spectral gap, the coherence does not necessarily have good enough $\ell_2$ sensitivity. First, we show that it can only change by a factor close to $1$:

\begin{lemma}[Sensitivity of Coherence]\label{lem:coherence-sensitivity}
    Let $M \in \R^{n\times n}$ be a symmetric matrix, and let $E$ satisfy $\sqrt{\sum_{ij}\abs{\paren{EE^\top}_{ij}}}\le \Delta$.
    If $\sigma_r-\sigma_{r+1} > 2\Delta$, then 
   \[
   \coh{r}{M + E} \le \Paren{1 + O\Paren{\frac{\Delta}{\sigma_{r} - \sigma_{r+1}}}}\coh{r}{M}\,.
   \]
\end{lemma}

We prove this lemma in \cref{sec:deferred-proofs-coherence}.

\cref{lem:coherence-sensitivity} allows us to use Gaussian mechanism on logarithm of the coherence. Concretely, we use the following algorithm:

\begin{algorithmbox}\label{alg:private-coherence}
    \textsc{PrivateCoherenceEstimator}\\
    %\mbox{}\\
    \textbf{Input:} $\e,\delta$, failure probability $0 < p < 1/2$, symmetric matrix $M \in \mathbb{R}^{n \times n}$, $r\in [n]$, $\Delta > 0$.\\
    \textbf{Hyperparameter:} Absolute constant $C'$.\\
    \textbf{Output:} Coherence estimator $\hat{\bm \mu}$.
    \begin{enumerate}
        \item $\hat{\bm \gamma} \gets \textsc{PrivateGapEstimator}(\e/2,\delta/2, p/2, M, r, \Delta)$. If $\hat{\bm \gamma} = \bot$, return $\bot$.
        \item Compute $\coh{r}{M}$.
        \item $\bm \ell \gets \log\Paren{\coh{r}{M}} + \mathbf w$, 
        where $\mathbf w \sim N\Paren{0, C'\cdot \frac{\Delta^2{\log(1/\delta)}}{\hat{\bm \gamma}^2\eps^2}}$
        \item return $\exp\paren{\bm \ell}$.
    \end{enumerate}
\end{algorithmbox}

\begin{lemma}\label{lem:private-coherence}
If $C'$ is large enough, then
    \cref{alg:private-coherence} is $(\eps,\delta)$-differentially private, and if 
    \[
    \frac{\Delta}{\sigma_r-\sigma_{r+1}} \cdot \frac{\sqrt{{\log\paren{1/p} }\cdot \log(1/\delta)}}{\eps} \le 1/C'\,,
    \]
    then with probability $1-p$ its output $\hat{\bm \mu}$ satisfies
        \[
    \frac{1}{2}{\coh{r}{M}}
    \le {\hat{\bm \mu}} 
    \le 2{\coh{r}{M}}\,.
    \]
\end{lemma}
\begin{proof}
    By \cref{lem:coherence-sensitivity}, $\bm \ell$ has sensitivity $\log\Paren{1 + O\Paren{\frac{\Delta}{\sigma_{r} - \sigma_{r+1}}}} \le O\Paren{\frac{\Delta}{\sigma_{r} - \sigma_{r+1}}}$. By \cref{lem:private-gap}, with probability $1-p/2$, 
        \[
    \frac{1}{2} \Paren{\sigma_r - \sigma_{r+1}} \le \hat{\bm \gamma} \le 2\Paren{\sigma_r - \sigma_{r+1}}\,,
    \]
    hence the sensitivity of $\bm \ell$ is bounded by $O\Paren{{\Delta}/{ \hat{\bm \gamma}}}$, and
    by \cref{thm:gaussian-mechanism}, the algorithm is $(\eps,\delta)$-differentially private. 
    Since with probability $1-p/2$
    \[
    \abs{\mathbf w} \le O\Paren{\sqrt{C'} \cdot \frac{\Delta\sqrt{\log(1/\delta)}}{\hat{\bm \gamma}\eps}\cdot\sqrt{\log\paren{1/p}}}\le \frac{1}{\sqrt{C'}} \le \abs{\log\Paren{1/2}}\,,
    \]
    using \cref{lem:dp-comp}, we get the desired bound.
\end{proof}

\subsection{Differentially private projector estimation}
In this section we privitely estimate the projector onto the space spanned by leading singular vectors of $M$. We use Gaussian mechanism and private estimations of the parameters (the spectral gap and the coherence). By \cref{lem:dp-comp}, the resulting estimator is private.

First let us show that under our assumptions, projectors onto singular spaces have small sensitivity.

\begin{lemma}[Sensitivity of Projectors]\label{lem:projector-sensitivity}
     Let $M,M' \in \R^{n\times n}$ be symmetric matrices, and let $E = M-M'$ satisfy $\sqrt{\sum_{ij}\abs{\paren{EE^\top}_{ij}}}\le \Delta$.
     Let $P\in \R^{d\times d}$ and $P'\in \R^{d\times d}$ be the orthogonal projectors onto leading $r$-dimensional singular spaces of $M$ and $M'$ respectively. If $\sigma_r-\sigma_{r+1} > 2\Delta$, then
    \[
    \normf{P-P'} \le \frac{4\Delta\sqrt{r \coh{r}{M}/n}}{\sigma_r - \sigma_{r+1}}\,.
    \]
\end{lemma}
\begin{proof}
    By \cref{fact:perturbation},
    \[
     \normf{P-P'} \le \frac{4\normf{E U}}{\sigma_r - \sigma_{r+1}} 
    \]
    Bu H\"older's inequality,
    \[
     \normf{E U} = \sqrt{\Iprod{EU, EU}}
     = \sqrt{\Iprod{UU^\top, E^\top E}}
     \le \sqrt{\normo{E^\top E}} \cdot \sqrt{\normi{P}}\,,
    \]
    hence we get the desired bound.

    % Let $a, b\in[n]$ be the row and the column of $E$ with nonzero entries. Then $\Paren{E^\top E}_{bb} = \sum_{k=1}^n E_{kb}^2 \le \normo{E}^2$, and if $i\neq b$ or $j\neq b$, $\Paren{E^\top E}_{ij} = E_{ai}E_{aj}$. Hence $\normo{E^\top E} \le \normo{E}^2 + \sum_{1\le i,j\le n} E_{ai}E_{aj} \le 2\normo{E}^2$. Since $\normi{P} = \frac{r}{n}\coh{r}{M}$, we get the desired bound.
\end{proof}

    Hence if we apply the Gaussian mechanism to the projector, we get a private matrix. Our private estimator is the projector onto the space of leading singular vectors of the resulting matrix. Concretely, we use the following algorithm:

\begin{algorithmbox}\label{alg:private-projector}
    \textsc{PrivateProjectorEstimator}\\
    %\mbox{}\\
    \textbf{Input:} $\e,\delta$, failure probability $0 < p < 1/2$, symmetric matrix $M \in \mathbb{R}^{n \times n}$, $r\in [n]$, $\Delta > 0$.\\
    \textbf{Hyperparameter:} absolute constant $C$, projector $\mathbf R$ onto a random $r$-dimensional subspaces for default output.\\
    \textbf{Output:} Projector $\hat{\mathbf P}$.
    \begin{enumerate}
        \item $\hat{\bm \gamma} \gets \textsc{PrivateGapEstimator}(\eps/4, \delta/4, p/4, M, r, \Delta)$. If $\hat{\bm \gamma} =\bot$, return $\mathbf R$.
        \item $\hat{\bm \mu} \gets \textsc{PrivateCoherenceEstimator}(\eps/4, \delta/4, p/4, M, r, \Delta, \hat{\bm \gamma})$. If $\hat{\bm \mu} =\bot$, return $\mathbf R$ .
        \item Compute the projector $P$ onto the space spanned by the top $r$ singular vectors of $M$.
        \item $\mathbf S \gets P + \mathbf G$,
        where $\mathbf G \sim N\Paren{0, C \cdot \frac{\Delta\sqrt{r \hat{{\bm \mu}}}}{\sqrt{n}\hat{\bm \gamma}} \cdot \frac{\sqrt{\log(1/\delta)}}{\eps}}^{n\times n}$. 
        \item Return the projector $\hat{\mathbf P}$ onto the space spanned by the top $r$ left singular vectors of $\mathbf S$.
    \end{enumerate}
\end{algorithmbox}

\begin{lemma}\label{lem:private-projector}
    If $C$ is large enough, then
    \cref{alg:private-projector} is $(\eps,\delta)$-differentially private, and
    with probability $1-p$ its output $\hat{\mathbf P}$ satisfies
    \[
           \Norm{\paren{\Id_n - \hat{\mathbf P}} U_{(r)}}
       \leq O\Paren{\frac{\Delta\sqrt{r \cdot \mu_r(M) + \log\paren{1/p}} }{\sigma_r - \sigma_{r+1}} 
       \cdot \frac{\sqrt{\log(1/\delta)}}{\eps}}\,,
    \]
    where $U_{(r)}$ is a matrix whose columns are $r$ leading singular vectors of $M$.
\end{lemma}
\begin{proof}
Note that 
    if
    \[
    \frac{\Delta}{\sigma_r-\sigma_{r+1}} \cdot \frac{\sqrt{{\log\paren{1/p} }\cdot \log(1/\delta)}}{\eps} \ge 1/C'\,,
    \]
    then the error bound is true. Hence further we assume that this value is smaller than $1/C'$.
    By \cref{lem:private-gap} and \cref{lem:private-coherence}, with probability $1-p/2$, $\hat{\bm \gamma}$ and $\hat{\bm \mu}$ differ from the true values by factor at most $2$. 
    By \cref{lem:projector-sensitivity}, the sensitivity of $P$ is at most $s=O\Paren{\frac{\Delta\sqrt{r \hat{{\bm \mu}} }}{\sqrt{n}\hat{\bm \gamma}}}$. 
    Hence by \cref{lem:dp-comp}, if we use Gaussian mechanism $\mathbf G$ with scale $\rho_1 \gtrsim s\cdot \frac{\sqrt{\log(1/\delta)}}{\eps}$, we get an $(\e,\delta)$-private output matrix $P + \mathbf G$. Let $\hat{\mathbf P}$ be the orthogonal projector onto the space spanned by leading $r$ left singular vectors of $P + \mathbf G$. 
    By \cref{fact:perturbation} and the concentration of the spectral norm of Gaussian matrices, with probability at least $1-2^{-n}-p/2$,
    \[
    \Norm{\paren{\Id_n - \hat{\mathbf P}} U_{(r)}} \le \norm{\mathbf G} \le O\paren{\rho_1\sqrt{n}}\,.
    \]
    Plugging the value of $\rho_1$ into this expression, we get the desired bound.
\end{proof}

%Next we use the private estimator of the projector to construct the low rank approximation  $\hat{\mathbf{M}}_{(r)}$.

\subsection{Differentially private low rank estimation}

For the low rank estimation, we use the following algorithm:

\begin{algorithmbox}\label{alg:private-low-rank}
    \textsc{PrivateLowRankEstimator}\\
    %\mbox{}\\
    \textbf{Input:} $\e,\delta$, failure probability $0 < p < 1/2$, symmetric matrix $M \in \mathbb{R}^{n \times n}$, $r\in [n]$, $\Delta > 0$.\\
    \textbf{Hyperparameter:} absolute constant $C$.\\
    \textbf{Output:} Rank $r$ symmetric matrix $\hat{\mathbf M}_{(r)}$.
    \begin{enumerate}
        \item $\hat{\mathbf P} \gets \textsc{PrivateProjectorEstimator}(\eps/2, \delta/2, p/2, M, r, \Delta)$. 
        \item Compute $\hat{\mathbf U}\in \R^{d\times r}$ 
        such that $\hat{\mathbf P}=\hat{\mathbf U}\hat{\mathbf U}^\top$ 
        and $\hat{\mathbf U}^\top\hat{\mathbf U}= \Id_r$.
        \item $\mathbf L \gets \hat{\mathbf U}^\top M  \hat{\mathbf U} \in \R^{r\times r}$.
         \item $\mathbf S \gets \mathbf L +\mathbf W$, where $\mathbf W = \mathbf W^\top$, 
         $\mathbf W_{ij}\simiid \Paren{C\cdot \frac{\Delta^2{\log(1/\delta)}}{\eps^2}}$ for $i\ge j$.
        \item Return $\hat{\mathbf{M}}_{(r)} = \hat{\mathbf U}\mathbf S \hat{\mathbf U}^\top$.
    \end{enumerate}
\end{algorithmbox}

\begin{lemma}\label{lem:private-low-rank}
        If $C$ is large enough, then
    \cref{alg:private-projector} is $(\eps,\delta)$-differentially private, and
    with probability $1-p$ its output $\hat{\mathbf{M}}_{(r)}$ satisfies
    \[
        \Norm{M-\hat{\mathbf{M}}_{(r)}}\leq \sigma_{r+1}+O\Paren{\sigma_1\cdot\frac{\sqrt{r \cdot \coh{r}{M} + \log\paren{1/p} }}{\sigma_r-\sigma_{r+1}} \cdot \frac{\Delta\sqrt{\log(1/\delta)}}{\eps}}\,,
    \]
    where $U_{(r)}$ is a matrix whose columns are $r$ leading singular vectors of $M$.
\end{lemma}

\begin{proof}
    Let $\hat{\mathbf U}\in \R^{d\times r}$ be such that $\hat{\mathbf P}=\hat{\mathbf U}\hat{\mathbf U}^\top$ and $\hat{\mathbf U}\hat{\mathbf U}^\top = \Id_r$. Consider $\mathbf L = \hat{\mathbf U}^\top M  \hat{\mathbf U} \in \R^{r\times r}$. Let us bound the sensitivity of $\mathbf L$:
    \[
    \normf{\hat{\mathbf U}^\top M' \hat{\mathbf U} - \hat{\mathbf U}^\top M  \hat{\mathbf U}} \le \normf{M'- M} \le \Delta\,.  
    \]
    Hence if we use symmetric Gaussian mechanism $\mathbf W$ with scale $\rho_2 \gtrsim \Delta\cdot \frac{\sqrt{\log(1/\delta)}}{\eps}$, we get an $(\e,\delta)$-private output matrix $\mathbf L + \mathbf W$. By the concentration of the spectral norm of Gaussian matrices, with probability at least $1-p/2$,
    \[
    \norm{\mathbf W} \le O\Paren{\rho_2 \sqrt{r + \log\paren{1/p}}}\,.
    \]

    Consider the estimator $\hat{\mathbf{M}}_{(r)} = \hat{\mathbf U} \Paren{\mathbf L+\mathbf W}\hat{\mathbf U}^\top = \hat{\mathbf P}\Paren{M + \mathbf W}\hat{\mathbf P}$. Let $\mathbf E = \hat{\mathbf P} - P$.
    Let us bound the error:
    \begin{align*}   
    \norm{M - \hat{\mathbf{M}}_{(r)}}
        &= \norm{M - \hat{\mathbf P}M\hat{\mathbf P} - \hat{\mathbf U}\mathbf W\hat{\mathbf U}^\top}
        \\&\le \norm{M - \Paren{P+\mathbf E} M\Paren{P+\mathbf E}}+ \norm{\mathbf W}
        \\&\le \norm{M - PMP} + \norm{\mathbf E MP} + \norm{MP \mathbf E } + \norm{\mathbf E M \mathbf E} 
        + O\Paren{\rho_2 \sqrt{r + \log\paren{1/p}}}
        \\&\le \sigma_{r+1} + 2\cdot \norm{\mathbf E}\cdot \norm{M} + \norm{\mathbf E}^2\cdot \norm{M} 
        + O\Paren{\rho_2 \sqrt{r + \log\paren{1/p}}}
        \\&\le \sigma_{r+1} + O\Paren{\sigma_1\rho_1 \sqrt{n}} + O\Paren{\rho_2 \sqrt{r + \log\paren{1/p}}}\,,
    \end{align*}
    where we used that $\norm{E}\le \min\Set{1,O\Paren{\rho_1 \sqrt{n}}}$. Plugging the values of $\rho_1$ and $\rho_2$ into this expression, we get the desired bound.
\end{proof}

\subsection{Bound on the coherence}

To finish the proof of \cref{thm:main-coherence}, we need to show that the coherence of $\mathbf{M}_{(r)}$ is close to the coherence of $M$.
The desired bound on the coherence of our estimator follows from the following theorem:

\begin{theorem}\label{thm:coherence-gaussian}
    Let $A\in \R^{n\times m}$ be a matrix of rank $r$,
    and $\mathbf W = \N(0, \sigma^2)^{n\times m}$ for some $\sigma > 0$. Let $r \le r' \le \rank\paren{A}$ be positive integers. For each $0 < p \le 1$, with probability $1-p$,
    \[
    {{r'}\coh{r'}{A+\mathbf  W}} 
    \le C{r\coh{r}{A}} +C\paren{r' + \log\paren{\paren{n+m}/p}} \,,
    \]
    where $C$ is some large enough absolute constant. 
    
    Furthermore, if $A = U\Sigma V^\top$ is the singular value decomposition of $A$, and $\norm{\paren{\Id_n - \hat{\mathbf P}}U}\le 0.99$ and $\norm{\paren{\Id_m - \hat{\mathbf Q}}V}\le 0.99$, where $\hat{\mathbf P}\in \R^{n\times n}$ and $\hat{\mathbf Q}\in \R^{m\times m}$ are orthogonal projectors onto the spaces spanned by leading $r'$ left and right singular vectors of $A + \mathbf W$, then
    \[
     {{r'}\coh{r'}{A+\mathbf  W}} \ge   \frac{1}{C}{r\coh{r}{A}} - C\paren{r' - \log\paren{\paren{n+m}/p}}\,.
    \]
\end{theorem}

We prove this theorem in \cref{sec:deferred-proofs-coherence}.

\section{Privatizing graphs with low coherence}\label{sec:threshold-rank}
In this section we use \cref{thm:main-coherence} to privatize graphs without significantly changing their coherence or perturbing their spectrum. 
Formally, we prove the following theorem. %which immediately implies \cref{thm:intro-threshold-rank}

%Recall we use $\cM_n$ for the set of $n$-by-$n$ doubly stochastic matrices.
%

\begin{theorem}\label{thm:threshold-rank}
    Let $\eps,\delta,\gamma\in [0,1]$  with $\delta \geq 10n^{-100},$ let $r>0$ be an integer. Let $C>0$ be a large enough constant.
    There exists a polynomial time $(\eps,\delta)$-DP algorithm that, given an $n$-vertex  graph $G, \eps,\delta,r$, with probability at least $1-n^{-O(1)}$ returns a symmetric $\rank$-$r$ matrix $\hat{\mathbf A}$ with the following guarantees. If $G$ has 
    \begin{align*}
        d_{\min}\geq C\Paren{\frac{\sqrt{\log(1/\delta)}}{\eps}\cdot \frac{\sqrt{r\cdot\mu_r+\log n}}{\gamma\cdot(\sigma_r(\bar{A}) -\sigma_{r+1}(\bar{A}))}}
    \end{align*}
    then
    \begin{enumerate}[(i)]
        \item $\Norm{\bar{A}_{(r)}-\hat{\mathbf A}}< \gamma$
        \item $\Omega\Paren{\sqrt{\mu_r}-\sqrt{\frac{\log n}{r}}}\leq \sqrt{\coh{r}{\hat{\mathbf A}}}\leq O\Paren{\sqrt{\mu_r}+\sqrt{\frac{\log n}{r}}}$
        \item $\simul_{\sigma_r(\bar{A})}(\hat{\mathbf{A}})\leq \simul_{\sigma_r(\bar{A})-\gamma}(\bar{A})$
    \end{enumerate}
\end{theorem}
To prove \cref{thm:threshold-rank} we will consider the following algorithm:

\begin{algorithmbox}\label{alg:threshold-rank}
    \mbox{}\\
    \textbf{Input:}  Graph $G\,, \eps,\delta,r$\\
    \textbf{Output:} $\hat{A}\in\R^{n\times n}\,.$
    \begin{enumerate}[(1)]
        \item Let $\hat{\mathbf{d}}_G=d_{\min}(G)+\mathbf{w}\sim N\Paren{0,10^6\frac{\log\Paren{\tfrac{4}{\delta}+n}}{\eps^2}}$. If $\hat{\mathbf{d}}_G\notin \Brac{ 4\cdot 10^3\frac{\sqrt{\log \Paren{\tfrac{4}{\delta}}\log\Paren{\tfrac{4}{\delta} + n}}}{\eps}, 2n^2}$ output $\bot$.
        \item Run the algorithm of \cref{thm:main-coherence} on input $\bar{A}(G)$ with parameters $\eps/2,\delta/4$, $r$, and $\bm \Delta= 16/(\lfloor\hat{\mathbf{d}}_G\rfloor-1).$ Return its output $\hat{\mathbf{A}}$.
        %\item Return $\hat{\mathbf{A}}:=\arg\min_{M} \Norm{\hat{\mathbf A}-M}\qquad\text{subj. to}\qquad  0\leq M_\ij\leq 1$
        %\item  Return $\hat{\mathbf{A}}:=\bar{\mathbf{M}}_{(2)}.$
    \end{enumerate}
\end{algorithmbox}
To study the guarantees of \cref{alg:threshold-rank}, we make use of the following statement which relates the $\ell_1$ distance of the normalized adjacency matrices of neighboring graphs with their minimum degree. 

\begin{fact}\label{fact:distance-normalized-adjacency-matrices}
    Let $G,G'$ be edge-adjacent $n$-vertex graphs.
    Then $\bar{A}(G),\bar{A}(G')$ are $\frac{8}{\min\set{d_{\min}(G),d_{\min}(G')}}$-adjacent per \cref{def:matrix-adjacency}.
\end{fact}
We prove \cref{fact:distance-normalized-adjacency-matrices} in \cref{sec:deferred-proofs}.
Differential privacy of \cref{alg:threshold-rank} is then direct consequence of the composition mechanism \cref{lem:dp-comp}.
\begin{lemma}\label{lem:dp-threshold-rank-alg}
    \cref{alg:threshold-rank} is $(\eps,\delta)$-edge-DP.
    \begin{proof}
        Let $G$ be the input graph and let $t:=6\cdot 10^3\frac{\sqrt{\log \Paren{\tfrac{4}{\delta}}\log\Paren{\tfrac{4}{\delta} + n}}}{\eps}.$ By the Gaussian mechanism, step (1) is $(\eps/2,\delta/3)$-DP.
        By concentration of the univariate Gaussian distribution, with probability at least $1-\delta/3$ it holds that $\Abs{d_{\min}(G)-\hat{\mathbf{d}}_G}\leq 10^3\frac{\sqrt{\log \Paren{\tfrac{4}{\delta}}\log\Paren{\tfrac{4}{\delta} + n}}}{\eps}=:t/4.$ %We may condition the rest of the analysis on this event since the measure of its complement is at most $\delta/2.$ 
        If the algorithm does not fail at step (1), then we have $d_{\min}(G)\geq 3t/4$
        and so $2d_{\min}(G)\geq \hat{\mathbf{d}}_G.$ Therefore, by \cref{fact:distance-normalized-adjacency-matrices}, for any $G'$ adjacent to $G$ we have that $\bar{A}(G),\bar{A}(G')$ are $\Paren{16/(\lfloor\hat{\mathbf{d}}_G\rfloor-1)}$-adjacent per \cref{def:matrix-adjacency}. The algorithm from \cref{thm:main-coherence} is $(\e/2, \delta/3)$-DP as long as $2d_{\min}(G)\geq \hat{\mathbf{d}}_G$, and the estimator of minimal degree is $(\e/2,\delta/3)$-DP. Hence by \cref{lem:dp-comp}, the composition algorithm is $(\e, \delta)$-DP.

    \end{proof}
\end{lemma}
We are ready to analyze the guarantees of \cref{alg:threshold-rank} and prove \cref{thm:threshold-rank}.

\begin{proof}[Proof of \cref{thm:threshold-rank}]
    By \cref{lem:dp-threshold-rank-alg} the algorithm is $(\eps,\delta)$ differentially private.
    So we only need to argue about its guarantees. Let $C>0$ be a large enough constant to be defined later.
    Notice that with probability at least $1-n^{-O(1)}$, we have $\hat{\mathbf{d}}_G\geq d_{\min}(G)-O\Paren{\frac{\sqrt{\log(2/\delta)\log (n)}}{\eps}}\geq \frac{d_{\min}(G)}{2}.$
    It follows by \cref{thm:main-coherence}
    \begin{align*}
        \Norm{\mathbf{A}-\bar{A}_{(r)}}&\leq   O\Paren{\frac{\sqrt{r\cdot\mu_r+\log(n)}}{d_{\min}\cdot(\sigma_r-\sigma_{r+1})}\cdot\frac{\sqrt{\log(1/\delta)}}{\eps}}\\
        &\leq O(\gamma/C)
    \end{align*}
    where we used the fact that $\Norm{\bar{A}}\leq 1,$ the assumption on $d_{\min}$ and \cref{fact:distance-normalized-adjacency-matrices} to bound the sensitivity.
    This implies \textit{(i)} for an appropriate choice of the constant $C.$
    By \cref{thm:main-coherence} \textit{(ii)} follows immediately.
    Finally, by Weyl's inequality we have
    \begin{align*}
        \sigma_{r+1}(\hat{\mathbf{A}})&\leq \sigma_{r+1}(\bar{A}) + \sigma_{1}(\hat{\mathbf A}-\bar{A})\\
        &= \sigma_{r+1}(\bar{A}) + \Norm{\hat{\mathbf{A}}-\bar{A}}\\
       &< \sigma_{r+1}(\bar{A})+\gamma.
    \end{align*}
    This implies \textit{(iii)} . %The bound on singular values can be proven similarly. Choosing the hidden constant $C>0$ large enough the result follows.
\end{proof}

\section{Solving CSPs under differential privacy}\label{sec:applications}
\tom{Reminder for myself: the use of $A'$ and $A$ is now unnecessary since we don't need to preserve coherence anymore. So we should remove it and simplify the proofs.}
We obtain here our differentially private applications. 
The section is organized as follows. In \cref{sec:correlation-rounding} we recap and extend the global correlation rounding framework of \cite{barak2011rounding, guruswami2011lasserre, raghavendra2012approximating} and apply our generalization to \maxcut, \twocsp and \maxbisection. Then we make these algorithms private respectively in \cref{sec:dp-max-cut}, \cref{sec:dp-max-two-csp} and \cref{sec:dp-max-bisection}. 

\subsection{Global correlation rounding}\label{sec:correlation-rounding}
We revisit here the global correlation rounding framework of \cite{barak2011rounding, guruswami2011lasserre, raghavendra2012approximating} obtaining algorithms for \maxcut and \twocsp.
We slightly extend the framework to make it work in the context of privacy. We emphasize that our algorithm remains the same as \cite{barak2011rounding}.

Necessary background on the sum-of-squares framework can be found in \cref{sec:sos-background}.
We index elements in $[n]\times [q]$ by pairs $i\ell$ with $i\in [n]$ and $\ell\in [q].$
Let $\cD_{nq}$ be the set of $nq$-by-$nq$ diagonal matrices with non-negative entries.
Observe that any $D\in \cD_{nq}$ induces a distribution over $[nq]$ where $i\ell$ is sampled with probability $D_{i\ell,i\ell}/\Normo{D}.$ With a slight abuse of notation we denote such distribution by $D.$ 
Similarly, any $A\in \R^{nq\times nq}$ induces a distribution over $[nq]\times [nq]$ where $\set{i\ell,j\ell'}$ us sampled with probability $\Abs{A_{i\ell,j\ell'}}/\Normo{A}.$ We write $\set{i\ell,j\ell'}\sim A$ to denote an entry sampled from the distribution induced by $A.$ 
Consider the following system of polynomial inequalities \(\mathcal{P}_{n,q}\) in variables $x_{11},x_{1q},\ldots,x_{n1},\ldots,x_{nq}:$ 
\begin{align}\label{eq:basic-sdp}%\tag{$\mathcal{P}_{n,q}$}
    \Set{
    \begin{aligned}
        &x_{i\ell}^2=x_{i\ell}&\forall i\in [n],\ell\in [q]\\
        &\sum_{\ell\in [q]}x_{i\ell}=1&\forall i\in [n]
    \end{aligned}
    }
\end{align}
We introduce the following well-known definitions.
\begin{definition}[Pseudo-covariance]
    Let $\zeta$ be a degree $2$-pseudo-distribution consistent with \ref{eq:basic-sdp}. The pseudo-covariance of $x_{i\ell},x_{j\ell'}$  is defined as $\widetilde{\Cov}(x_{i\ell},x_{j\ell'})= \tilde{\E}\Brac{x_{i\ell}x_{j\ell'}}-\tilde\E\Brac{x_{i\ell}}\tilde\E\Brac{x_{j\ell'}}.$
\end{definition}
\begin{definition}[Global correlation]
    Let $\zeta$ be a degree $2$-pseudo-distribution consistent with \ref{eq:basic-sdp} and let $D\in\cD_{nq}$. The global correlation of $\zeta$ w.r.t. $D$ is defined as 
    \begin{align*}
        \GC_D(\zeta):=\E_{i \ell,j\ell'\sim D} \Brac{\sum_{\ell,\ell'\in [q]}\widetilde{\Cov}(x_{i\ell}x_{j\ell'})^2}.
    \end{align*}
\end{definition}

%\noindent For a graph $G,$ the degree profile matrix $D(G)$ induces a distribution over the vertices. With a slight abuse of notation we also denote such distribution by $D(G).$

\begin{definition}[Local correlation]
    Let $\zeta$ be a degree $2$-pseudo-distribution consistent with \ref{eq:basic-sdp} and let $A\in\R^{nq\times nq}$ be symmetric. The local correlation of $\zeta$ on $A$ is defined as 
    \begin{align*}
        \LC_{A}(\zeta):=\E_{\set{i\ell,j\ell'}\sim A} \Brac{\Abs{\widetilde{\Cov}(x_{i\ell}x_{j\ell'})}}. %\cdot \Paren{d_G(i)d_G(j)}^{1/2}.
    \end{align*}
\end{definition}
\noindent Given a pseudo-distribution $\zeta$ consistent with \ref{eq:basic-sdp}, the following rounding algorithm is often called independent rounding.

\begin{algorithmbox}\label{alg:independent-rounding}
    \mbox{Independent rounding}\\
    \textbf{Input:} Pseudo-distribution $\zeta$ consistent with \ref{eq:basic-sdp}.\\
    \textbf{Output:} Integral solution $\hat{x}$ to \ref{eq:basic-sdp}.
    \begin{enumerate}
        \item For each $i\in [n]:$ 
        \begin{enumerate}
            \item Sample $\bm \ell\in [q]$ from the distribution induced by $\tilde{\E}_\zeta\Brac{x_{i1}},\ldots,\tilde{\E}_\zeta\Brac{x_{iq}}.$
            \item Set $\hat{\mathbf x}_{i\bm \ell}=1$ and $\hat{\mathbf x}_{i\ell'}=0$ for all $\ell'\neq \bm \ell.$
        \end{enumerate}
        \item Return $\hat{\mathbf{x}}.$
    \end{enumerate}
\end{algorithmbox}

The following statement shows that, when the local correlation is small, then the output of \cref{alg:independent-rounding} will be closed to the quadratic form $\tilde{\E}_\zeta\Brac{\iprod{x, Ax}}.$
%Let $\cS_{nq}$ be the set of $nq$-by-$nq$ symmetric matrices with entries in $[0,1]$
%Note that $D(A)\in\cD_{nq}$ for $A\in \cS_{nq}.$

\begin{lemma}[Rounding error]\label{lem:rounding-error}
    Let $\zeta$ be a degree $2$-pseudo-distribution consistent with \ref{eq:basic-sdp} and let $A\in\R^{nq\times nq}$ be a symmetric matrix  such that, for any $i\in [n]$ and  for all $\ell,\ell\in [q],$ it holds $A_{i\ell,i\ell'}=0.$
    Then
    \begin{align*}
        \Abs{\tilde{\E}_\zeta\Brac{\iprod{\dyad{x},A}} - \E \Brac{\iprod{\dyad{\hat{\mathbf{x}}},A}}}\leq \Normo{A}\cdot \LC_A(\zeta).
    \end{align*}
    \begin{proof}
        Let $\Sigma$ be the $nq$-by-$nq$ pseudo-covariance matrix given by $\zeta.$
        By direct computation, using independence of the rounding,
        \begin{align*}
            \E \Brac{\iprod{\dyad{\hat{\mathbf{x}}},A}} &=\sum_{\substack{i,j\in [n]\\i\neq j}}\sum_{\ell,\ell'\in [q]}\E\Brac{\hat{\mathbf{x}}_{i\ell}}\E\Brac{\hat{\mathbf{x}}_{j\ell'}}A_{i\ell,j\ell'}\\
            &=\sum_{\substack{i,j\in [n]\\i\neq j}}\sum_{\ell,\ell'\in [q]}\tilde{\E}\Brac{x_{i\ell}}\tilde{\E}\Brac{x_{j\ell'}}A_{i\ell,j\ell'}\\
            &=\sum_{\substack{i,j\in [n]}}\sum_{\ell,\ell'\in [q]}\tilde{\E}\Brac{x_{i\ell}}\tilde{\E}\Brac{x_{j\ell'}}A_{i\ell,j\ell'}\\
            &=\iprod{\dyad{\tilde{\E}\Brac{x}}, A}\\
            &=\iprod{\dyad{\tilde{\E}\Brac{x}}, A} + \tilde{\E}_\zeta\Brac{\iprod{\dyad{x},A}}- \tilde{\E}_\zeta\Brac{\iprod{\dyad{x},A}}\\
            &=\tilde{\E}_\zeta\Brac{\iprod{\dyad{x},A}}- \iprod{\Sigma,A}.
            %&=\tilde{\E}_\zeta\Brac{\iprod{\dyad{x},A}} -  \iprod{\widetilde{\Cov},A}
        \end{align*}
        Therefore
        \begin{align*}
            \Abs{\tilde{\E}_\zeta\Brac{\iprod{\dyad{x},A}} - \E \Brac{\iprod{\dyad{\hat{\mathbf{x}}},A}}}\leq\Abs{\iprod{\Sigma,A}}\leq \Normo{A}\cdot\LC_A(\zeta).
        \end{align*}
    \end{proof}
\end{lemma}

The next result states that it is always possible to find a pseudo-distribution consistent with \ref{eq:basic-sdp} with low global correlation. For a matrix $A$, let $\OPT(A)$ be the objective value of the function $\max \iprod{x, Ax}$ over integral solutions consistent with \ref{eq:basic-sdp}.

\begin{lemma}[Driving down global correlation, \cite{barak2011rounding,raghavendra2012approximating}]\label{lem:driving-down-global-correlation}
    Let $A\in \R^{nq\times nq}$ be symmetric. 
    There exists an algorithm that, given $A$, runs in randomized time $q^{O(1/\eta)}n^{O(1)}$ and returns a degree-$2$ pseudo-distribution $\zeta$, consistent with \ref{eq:basic-sdp} satisfying
    \begin{enumerate}
        \item $\tilde{\E}_\zeta\Brac{\iprod{x, Ax}}\geq \OPT,$
        \item $\GC_{D}(\zeta)\leq \eta.$
    \end{enumerate}
\end{lemma}
%\tom{It would be better to prove this lemma in the appendix but for now let's just refer to previous work.}
We are now ready to introduce our key innovation, which introduces a trade-off between local correlation, global correlation and incoherence.

\begin{lemma}[Local correlation implies global correlation via incoherence \cite{barak2011rounding, raghavendra2012approximating}]\label{lem:local-to-global-incoherence}
    Let $\tau, \rho\in [0,1].$ Let $\zeta$ be a degree $2$-pseudo-distribution consistent with \ref{eq:basic-sdp}, let $A\in \R^{nq\times nq}$  be symmetric and let $D\in \cD_{nq}.$ Suppose that $\simul_\tau(D^{-1/2}AD^{-1/2})=r>0$ and $\Normo{A}\geq\rho \Normo{D}.$ Then
    \begin{align*}
        \sqrt{\GC_{D}(\zeta)}\geq \Paren{\LC_A(\zeta)-\tau}\cdot  \frac{\rho}{\sqrt{r}}.
    \end{align*}
\end{lemma}
\begin{proof}
    Let $\Sigma$ be the pseudo-covariance matrix of $\zeta.$
    Let $\sum_i\sigma_i v_i\transpose{u_i}$ be the singular value decomposition of $\bar{A}=D^{-1/2}AD^{-1/2}.$ Then
    \begin{align*}
        \iprod{A,\Sigma} &= \iprod{\bar{A},D^{1/2}\Sigma D^{1/2}} \\
        &=\sum_{i} \iprod{\sigma_i v_i\transpose{u_i}, D^{1/2}\Sigma D^{1/2}}\\
        &=\sum_{\sigma_i\geq \tau}\iprod{\sigma_i v_i\transpose{u_i}, D^{1/2}\Sigma D^{1/2}} + \sum_{\sigma_i<\tau}\iprod{\sigma_i v_i\transpose{u_i}, D^{1/2}\Sigma D^{1/2}}\\
        &\leq \sum_{\sigma_i\geq \tau}\iprod{v_i\transpose{u_i}, D^{1/2}\Sigma D^{1/2}}+\tau\cdot \Tr D^{1/2}\Sigma D^{1/2}\\
        &\leq \Normf{\sum_{\sigma_i\geq \tau}v_i\transpose{u_i}}\Normf{D^{1/2}\Sigma D^{1/2}}+\tau\cdot \Tr{D}\\
        &\leq \sqrt{r}\Normf{D^{1/2}\Sigma D^{1/2}}+\tau\cdot \Normo{D}\\
        &\leq \sqrt{r}\Normf{D^{1/2}\Sigma D^{1/2}}+\tau\cdot \Normo{D}.
    \end{align*}
    Dividing both sides by $\frac{1}{\Normo{A}}$ we get
    \begin{align*}
        \LC_A(\zeta)-\tau&\leq r\sqrt{\frac{\Normf{D^{1/2}\Sigma D^{1/2}}^2}{\Normo{A}^2}}\\
        &\leq  \sqrt{\frac{r\Normf{D^{1/2}\Sigma D^{1/2}}^2}{\Normo{A}^2}}\\
        &\leq \sqrt{r\frac{\Normf{D^{1/2}\Sigma D^{1/2}}^2}{\rho^2\Normo{D}^2}}\\
        &=\sqrt{\tfrac{r}{\rho^2}\GC_{D}(\zeta)}.
    \end{align*}
\end{proof}

\subsubsection{Maximum cut}\label{sec:max-cut}
We apply here the technology developed above in the context of \maxcut. 
We denote by $\OPT(G)$ the $\maxcut$ value on graph $G.$ 

\begin{theorem}\label{thm:max-cut}
    Let $\gamma\in [0,1].$ There exists an algorithm that, given a $n$-vertex graph $G, \eta\in [0,1], D\in\cD_{n}$ and a symmetric matrix $\tilde{A}\in\R^{n\times n},$ satisfying
    \begin{enumerate}[(i)]
        \item  $\Normo{D}\leq O(\Normo{A(G)})$
        \item $\Norm{\tilde{A}}\leq O(1)$
        \item for any $x\in\set{0,1}^n,\, \Abs{\Iprod{x,\Paren{A(G)-D^{1/2}\tilde{A}D^{1/2}}x}}\leq \Normo{D}\cdot\gamma$
        %$\Norm{D^{-1/2}A(G)D^{-1/2}-\tilde{A}}\leq \gamma$
    \end{enumerate}
    returns a bipartition  such that, the total weight of the cut edges is at least
    \begin{align*}
        \Paren{1 - O(\eta+\gamma)}\cdot\OPT,
    \end{align*}
    whenever $\simul_{\eta}(\tilde{A})\leq r.$
    Moreover, the algorithm runs in randomized time $n^{O(1)}\cdot\exp\Set{O\Paren{\frac{r}{\eta^{2}}}}.$
    \begin{proof}
        Consider the label extended $2n$-vertex graph $G'$ with edges $\set{i\ell,j\ell'}$ if and only if $\set{\ij}\in E(G)$ and $\ell\neq\ell'$ (here we index vertices by pairs in $[n]\times [2]$).
        This operation only  exactly doubles the multiplicity of each singular value.
        We may construct similarly a $2n$-by-$2n$ matrix $\tilde{A}'$ from $\tilde{A}.$
        Note that $\coh{2r}{\tilde{A}'}\leq \coh{r}{\tilde{A}}.$
        Let $D'=D\tensor \Id_2.$
        The maximum cut corresponds to the objective value of $\max\iprod{x,A(G')x}$ where the maximum is taken over solutions of \ref{eq:basic-sdp} for $q=2$. 
        
        By \cref{lem:driving-down-global-correlation}, for any $\bar\eta>0,$ we can compute in time $2^{O(1/\bar{\eta})}n^{O(1)}$ a degree-$2$ pseudo-distribution $\zeta$ consistent with \ref{eq:basic-sdp}, with objective value $\OPT$ and satisfying $\GC_{D'}\leq \bar{\eta}.$
        By assumption on the spectral norm of $\tilde{A}'$
        \begin{align*}
            \Normo{D'^{1/2}\tilde{A}'D'^{1/2}}\leq \Norm{\tilde{A}'}\Normo{D'}\leq O(\Normo{D'}).
        \end{align*}
        Hence picking
        \begin{align*}
            \bar{\eta}= C(\eta^2\cdot r)%\cdot\frac{\coh{2r}{\tilde{A}'}\cdot 2r}{\sqrt{n}}=O(\eta)\cdot\frac{\coh{r}{\tilde{A}}\cdot r}{\sqrt{n}}
        \end{align*}
        for some large enough constant $C>0,$
        we get $\LC_{\tilde{A}'}(\zeta)\leq \eta$ by \cref{lem:local-to-global-incoherence}.
        Notice now that for any integral solution  $x\in\set{0,1}^{2n}$ to \ref{eq:basic-sdp}
        \begin{align*}
            \iprod{x,A(G')x}&=\Iprod{x,\Paren{A(G')-D'^{1/2}\tilde{A}'D'^{1/2}+D'^{1/2}\tilde{A}'D'^{1/2}}x}\\
            &=\iprod{x,D'^{1/2}\tilde{A}'D'^{1/2}x} + \Iprod{x,\Paren{A(G')-D'^{1/2}\tilde{A}'D'^{1/2}}x}.
        \end{align*}
        And so
        \begin{align*}
            \Abs{\Iprod{x,\Paren{A(G')-D'^{1/2}\tilde{A}'D'^{1/2}}x}}%&\leq \Snormt{D'^{1/2}x}\cdot\Norm{D'^{-1/2}A(G')D'^{-1/2}-\tilde{A}'}\\
            %&\leq 2 \Normo{D}\cdot\Norm{D'^{-1/2}A(G')D'^{-1/2}-\tilde{A}'}\\
            &\leq 2\Normo{D}\cdot\gamma\\
            &\leq O(\gamma)\cdot\OPT
        \end{align*}
        where the first inequality follows by assumption on $\tilde{A}, D$ and the last inequality follows as $\OPT\geq 4\Normo{A(G)}\geq\Omega(\Normo{D}).$
        By \cref{lem:local-to-global-incoherence} this implies $\LC_{A(G')}(\zeta)\leq O(\eta + \gamma).$ By \cref{lem:rounding-error}, the result follows.
    \end{proof}
\end{theorem}

\subsubsection{Maximum 2-CSP}\label{sec:max-csp}
We extend here the result for \maxcut to \twocsp.
We start establishing some notation. 
A \twocsp instance $\cI$ consists of a graph $G(\cI)=([n], E, w),$ known as the constraint graph , where every edge $\set{i,j}$ is labeled with a  binary relations $R_{\cI}\set{i,j}\subseteq [q]^2.$ Here, $q$ is known as the alphabet size. 
The instance $\cI$ can also be represented through its labeled extended graph.
\begin{definition}[Label extended graph]\label{def:label-extended-graph}
    For a \twocsp instance $\cI$, the label extend graph, denoted by $\Gamma(\cI),$ is the graph with:
    \begin{enumerate}
        \item vertex set $[n]\times [q],$ (we index vertices by pairs)
        \item an edge $\set{i\ell,j\ell'}$ with weight $w\set{\ij}$ if $\set{i,j}\in E$ and $\set{\ell,\ell'}\in R_\cI\set{i,j}.$
    \end{enumerate}
\end{definition}
\noindent With a slight abuse of notation, we often equate $\cI$ with its label extended graph $\Gamma(\cI)$, writing for example $A(\cI)$ in place of $A(\Gamma(\cI)).$ %We say that a \twocsp instance $\cI$ is undirected if $\Gamma(\cI)$ is undirected. 
And $\sigma_1,\ldots,\sigma_{nq}$ for the singular values of $\Gamma(\cI).$
The value of an assignment $x\in [q]^n$ is given by
\begin{align*}
    \valI{\cI}{x}=\sum_{\set{i,j}\in E(\cI)} w\set{i,j}\cdot \iverson{\set{x_i,x_j}\in R_{\cI}\set{i,j}},
\end{align*}
where $\iverson{\cdot}$ denotes the Iverson brackets.
The optimal value of $\cI$, also called the objective value, is then
\begin{align*}
    \OPT(\cI) =\max_{x\in [q]^n}\valI{\cI}{x}.
\end{align*}
\noindent We prove the following statement.
\begin{theorem}\label{thm:max-2-csp}
    There exists an algorithm that, given a \twocsp instance $\cI$ over $n$ variables and alphabet $[q]$, an integer $r>0,$ $\eta\in[0,1],$ $D\in \cD_{nq}$ and a symmetric matrix $\tilde{A}\in\R^{nq\times nq},$ satisfying
    \begin{enumerate}[(i)]
        \item $\Normo{D}\leq O(\Normo{A(\cI)})$
        \item $\Norm{\tilde{A}}\leq O(1)$
        \item for any integral solution $x\in\set{0,1}^{nq}$ to \ref{eq:basic-sdp}, $\Abs{\Iprod{x,\Paren{A(\cI)-D^{1/2}\tilde{A}D^{1/2}}x}}\leq \Normo{D}\cdot\gamma$
    \end{enumerate}
    returns an assignment with value at least
   \begin{align*}
        %\OPT - O\Paren{\frac{\eta}{\lambda_r}}\Normo{A(G)}\,.
        \Paren{1- O(\eta+\gamma)}\cdot\OPT,
    \end{align*}
    whenever $\simul_{\eta}(\tilde{A})\leq r$.
    Moreover, the algorithm runs in randomized time $(nq)^{O(1)} \cdot\exp\Set{O\Paren{\frac{r\log q}{\eta^2}}}.$
    \begin{proof}
        For an assignment $x\in [q]^n,$ let $\chi\in \set{0,1}^{nq}$ be the vector with entries (indexed by pairs $i\in[n],\ell\in[q]$)
        \begin{align*}
            \chi_{(i, \ell)} =
            \begin{cases}
                1&\text{ if }x_i=\ell\\
                0&\text{ otherwise.}
            \end{cases}
        \end{align*}
        Note that $\valI{\cI}{x}=\iprod{\chi, \Gamma(\cI)\chi}.$
        By \cref{lem:driving-down-global-correlation} we can compute in time $q^{O\Paren{1/\bar{\eta}}}n^{O(1)}$ a degree-$2$ pseudo-distribution consistent with \ref{eq:basic-sdp} with objective value $\OPT$ and satisfying $\GC_{D}\leq \bar{\eta}$ for any $\bar{\eta}>0.$
        By assumption on the spectral norm of $\tilde{A}$
        \begin{align*}
            \Normo{D^{1/2}\tilde{A}D^{1/2}}\leq \Norm{\tilde{A}D}\Normo{DD}\leq O(\Normo{D}).
        \end{align*}
        Hence picking
        \begin{align*}
            \bar{\eta}= C(\eta^2\cdot r)
        \end{align*}
        for a large enough constant $C>0,$
        we get $\LC_{\tilde{A}}(\zeta)\leq \eta$ by \cref{lem:local-to-global-incoherence}.
        Notice now that for any integral solution  $x\in\set{0,1}^{qn}$ to \ref{eq:basic-sdp}
        \begin{align*}
            \iprod{x,A(\cI)x}&=\Iprod{x,\Paren{A(\cI)-D^{1/2}\tilde{A}D^{1/2}+D^{1/2}\tilde{A}D^{1/2}}x}\\
            &=\iprod{x,D^{1/2}\tilde{A}D^{1/2}x} + \Iprod{x,\Paren{A(\cI)-D^{1/2}\tilde{A}D^{1/2}}x}.
        \end{align*}
        And so
        \begin{align*}
            \Abs{\Iprod{x,\Paren{A(\cI)-D^{1/2}\tilde{A}D^{1/2}}x}}%&\leq \Snormt{D'^{1/2}x}\cdot\Norm{D'^{-1/2}A(G')D'^{-1/2}-\tilde{A}'}\\
            %&\leq 2 \Normo{D}\cdot\Norm{D'^{-1/2}A(G')D'^{-1/2}-\tilde{A}'}\\
            &\leq 2\Normo{D}\cdot\gamma\\
            &\leq O(\gamma)\cdot\OPT
        \end{align*}
        where the first inequality follows by assumption on $\tilde{A}, D$ and the last inequality follows as $\OPT\geq 4\Normo{A(G)}\geq\Omega(\Normo{D}).$
        By \cref{lem:local-to-global-incoherence} this implies $\LC_{A(\cI)}(\zeta)\leq O(\eta + \gamma).$ By \cref{lem:rounding-error}, the result follows.
    \end{proof}
\end{theorem}

%As for \maxcut we can immediately obtain a cleaner, less general statement.

%\begin{corollary}\label{cor:max-csp-easy}
%    There exists an algorithm that, given a \twocsp instance $\cI$ over $n$ variables and alphabet $[q],$ an integer $r>0,\eta\in\Brac{0,1},$ returns an assignemnt with value at least
%    \begin{align*}
%        \Paren{1 - O(\eta)}\cdot\OPT,
%    \end{align*}
%    whenever $\simul_{\eta}(\bar{A}(\cI))\leq r.$
%    Moreover, the algorithm runs in randomized time $n^{O(1)}\cdot\exp\Set{O\Paren{\frac{\coh{r}{\cI}^2\cdot r^2\cdot\log q}{nq\cdot\eta^{2}}}}.$
 %   \begin{proof}
%        The result follows from \cref{thm:max-2-csp}  setting % $D=D(\cI),$ $\tilde{A}=\bar{A}(\cI).$
%    \end{proof}
%\end{corollary}

\subsubsection{Maximum bisection}\label{sec:max-bisection}
The technology developed for \cref{thm:max-cut} and \cref{thm:max-2-csp} immediately extend to settings in which additional global constraints are enforced on feasible solutions. As a proof of concept we extend them to \maxbisection, which is the problem of finding the maximum balanced cut.

\begin{theorem}\label{thm:max-bisection}
    Let $\gamma, p\in [0,1].$ There exists an algorithm that, given a $n$-vertex graph $G, \eta\in [0,1], D\in\cD_{n}$ and a symmetric matrix $\tilde{A}\in\R^{n\times n},$ satisfying
    \begin{enumerate}[(i)]
        \item  $\Normo{D}\leq O(\Normo{A(G)})$
        \item $\Norm{\tilde{A}}\leq O(1)$
        \item for any $x\in\set{0,1}^n,\, \Abs{\Iprod{x,\Paren{A(G)-D^{1/2}\tilde{A}D^{1/2}}x}}\leq \Normo{D}\cdot\gamma$
        %$\Norm{D^{-1/2}A(G)D^{-1/2}-\tilde{A}}\leq \gamma$
    \end{enumerate}
    with probability $1-n^{-O(1)},$ returns a bipartition  $(L,R)$ such that, the total weight of the cut edges is at least
    \begin{align*}
        \Paren{1 - O(\eta+\gamma)}\cdot\OPT,
    \end{align*}
    and $\min\set{\card{L},\card{R}}\geq \tfrac{n}{2}-\tilde{O}(\sqrt{n}),$
    whenever $\simul_{\eta}(\tilde{A})\leq r.$
    Moreover, the algorithm runs in time $n^{O(1)}\cdot\exp\Set{O\Paren{\frac{r}{\eta^{2}}}}.$
\end{theorem}

Because the proof of \cref{thm:max-bisection} is  similar to that of \cref{thm:max-cut}, we defer it to \cref{sec:deferred-proofs}.

%We then show that these ideas can be extended to the constrained setting by obtain
%in \cref{sec:correlation-rounding} we prove a general statement about the global correlation rounding framework \cite{barak2011rounding, guruswami2011lasserre, raghavendra2012approximating} for matrices with low coherence. Necessary background on the sum-of-squares framework can be found in \cref{sec:sos-background}. We then use this result to obtain differentially private algorithms for \maxcut and \twocsp respectively in \cref{sec:dp-max-cut}, and \cref{sec:dp-max-two-csp}.
%Finally, in \cref{sec:conductance}, \cref{sec:small-set-expansion}, \cref{sec:k-way-partitioning} we use \cref{thm:threshold-rank} to privatize existing algorithms for \conductance, \smallsetexpansion and general $k$-way graph partitioning tasks.

%Throughout the section we assume $n$ to be the number of vertices in the input graph at hand and  $0<\eps,\delta\leq 1.$
%For simplicity of the exposition we do not repeat these assumptions in the theorem statements.

\subsection{Maximum cut under differential privacy}\label{sec:dp-max-cut}
%For an $n$-vertex  graph $G=(V,E)$, \maxcut is the task of finding a bipartition  that maximizes the weight of edges crossing the partition. We denote by $\OPT(G)$ the $\maxcut$ value on graph $G.$ We do not specify the graph when the context is clear.

%\begin{theorem}[\cite{barak2011rounding}]%\label{thm:max-cut}
%    There exists an algorithm that, given a graph $G$, an integer $r>0$ and $0<\eta\leq 1,$  returns a bipartition  such that, the total weight of the cut edges is at least
%    \begin{align*}
        %\OPT - O\Paren{\frac{\eta}{\lambda_r}}\Normo{A(G)}\,.
 %       \Paren{1- \eta}\cdot\OPT,
 %   \end{align*}
%    whenever $\mul_{\eta^{O(1)}}(G)\leq r$.
%    Moreover, the algorithm runs in randomized time $2^{O(r/\eta^{O(1)})}n^{O(1)}.$
%\end{theorem}
We combine here \cref{thm:max-cut} with \cref{thm:threshold-rank}.
We reuse the notation introduce in \cref{sec:correlation-rounding}.

\begin{theorem}[Edge-DP \maxcut]\label{thm:dp-max-cut}
    Let $\eps,\delta,\kappa\in [0,1]$ with $\delta\geq 10n^{-100}.$ 
    Let $C>0$ be a large enough universal constant.
    There exists an $(\eps,\delta)$-DP algorithm that, given a graph $G$, $\eps,\delta,\kappa$ an integer $r>0,$ with probability at least $1-n^{-O(1)}$ returns a bipartition such that the number of cut edges is at least
    \begin{align*}
        %\OPT -O\Paren{\frac{n\log n}{\eps}} - 
        \Paren{1 -O(\sigma_{r+1}+\kappa +\gamma)}\cdot\OPT
    \end{align*}
    whenever $G$ has 
    \begin{align*}
        d_{\min}\geq C\Paren{\frac{\sqrt{\log(1/\delta)}}{\eps}\cdot \frac{\sqrt{r\cdot\mu_r+\log n}}{\gamma^2\cdot(\sigma_r -\sigma_{r+1})}},\qquad \sigma_r\geq 2\gamma + 3\sigma_{r+1}.
    \end{align*}
    Moreover, the algorithm runs in randomized time
    \begin{align*}
        n^{ O\Paren{1}}\cdot \exp\Set{O\Paren{\frac{r}{(\sigma_{r+1}^2+\kappa^2)\cdot(\sigma_{r+1}+\gamma)}}}.
    \end{align*}
\end{theorem}

Note that by setting $\kappa=\gamma=0.001$ we immediately get \cref{thm:dp-max-cut-easy}.
We next present the algorithm behind \cref{thm:dp-max-cut}. To do that we define the following convex set $\cS_{n}(A)$.
\begin{align}\label{eq:projected set}%\tag{$\cS_{n}(A)$}
    \Set{
    \begin{aligned}
        &M_{ij}\geq 0&\forall i,j\in [n]&\\
        &\Norm{M-A}\leq  \Abs{1-\tfrac{1}{\sqrt{n}}\Norm{A\mathbb{1}}}.
    \end{aligned}
    }
\end{align}
The algorithm is the following.
\begin{algorithmbox}\label{alg:dp-max-cut}
    \mbox{}\\
    \textbf{Input:} Graph $G\,, 0<\eps,\delta,\kappa\leq 1,$ integer $r>0$\\
    \textbf{Output:} $\hat{x}\in\set{\pm 1}^n$
    \begin{enumerate}[(1)]
        \item Let $\hat{\mathbf{D}}$ be the diagonal matrix with entries $D(G)_{ii}+\mathbf{w}_i$, where each $\mathbf{w}_i$ is sampled independently from $ N\Paren{0, \frac{4\log(4/\delta)}{\eps^2}}$.  Output $\bot$ if any entry of $\hat{\mathbf{D}}$ is negative.
        \item Run the algorithm of \cref{thm:threshold-rank} on input $\bar{A}(G)$ with parameters $\eps/2,\delta/2,r$.  Let $\hat{\mathbf{A}}'$ be its output.
        \item Project $\hat{\mathbf{A}}'$ onto $\cS_{n}(\hat{\mathbf{A}}').$  Let $\hat{\mathbf{A}}$ be the output.
        \item Run the algorithm of \cref{thm:max-cut} on $\hat{\mathbf D}^{1/2}\hat{\mathbf{A}}\hat{\mathbf D}^{1/2}$ with parameters  $\hat{\mathbf{D}}, \hat{\mathbf{A}}',$ $\bm\eta=\max\set{\sigma_{r+1}(\hat{\mathbf{A}}),\kappa}.$
    \end{enumerate}
\end{algorithmbox}

It is easy to see that \cref{alg:dp-max-cut} is indeed differentially private.

\begin{fact}\label{fact:max-cut-alg-dp}
    \cref{alg:dp-max-cut} is $(\eps,\delta)$-DP.
    \begin{proof}
        By \cref{thm:gaussian-mechanism} step (1) is $(\eps/2,\delta/2)$-DP. By \cref{thm:threshold-rank} step (2) is also $(\eps/2,\delta/2)$-DP. Step (3) and (4) only uses the output of the previous two steps and the public parameter in input $r$. Hence by \cref{lem:dp-post-processing} and \cref{lem:dp-composition} the whole algorithm is $(\eps,\delta)$-DP.
    \end{proof}
\end{fact}

To prove \cref{thm:dp-max-cut} we will use the following fact, which bound the  quadratic form of any vector over the difference between the input graph and the privatize graph obtained in step (3). 

\begin{fact}\label{fact:binary-quadratic-form-perturbation}
    Let $G$ be an $n$-vertex graph. Let $\hat{A}$ be a symmetric matrix  satisfying $\Norm{\Bar{A}(G)-\hat{A}}\leq \gamma$ and $\Norm{\hat{A}}\leq \rho.$ Let $0\sle\hat{D}\in\R^{n\times n}$ be a diagonal matrix such that $\Normi{D(G)-\hat{D}}\leq \beta.$
    Then for any $x\in \R^n$
    \begin{align*}
        \Abs{\iprod{x, A(G)x}- \iprod{x, \hat{D}^{1/2}\hat{A}\hat{D}^{1/2}x}}\leq &2(\Normi{x}^2+\Normi{x})\cdot\rho\cdot\beta\cdot n \\
        &+ 2\rho\cdot\Normi{x}\cdot\sqrt{\beta \cdot n\cdot \Normo{D}} \\
        &+\Normi{x}\cdot\gamma\cdot \Normo{D}.
    \end{align*}
\end{fact}
We defer the proof of \cref{fact:binary-quadratic-form-perturbation} to \cref{sec:deferred-proofs}.
The next statement shows that if $\hat{\mathbf{A}}'$ is close to $\bar{G}$ in spectral norm, then so is $\hat{\mathbf{A}}.$

\begin{fact}\label{fact:spectral-norm-after-projection}
    Let $G$ be an $n$-vertex graph and $r>0$ an integer. Let $\hat{A}'\in\R^{n\times n}$ be a symmetric matrix satisfying $\Norm{\bar{A}_{(r)}(G)-\hat{A}'}\leq \gamma$ and let $\hat{A}$ be the projection of $\hat{A}'$ onto $\cS_{n}(\hat{A}').$
    Then
    \begin{align*}
        \Norm{\bar{A}(G)-\hat{A}}\leq 2\sigma_{r+1}+2\gamma.
    \end{align*}
    \begin{proof}
        Because $\bar{A}\in \cS_n(\hat{A}'),$ by triangle inequality
        \begin{align*}
             \Norm{\bar{A}-\hat{A}}&=\Norm{\bar{A}-\hat{A}+\hat{A}'-\hat{A}'}\\
             &\leq \Norm{\bar{A}-\hat{A}'}+\Norm{\hat{A}-\hat{A}'}\\
             &\leq 2\Norm{\bar{A}-\hat{A}'}.
        \end{align*}
    \end{proof}
\end{fact}

We are finally ready to study the guarantees of \cref{alg:dp-max-cut} and prove the Theorem.
\begin{proof}[Proof of \cref{thm:dp-max-cut}]
    By concentration of the Gaussian distribution, with probability at least $1-n^{-200}$ we have $\max_{i\in [n]}\Abs{\hat{\mathbf{D}}_{ii}-D_{ii}}\leq O\Paren{\frac{\sqrt{\log n}}{\eps}}$ for a large enough hidden constant. We condition the rest of the analysis on this event. We also condition the analysis on the event that the conclusion of \cref{thm:threshold-rank} is verified. All these events happen simultaneously with probability at least $1-n^{-O(1)}.$
    Then by \cref{thm:threshold-rank} $\Norm{\hat{\mathbf{A}}'-\bar{A}_{(r)}}\leq \gamma$ and thus by \cref{fact:spectral-norm-after-projection}
    \begin{align*}
        \Norm{\hat{\mathbf{A}}-\bar{A}}\leq2\sigma_{r+1}+2\gamma.
    \end{align*}
     We also have by assumption on $\gamma,$ $\sigma_{r+1}(\hat{\mathbf{A}})\leq \sigma_{r+1}(\bar{A}).$
    %\begin{align*}
    %    \sigma_{r+1}(\hat{\mathbf{A}})&\leq \sigma_{r+1}(\bar{A}) %+ \sigma_{1}(\hat{\mathbf A}-\bar{A})\\
    %    &\leq 3\sigma_{r+1}(\bar{A}) +2\gamma\\
    %    &<\sigma_r(\bar A).
    %\end{align*} 
    Next observe that for any $x\in\set{0,1}^n$
     \begin{align*}
         \Abs{\iprod{x,\hat{\mathbf D}^{1/2}(\hat{\mathbf{A}}-\hat{\mathbf{A}}')\hat{\mathbf D}^{1/2}x}}\leq \Normo{\hat{\mathbf{D}}}\Norm{\hat{\mathbf{A}}-\hat{\mathbf{A}}'}\leq \Paren{4\sigma_{r+1}+4\gamma}\Normo{A}.
     \end{align*}
     As  $\Norm{\hat{\mathbf{A}}'}\leq 1 + \gamma,$ to apply \cref{thm:max-cut} we now  argue that $\Normo{\hat{\mathbf{D}}}$ is close to $\Normo{\hat{\mathbf D}^{1/2}\hat{\mathbf{A}}\hat{\mathbf D}^{1/2}}.$
     Indeed, we have
     \begin{align*}
            \Abs{\Normo{\hat{\mathbf D}^{1/2}\hat{\mathbf{A}}\hat{\mathbf D}^{1/2}} - \Normo{\hat{\mathbf{D}}}}&\leq \Abs{\Normo{\hat{\mathbf D}^{1/2}\hat{\mathbf{A}}\hat{\mathbf D}^{1/2}}-\Normo{A}} + \Abs{\Normo{A}- \Normo{\hat{\mathbf{D}}}}\\
            &\leq \Paren{4\sigma_{r+1}+\gamma}\Normo{\hat{\mathbf D}^{1/2}\hat{\mathbf{A}}\hat{\mathbf D}^{1/2}}
     \end{align*}
     Applying now \cref{thm:max-cut} we obtain, in expectation, an integral solution of value 
     \begin{align*}
         (1-O(\sigma_{r+1}+\kappa+\gamma))\cdot \OPT(\hat{\mathbf D}^{1/2}\hat{\mathbf{A}}\hat{\mathbf D}^{1/2}).
     \end{align*}
     The required time is 
     \begin{align*}
         n^{O(1)}&\cdot \exp\Set{O\Paren{\frac{r}{(\sigma_{r+1}^2+\kappa^2)\cdot(\sigma_{r+1}+\gamma)}}}.
         %&\leq n^{ O\Paren{1}}\cdot \exp\Set{O\Paren{\frac{r}{(\sigma_{r+1}^2+\kappa^2)\cdot(\sigma_{r+1}+\gamma)}}}.
     \end{align*}
     %\begin{align*}
     %    n^{O(1)}&\cdot \exp\Set{O\Paren{\frac{r^2}{n\cdot(\sigma_{r+1}^2+\kappa^2)\cdot(\sigma_{r+1}+\gamma)}}\Paren{\coh{r}{\bar{A}}^2 +\frac{\log^2 n}{r^2}}}\\
     %    &\leq n^{ O\Paren{1+\frac{\log n}{n\cdot (\sigma_{r+1}^2+\kappa^2)\cdot(\sigma_{r+1}+\gamma)}}}\cdot \exp\Set{O\Paren{\frac{r^2\cdot\coh{r}{\bar{A}}^2}{n\cdot(\sigma_{r+1}^2+\kappa^2)\cdot(\sigma_{r+1}+\gamma)}}}.
     %\end{align*}
     as the first three steps of the algorithm require polynomial time.
     It remains to argue that any  solution for $\hat{\mathbf D}^{1/2}\hat{\mathbf{A}}\hat{\mathbf D}^{1/2}$ of high objective value is also a good solution for the original graph $G$.
     To this end, Observe now that for every partition $x\in\set{\pm 1}^n$, the weight of cut edges in $G$ is exactly $\frac{1}{2}\Paren{\Normo{A}-\iprod{x,Ax}}.$ 
    By \cref{fact:binary-quadratic-form-perturbation}
    \begin{align*}
        \Abs{\iprod{x,Ax}-\iprod{x,\hat{\mathbf D}^{1/2}\hat{\mathbf{A}}\hat{\mathbf D}^{1/2}x}}&\leq O\Paren{\frac{n\sqrt{\log n}}{\eps} + \sqrt{\frac{\Normo{A}n\sqrt{\log n}}{\eps}}} + \Paren{2\sigma_{r+1}+\frac{\gamma}{16}}\Normo{A}\\
        &\leq  \Paren{2\sigma_{r+1}+\frac{\gamma}{8}}\Normo{A},
    \end{align*}
    where we used the bound on  $d_{\min}(G).$ 
    Similarly, because both matrices have non-negative entries, we also have
    \begin{align*}
        \Abs{\Normo{A}-\Normo{\hat{\mathbf D}^{1/2}\hat{\mathbf{A}}\hat{\mathbf D}^{1/2}}}=\Abs{\iprod{\mathbb{1}, \Paren{A-\hat{\mathbf D}^{1/2}\hat{\mathbf{A}}\hat{\mathbf D}^{1/2}}\mathbb{1}}}\leq \Paren{2\sigma_{r+1}+\frac{\gamma}{8}}\Normo{A}.
    \end{align*}
    Combining the two inequalities we have for any $x\in\set{\pm 1}^n$
    \begin{align}
        \frac{1}{2}\Abs{\Normo{A}-\iprod{x,Ax}- \Normo{\hat{\mathbf D}^{1/2}\hat{\mathbf{A}}\hat{\mathbf D}^{1/2}} + \iprod{x,\hat{\mathbf D}^{1/2}\hat{\mathbf{A}}\hat{\mathbf D}^{1/2}x}}&\leq \Paren{2\sigma_{r+1}+\frac{\gamma}{4}}\Normo{A}\label{eq:generic-bound}\\
        %&\leq \Paren{2\eta+\frac{\gamma}{4}}\Normo{A}\\
        &\leq \Paren{8\sigma_{r+1}+\gamma} \OPT.\notag
    \end{align}
    using the fact that $\OPT\geq \Normo{A}/4.$
    This concludes the proof.
    %
    %Now note that $\hat{\mathbf A}$ is by construction the normalized adjacency matrix of the graph $\hat{\mathbf D}^{1/2}\hat{\mathbf{A}}\hat{\mathbf D}^{1/2}.$
    %
    %Similarly, because both matrices have non-negative entries
    %\begin{align*}
    %    \Abs{\Normo{A}-\Normo{\hat{\mathbf D}^{1/2}\hat{\mathbf{A}}\hat{\mathbf D}^{1/2}}}=\Abs{\iprod{\mathbb{1}, \Paren{A-\hat{\mathbf D}^{1/2}\hat{\mathbf{A}}\hat{\mathbf D}^{1/2}}\mathbb{1}}}\leq \frac{\gamma}{4}\Normo{A}.
    %\end{align*}
    %Combining the two inequalities we have for any $x\in\set{\pm 1}^n$
    %\begin{align}
    %    \frac{1}{2}\Abs{\Normo{A}-\iprod{x,Ax}- \Normo{\hat{\mathbf D}^{1/2}\hat{\mathbf{A}}\hat{\mathbf D}^{1/2}} + \iprod{x,\hat{\mathbf D}^{1/2}\hat{\mathbf{A}}\hat{\mathbf D}^{1/2}x}}&\leq \frac{\gamma}{4}\Normo{A}\label{eq:generic-bound}\\
    %    &\leq \gamma \cdot \OPT\notag
    %\end{align}
    %where we used the fact that $d_{\min}(G)\geq C\Paren{\frac{\sqrt{\log n}}{\eps}}$ and that $\OPT\geq \Normo{A}/4.$
    %By \cref{fact:low-threshold-rank-spectral-gap}, for some $r'\leq r,$ $\bar{A}$ has $(r,(1-\eta^{O(1)})/r)$-spectral gap. Since  $\gamma\leq (1-\eta^{O(1)})/r$ then $\hat{\mathbf A}$ satisfies $\mul_{\lambda_r(\bar A)}(\hat{\mathbf A})\leq \mul_{\lambda_r(\bar A)}( A) + 2$, applying  \cref{thm:max-cut} the result follows.
\end{proof}

\subsection{Maximum 2-CSP under differential privacy}\label{sec:dp-max-two-csp}
In this section we extend our DP result for \maxcut to \twocsp.
We reuse the notation introduce in \cref{sec:correlation-rounding}.
We are interested in differential privacy with respect to edge-adjacency over the label extended graph.
That is, we say two \twocsp instances $\cI,\cI'$ are adjacent if the respective label extended graphs are adjacent according to \cref{def:graph-adjacency}.
That is, we assume the alphabet and the set of variables to be \textit{public} knowledge, but \textit{not} the constraint graph and the collection of relations.
Combining \cref{thm:max-2-csp}  with \cref{thm:threshold-rank} we obtain the following theorem.

\begin{theorem}[Edge-DP \twocsp]\label{thm:dp-2-csp}
    Let $\eps,\delta, \kappa\in [0,1]$ with $\delta\geq 10n^{-100}.$ 
    Let $C>0$ be a large enough universal constant.
    There exists an $(\eps,\delta)$-DP algorithm that, given an undirected, \twocsp instance  $\cI$ over $n$ variables and alphabet $[q]$, $\eps,\delta,$ an integer $r>0,$ with probability at least $1-n^{-O(1)}$ returns an assignment of value
    \begin{align*}
        %\OPT -O\Paren{\frac{n\log n}{\eps}} - 
        \Paren{1 -O(\sigma_{r+1}+\kappa+\gamma)}\cdot\OPT
    \end{align*}
    whenever $\Gamma(\cI)$ has 
    \begin{align*}
        d_{\min}\geq C\Paren{\frac{\sqrt{\log(1/\delta)}}{\eps}\cdot \frac{ \sqrt{r\cdot\mu_r+\log n}}{\gamma^2\cdot(\sigma_r -\sigma_{r+1})}},\qquad \sigma_r\geq 2\gamma + 3\sigma_{r+1}.
    \end{align*}
    Moreover, the algorithm runs in randomized time 
    \begin{align*}
        n^{O\Paren{1}}\cdot\exp\Set{O\Paren{\frac{r\cdot \log q}{(\sigma_{r+1}^2+\kappa^2)\cdot(\sigma_{r+1}+\gamma)}}}.
    \end{align*}
\end{theorem}

The algorithm for \twocsp is similar to  \cref{alg:dp-max-cut}. We state it here for completeness.

\begin{algorithmbox}\label{alg:dp-max-csp}
    \mbox{}\\
    \textbf{Input:} Undirected instance $\cI\,, 0<\eps,\delta,\leq 1,$ integer $>0$\\
    \textbf{Output:} $\hat{x}\in [q]^n$
    \begin{enumerate}[(1)]
        %\item Construct the label extend graph $\Gamma.$
        \item Let $\hat{\mathbf{D}}$ be the $nq$-by-$nq$ diagonal matrix with entries $D(\cI)_{i\ell,i\ell}+\mathbf{w}_{i\ell}$, where each $\mathbf{w}_{i\ell}$ is sampled independently from $ N\Paren{0, \frac{4\log(4/\delta)}{\eps^2}}$.  Output $\bot$ if any entry of $\hat{\mathbf{D}}$ is negative.
        \item Run the algorithm of \cref{thm:threshold-rank} on input $\bar{A}(\cI)$ with parameters $\eps/2,\delta/2,r$. Let $\hat{\mathbf{A}}'$ be its output.
        \item Project $\hat{\mathbf{A}}'$ onto $\cS_{nq}(\hat{\mathbf{A}}')$ (as defined in \ref{eq:projected set}).  Let $\hat{\mathbf{A}}$ be the output.
        \item Run the algorithm of \cref{thm:max-2-csp}  on input $\hat{\mathbf D}^{1/2}\hat{\mathbf{A}}\hat{\mathbf D}^{1/2}$ with parameters  $\bm\eta=\max\set{\sigma_{r+1}(\hat{\mathbf{A}}),\kappa},\hat{\mathbf{D}}, \hat{\mathbf{A}}'.$
    \end{enumerate}
\end{algorithmbox}

Next we prove the Theorem.
\begin{proof}[Proof of \cref{thm:dp-2-csp}]
    We use $A$ to denote the adjacency matrix of $\Gamma(\cI)$ and $A(G)$ to denote the adjacency matrix of $G(\cI).$ We have the following relation $\Normo{A(G)}\leq\Normo{A}\leq q^2\Normo{A(G)}.$
    By definition of adjacency of \twocsp instances, $(\eps,\delta)$-differential privacy follows by \cref{fact:max-cut-alg-dp}.
    By concentration of the Gaussian distribution, with probability at least $1-n^{-200}$ we have 
    \begin{align*}
        \max_{i\in [n],\ell\in [q]}\Abs{\hat{\mathbf{D}}_{i\ell,i\ell}-D_{i\ell,i\ell}}\leq O\Paren{\frac{\sqrt{\log (nq)}}{\eps}}\leq O\Paren{\frac{\sqrt{\log (n)}}{\eps}}
    \end{align*}
    for a large enough hidden constant. We condition the rest of the analysis on this event. We also condition the analysis on the event that the conclusion of \cref{thm:threshold-rank} is verified. All these events happen simultaneously with probability at least $1-n^{-O(1)}.$
    Then by \cref{thm:threshold-rank} $\Norm{\hat{\mathbf{A}}'-\bar{A}_{(r)}}\leq \gamma$ and thus by \cref{fact:spectral-norm-after-projection}
    \begin{align*}
        \Norm{\hat{\mathbf{A}}-\bar{A}}\leq 2\sigma_{r+1}+2\gamma.
    \end{align*}
    We also have $\sigma_{r+1}(\hat{\mathbf{A}})<\sigma_r(\bar A).$
    For any assignment $x\in [q]^n,$ let $\chi\in \set{0,1}^{nq}$ be the vector with entries (indexed by pairs $i\in[n],\ell\in[q]$)
    \begin{align*}
        \chi_{(i, \ell)} =
        \begin{cases}
            1&\text{ if }x_i=\ell\\
            0&\text{ otherwise.}
        \end{cases}
    \end{align*}
    Note that $\valI{\cI}{x}=\iprod{\chi, \Gamma(\cI)\chi}.$
    Next observe that for any $x\in\set{0,1}^n$
    \begin{align*}
         \Abs{\iprod{\chi,\hat{\mathbf D}^{1/2}(\hat{\mathbf{A}}-\hat{\mathbf{A}}')\hat{\mathbf D}^{1/2}\chi}}\leq \Normo{\hat{\mathbf{D}}\chi}\Norm{\hat{\mathbf{A}}-\hat{\mathbf{A}}'}\leq \Paren{4\sigma_{r+1}+4\gamma}\Normo{A(G)}
     \end{align*}
     where $A(G)$ is the adjacency matrix of the constrained graph of $\cI.$
    As  $\Norm{\hat{\mathbf{A}}'}\leq 1 + \gamma,$ to apply \cref{thm:max-2-csp} we now  argue that $\Normo{\hat{\mathbf{D}}}$ is close to $\Normo{\hat{\mathbf D}^{1/2}\hat{\mathbf{A}}\hat{\mathbf D}^{1/2}}.$
    Indeed, we have
    \begin{align*}
            \Abs{\Normo{\hat{\mathbf D}^{1/2}\hat{\mathbf{A}}\hat{\mathbf D}^{1/2}} - \Normo{\hat{\mathbf{D}}}}&\leq \Abs{\Normo{\hat{\mathbf D}^{1/2}\hat{\mathbf{A}}\hat{\mathbf D}^{1/2}}-\Normo{A}} + \Abs{\Normo{A}- \Normo{\hat{\mathbf{D}}}}\\
            &\leq \Paren{4\sigma_{r+1}+\gamma}\Normo{\hat{\mathbf D}^{1/2}\hat{\mathbf{A}}\hat{\mathbf D}^{1/2}}
    \end{align*}    
    Applying now \cref{thm:max-2-csp} we obtain, in expectation, an integral solution of value 
     \begin{align*}
         (1-O(\sigma_{r+1}+\gamma))\OPT(\hat{\mathbf D}^{1/2}\hat{\mathbf{A}}\hat{\mathbf D}^{1/2}).
     \end{align*}
     The required time is 
     \begin{align*}
         (nq)^{O(1)}&\cdot \exp\Set{O\Paren{\frac{r\cdot \log q}{(\sigma_{r+1}^2+\kappa^2)\cdot(\sigma_{r+1}+\gamma)}}}\\
         &\leq n^{O\Paren{1}}\cdot\exp\Set{O\Paren{\frac{r\cdot \log q}{(\sigma_{r+1}^2+\kappa^2)\cdot(\sigma_{r+1}+\gamma)}}}
     \end{align*}
     as the first three steps of the algorithm require polynomial time.
     It remains to argue that any  solution of high objective value for the instance $\hat{\bm\cI}$ with label extended graph $\hat{\mathbf D}^{1/2}\hat{\mathbf{A}}\hat{\mathbf D}^{1/2}$  is also a good solution for the original instance $\cI$.
     To this end, observe now that for every $x\in[q]^n$ and corresponding indicator vector $\chi\in\set{0,1}^{nq}$
     \begin{align*}
         \Abs{\valI{\cI}{x}-\valI{\hat{\bm\cI}}{x}}&\leq \Abs{\iprod{\chi, (A -\hat{\mathbf{D}}^{1/2}\hat{\mathbf{A}}\hat{\mathbf{D}}^{1/2})\chi}}
     \end{align*}
    By \cref{fact:binary-quadratic-form-perturbation}
    \tom{Formally we cannot use that fact because we lose a q factor. We need to reprove it for sparse vectors (just a technicality, the inequality is true).}, 
    \begin{align*}
        \Abs{\iprod{\chi,A\chi}-\iprod{\chi,\hat{\mathbf D}^{1/2}\hat{\mathbf{A}}\hat{\mathbf D}^{1/2}\chi}}&\leq O\Paren{\frac{n\sqrt{\log n}}{\eps} + \sqrt{\frac{\Normo{A(G)}n\sqrt{\log n}}{\eps}}} + \Paren{2\sigma_{r+1}+\frac{\gamma}{16}}\Normo{A(G)}\\
        &\leq  \Paren{2\sigma_{r+1}+\frac{\gamma}{8}}\Normo{A(G)}\\
        &\leq O\Paren{\sigma_{r+1}+\gamma}\Normo{A(G)}
    \end{align*}
    where we used the bound on  $d_{\min}(G)$ and the fact that $\OPT\geq\Normo{A(G)}/4.$
    This concludes the proof.
\end{proof}

\subsection{Maximum bisection under differential privacy}\label{sec:dp-max-bisection}
As without privacy we can immediately extend \cref{thm:dp-max-cut} and \cref{thm:dp-2-csp} to settings with global constraints.
To illustrate it we prove the following theorem.

\begin{theorem}[Edge-DP \maxbisection]\label{thm:dp-max-bisection}
    Let $\eps,\delta,\kappa\in [0,1]$ with $\delta\geq 10n^{-100}.$ Let $C>0$ be a large enough constant. There exists an $(\eps,\delta)$-DP algorithm that, given a graph $G$, $\eps,\delta,$ an integer $r>0,$ with probability at least $1-n^{-O(1)}$ returns a bipartition $(L,R)$ such that the total weight of cut edges is at least
    \begin{align*}
        %\OPT -O\Paren{\frac{n\log n}{\eps}} - 
        \Paren{1 -O(\sigma_{r+1}+\kappa+\gamma)}\cdot\OPT
    \end{align*}
    and $\min\set{\card{L},\card{R}}\geq \tfrac{n}{2}-\tilde{O}(\sqrt{n}),$
    whenever $G$ has 
    \begin{align*}
        d_{\min}\geq O\Paren{\frac{\sqrt{\log(1/\delta)}}{\eps}\cdot \frac{\sqrt{r\cdot\mu_r+\log n}}{\gamma^2\cdot(\sigma_r -\sigma_{r+1})}},\qquad \sigma_r\geq 2\gamma + 3\sigma_{r+1}.
    \end{align*}
    Moreover, the algorithm runs in  time
    \begin{align*}
        n^{ O\Paren{1}}\cdot \exp\Set{O\Paren{\frac{r}{(\sigma_{r+1}^2+\kappa^2)\cdot(\sigma_{r+1}+\gamma)}}}.
    \end{align*}
\end{theorem}
We defer the proof to \cref{sec:deferred-proofs}.

\phantomsection
\addcontentsline{toc}{section}{Bibliography}
{\footnotesize
\bibliographystyle{amsalpha} 
\bibliography{scholar}
}
\clearpage
\appendix
\section{Deferred proofs}\label{sec:deferred-proofs}
We present here proofs deferred in the main  body of the paper.

\subsection{Deferred proofs of \cref{sec:privatization-mechanism}}\label{sec:deferred-proofs-coherence}
\begin{proof}[Proof of \cref{lem:coherence-sensitivity}]
Let $U, U'\in \R^{d\times r}$ be matrices whose columns are $r$ leading singular vectors of $M$ and $M'=M+E$ respectively. By triangle inequality,
    \begin{align*}
    \sqrt{\coh{r}{M + E}} 
    &= 
    \sqrt{\frac{n}{r}}\norm{U'}_{2\to\infty} 
    \\&\le \sqrt{\frac{n}{r}}\norm{U U^\top U'}_{2\to\infty}
    + \sqrt{\frac{n}{r}}\norm{U' - UU^\top U'}_{2\to\infty}
    \\&\le \sqrt{\frac{n}{r}}\norm{U}_{2\to\infty}
    +\sqrt{\frac{n}{r}}\norm{U' - UU^\top U'}_{2\to\infty}
    \\&=\sqrt{\coh{r}{M}} 
    +\sqrt{\frac{n}{r}}\norm{U' - UU^\top U'}_{2\to\infty}\,,
    \end{align*}
    where we used the fact that $\normcol{AB}\le \normcol{A}\norm{B}$ for all matrices $A,B$.
    By \cref{fact:perturbation},
    \[
    \norm{U' - UU^\top U'}_{2\to\infty} \le \norm{U' - UU^\top U'} \le \frac{4\norm{E U}}{\sigma_r - \sigma_{r+1}} 
    \le \frac{4\normf{E U}}{\sigma_r - \sigma_{r+1}}\,.
    \]
    Bu H\"older's inequality,
    \[
     \normf{E U} = \sqrt{\Iprod{EU, EU}}
     = \sqrt{\Iprod{UU^\top, E^\top E}}
     \le \sqrt{\sum_{ij}\abs{\paren{EE^\top}_{ij}}} \cdot \sqrt{\normi{P}} \le \Delta  \sqrt{\frac{r}{n}\coh{r}{M}} \,.
    \]
    Hence 
    \[
    \sqrt{\coh{r}{M + E}} \le \Paren{1 + O\Paren{\frac{\Delta}{\sigma_{r} - \sigma_{r+1}}}}\coh{r}{M}\,.
    \]
    Since $\sigma_r-\sigma_{r+1} > 2\Delta$, $\Paren{1 + O\Paren{\frac{\Delta}{\sigma_{r} - \sigma_{r+1}}}}^2 \le \Paren{1 + O\Paren{\frac{\Delta}{\sigma_{r} - \sigma_{r+1}}}}$, and we get the desired bound.
\end{proof}

\begin{proof}[Proof of \cref{thm:coherence-gaussian}]

Note that it is enough to show these bounds for $\normi{\hat{\mathbf P}}$ in terms of $\normi{UU^\top}$, and $\normi{\hat{\mathbf Q}}$ in terms of $\normi{VV^\top}$. Without loss of generality, consider the left singular spaces, i.e. $\hat{\mathbf P}$ and $P=UU^\top$.

Observe that if $r > n/2$ then the statement is true: Both $\coh{r'}{A+\mathbf  W}$ and $\coh{r}{A}$  are at most $2$ and at least $1$. Further we assume that $r\le n/2$.

Let $P = UU^\top$ and $\hat{\mathbf P} = \hat{\mathbf U}\hat{\mathbf U}^\top$ such that $\hat{\mathbf U}^\top \hat{\mathbf U} = \Id_{r'}$.
    Denote $\mathbf  O = \hat{\mathbf U}^\top U$ and $\mathbf  E = \hat{\mathbf U}\mathbf O - U $. Hence
    \[
    \hat{\mathbf U}\mathbf O = P\hat{\mathbf U}\mathbf O + \Paren{\Id_n - P}\hat{\mathbf U}\mathbf O = UU^\top\hat{\mathbf U}\mathbf O + \Paren{\Id_n - P}\hat{\mathbf U}\mathbf O 
    = U + UU^\top \mathbf E + \Paren{\Id_n - P}\hat{\mathbf U}\mathbf O \,.
    \]
    Note that 
    $\normi{\hat{\mathbf P}} = \normcol{\hat{\mathbf U}\mathbf O}^2\,,$
    and similarly $\normi{P} = \normcol{U}^2$. Hence
    \begin{align*}
    \Abs{\sqrt{\normi{\hat{\mathbf P}} } - \sqrt{\normi{P}}} 
    &=
    \Abs{\normcol{\hat{\mathbf U}\mathbf O} - \normcol{U}}
    \\&\le \normcol{UU^\top \mathbf E} + \normcol{\Paren{\Id_n - P}\hat{\mathbf U}\mathbf O}
    \\&\le \normcol{U}\norm{U^\top \mathbf E} +  \normcol{\Paren{\Id_n - P}\hat{\mathbf U}}\norm{\mathbf O}
    \\&\le \norm{\mathbf E}{\sqrt{\normi{P}}} + \normcol{\Paren{\Id_n - P}\hat{\mathbf U}}\,.
    \end{align*}
    
    Let us bound $\normcol{\Paren{\Id_n - P}\hat{\mathbf U}}$. Let $U_\perp \in \R^{n\times(n-r)}$ be a matrix such that $\Id_n - P = U_\perp U_\perp^\top$ and $U_\perp^\top U_\perp = \Id_{n-r}$. Let $\mathbf R\in \R^{(n-r) \times (n-r)}$ be a random orthogonal matrix independent of $\mathbf  W$ (and, hence, $\hat{\mathbf U}$), and let $T = P + U_\perp \mathbf R U_\perp^\top$. Note that $\mathbf T$ is orthogonal: $\mathbf T\mathbf T^\top = P + (\Id_n-P) = \Id_n$. Since $\mathbf T A  = A$ and $\mathbf T\mathbf W  $ has the same distribution as $\mathbf W$, $\mathbf  T\hat{\mathbf P}\mathbf T^\top$ has the same distribution as $\hat{\mathbf P}$, and $\mathbf T\hat{\mathbf U}$ has the same distribution as $\hat{\mathbf U}$. Hence it is enough to bound $\normcol{\Paren{\Id_n - P}\mathbf T\hat{\mathbf U}} = \normcol{U_\perp \mathbf R U_\perp^\top\hat{\mathbf U}}$. 
    
    For each $i\in [n]$ consider $F_i : \R^{(n-r)\times(n-r)} \to \R^{r'}$ defined as\footnote{While $F_i$ depend on random variable $\hat{\mathbf U}$, we do not write them in boldface to avoid confusion. We study them as functions of $\mathbf R$, and since $\hat{\mathbf U}$ and $\mathbf R$ are independent, we can treat them as fixed (non-random) functions.} $F_i(X) = \Paren{U_\perp^\top}_i XU_\perp^\top\hat{\mathbf U}$, where $\Paren{U_\perp^\top}_i$ is the $i$-th row of $U_\perp$. Since $\max_{i\in[n]} \norm{F_i(\mathbf R)} = \normcol{U_\perp  \mathbf R U_\perp^\top\hat{\mathbf U}}$, 
    we need to bound $ \norm{F_i(\mathbf R)}$ for each $i\in [n]$. Note that each $F_i$ is linear, and is $1$-Lipschitz since for each $X$ such that $\normf{X} \le 1$, 
    \[
    \norm{F_i(X)} 
    \le \norm{\Paren{U_\perp^\top}_i} \cdot \norm{X}\cdot \norm{U_\perp^\top}\cdot \norm{\hat{\mathbf U}}
    \le 1\,.
    \]
    Below we show that the norms of $F_i(\mathbf R)$ admit a  concentration bound similar to the norm of $N(0, 1/(n-r))^{r'}$.

    Let $\phi: \R^{r'} \to \R$ be an arbitrary $1$-Lipschitz function. The functions $\phi \circ F_i : \R^{(n-r)\times(n-r)} \to \R$ are $1$-Lipschitz, and hence by Theorem 5.2.7 from \cite{vershynin2009high}, for some absolute constant $C'\ge 1$ and all $t \ge 0$,
    \[
    \Pr\Paren{\Abs{\phi(F_i(\mathbf R)) - \E \phi(F_i(\mathbf R))} \ge t} \le 2\exp\Paren{-t^2\paren{n-r}/C'}\,.
    \]
    Hence by Theorem 2.3 from \cite{adamczak2015note}, ${r'}$-dimensional random vectors $F_i(\mathbf R)$ satisfy the Hanson-Wright concentration inequality. That is, for some absolute constant $C''$ and for all $t\ge 0$,
    \[
    \Pr\Paren{\norm{F_i(\mathbf R)}^2 - \E \norm{F_i(\mathbf R)}^2 \ge t}\le 2\exp\Paren{-\min\Set{t^2(n-r)^2/r, t(n-r)}/C''}\,.
    \]
    Therefore, by union bound, with probability at least $1-p$, for all $i\in [n]$, $\norm{F_i(\mathbf R)} \le O\Paren{\sqrt{\frac{r' + \log\paren{n/p}}{n}}}$ (here we used that $r\le n/2$). 
    
    Therefore, we get
    \[
    \Abs{\sqrt{\normi{\hat{\mathbf P}}} - \sqrt{\normi{P}}} \le 
    \norm{\mathbf E}\sqrt{\normi{P}} +O\Paren{\sqrt{\frac{r' + \log\paren{n/p}}{n}}}\,.
    \]
    Since $\norm{\mathbf E}\le 1$, we immediately get the desired upper bound. If $\norm{\paren{\Id_n - \hat{\mathbf P}}U} = \norm{U-\hat{\mathbf U}\mathbf O}\le 0.99$, then $\norm{\mathbf E}\le 0.99$, and, after rearranging, we get the desired lower bound. 
\end{proof}

\subsection{Deferred proofs of \cref{sec:threshold-rank}}
We present here the proof of \cref{fact:distance-normalized-adjacency-matrices},
\begin{proof}[Proof of \cref{fact:distance-normalized-adjacency-matrices}]
    For simplicity let $A,\bar{A}$ and $D$ be respectively the adjacency matrix, the normalized adjacency matrix, and the degree profile of $G.$ Similarly define $A',\bar{A}',D'$ for $G'.$
    Suppose without loss of generality that $G'$ is obtained from $G$ removing edge $ab$ with weight $1.$ %This assumption can be made without loss of generality since the minimum degree can then only decrease. 
    We may further assume $d_{\min}(G),d_{\min}(G')\geq 1$ since otherwise the statement is trivially true.
    Now, notice that $\bar{A},\bar{A}'$ differ only in rows $a,b$ and columns $a,b$.
    %Let $M=A-A'.$ Then it holds that 
    %\begin{align*}
    %    \Normo{\dyad{M}}&=\Normf{M_a}^2+\Normf{M_b}^2+ 2\abs{\iprod{M_a,M_b}} + \sum_{i,j\notin\set{a,b}} \Abs{M_{ia}M_{ja}+M_{ib}M_{jb}}\\
    %    &\leq \Normf{M_a}^2+\Normf{M_b}^2+ 2\Normf{M_a}\Normf{M_b}+\Normo{M}^2\\
    %    &\leq 2\Normf{M_a}^2+2\Normf{M_b}^2+\Normo{M}^2\\
    %    &\leq 5\Normo{M}^2\\
    %    &=5\Normo{A-A'}^2.
    %\end{align*}
    Therefore it suffices to bound the $\ell_1$-norm of $A-A'.$ To this end
    \begin{align*}
        \Normo{\bar{A}-\bar{A}'}&=\Normo{D^{-1/2}AD^{-1/2}-D'^{-1/2}A'D'^{-1/2}}\\
        &=\Normo{D^{-1/2}AD^{-1/2}-D'^{-1/2}(A'-A+A)D'^{-1/2}}\\
        &\leq \Normo{D^{-1/2}AD^{-1/2}-D'^{-1/2}AD'^{-1/2}} + \Normo{D'^{-1/2}(A'-A)D'^{-1/2}}.
    \end{align*}
    We rewrite the first term as
    \begin{align*}
        \Normo{D^{-1/2}AD^{-1/2}-D'^{-1/2}AD'^{-1/2}} &=2\sum_{\ij \in E(G)} \Abs{\frac{w(ij)}{\sqrt{d(i)d(j)}} -\frac{w(ij)}{\sqrt{d'(i)d'(j)}}}\\
        &=2\sum_{\ij \in E(G)}w(ij)\Abs{\frac{\sqrt{d'(i)d'(j)}-\sqrt{d(i)d(j)}}{\sqrt{d(i)d(j)d'(i)d'(j)}}}\,.
    \end{align*}
    As the two sums only differ in terms corresponding to edges incident to $a$ or $b,$ we bound
    \begin{align*}
        \sum_{j\in N_{G'}(a)}w(ij)\Abs{\frac{\sqrt{(d(a)-1)d(j)}-\sqrt{d(a)d(j)}}{\sqrt{d(a)d(j)^2(d(a)-1)}}} \leq \sum_{j\in N_{G'}(a)}\frac{w(ij)}{\sqrt{d(j)d(a)}}\Abs{\frac{1-\sqrt{1-\frac{1}{d(a)}}}{\sqrt{1-\frac{1}{d(a)}}}}\leq \frac{2}{\sqrt{d(a)d_{\min}(G)}}\leq \frac{2}{d_{\min}(G)}\,,
    \end{align*}
    and so
    \begin{align*}
        \sum_{j\in N_{G}(a)}\Abs{\frac{w(ij)}{\sqrt{d(i)d(j)}} -\frac{w(ij)}{\sqrt{d'(i)d'(j)}}} \leq \frac{2}{d_{\min}(G)} + \frac{1}{\sqrt{d(b)d(a)}}\leq \frac{3}{d_{\min}(G)}.
    \end{align*}
    %For the edge $ab$ we use the bound
    %\begin{align*}
     %   \Abs{\frac{\sqrt{d'(a)d'(b)}-\sqrt{d(a)d(b)}}{\sqrt{d(a)d(b)d'(a)d'(b)}}}&\leq  \frac{1}{\sqrt{d(a)d(b)}} \Abs{\frac{1-\sqrt{\Paren{1-\frac{1}{d(a)}}\Paren{1-\frac{1}{d(b)}}}}{\sqrt{\Paren{1-\frac{1}{d(a)}}\Paren{1-\frac{1}{d(b)}}}}}\leq \frac{2}{d_{\min}(G)}.
    %\end{align*}
    Repeating the argument for $N_G(b)$, we get $\Normo{D^{-1/2}AD^{-1/2}-D'^{-1/2}AD'^{-1/2}}\leq \frac{6}{d_{\min}(G)}.$
    For the second term we immediately have 
    \begin{align*}
        \Normo{D'^{-1/2}(A'-A)D'^{-1/2}}=\frac{2}{\sqrt{d'(a)d'(b)}} \leq  \frac{2}{d_{\min}(G')}
    \end{align*}
     implying $\Normo{\bar{A}-\bar{A}'}\leq 8/\min\set{d_{\min}(G),d_{\min}(G')}.$
     Finally, applying \cref{fact:adjacency-l1} the statement follows.
\end{proof}

\subsection{Deferred proofs of \cref{sec:applications}}
We prove here \cref{thm:max-bisection}. To do so we state an extension of \cref{lem:driving-down-global-correlation}, again taken from previous work.

\begin{lemma}[Driving down global correlation, \cite{barak2011rounding,raghavendra2012approximating}]\label{lem:driving-down-global-correlation-2}
    Let $A\in \R^{nq\times nq}$ be symmetric. 
    There exists an algorithm that, given $A$, runs in randomized time $q^{O(1/\eta)}n^{O(1)}$ and returns a degree-$2$ pseudo-distribution $\zeta$, consistent with \ref{eq:basic-sdp}$\cup \set{\sum_{i}x_{i1}=\tfrac{n}{2}}$ satisfying
    \begin{enumerate}
        \item $\tilde{\E}_\zeta\Brac{\iprod{x, Ax}}\geq \OPT,$
        \item $\GC_{D}(\zeta)\leq \eta.$
    \end{enumerate}
\end{lemma}

Next we prove the theorem.
\begin{proof}[Proof of \cref{thm:max-bisection}]
    The proof proceeds as for \cref{thm:max-cut} with the difference that we apply \cref{lem:driving-down-global-correlation-2} in place of \cref{lem:driving-down-global-correlation}.
    Because we use \cref{alg:independent-rounding} to round the pseudo-distribution into an integral solution, by standard concentration of measure arguments we get that for the returned partition $(L,R)$ it holds with probability at least $1-n^{-O(1)}$, $\min\set{\card{L},\card{R}}\geq \tfrac{n}{2}-O(\sqrt{n\log n}).$ The result follows repeating the algorithm $n^{O(1)}$ times and picking the best solution.
\end{proof}

We prove \cref{fact:binary-quadratic-form-perturbation}.
\begin{proof}[Proof of \cref{fact:binary-quadratic-form-perturbation}]
    For any $x\in\R^n$ we may rewrite
    \begin{align*}
        \iprod{x,Ax} &= \iprod{x,D^{1/2}\bar{A}D^{1/2}x}\\
        &=\iprod{x,D^{1/2}(\bar{A}-\hat{A}+\hat{A})D^{1/2}x}\\
        &=\iprod{x,D^{1/2}\hat{A}D^{1/2}x} + \iprod{x,D^{1/2}(\bar{A}-\hat{A})D^{1/2}x}.
    \end{align*}
    The second term can be bounded by
    \begin{align*}
        \iprod{x,D^{1/2}(\bar{A}-\hat{A})D^{1/2}x}&\leq \Snormt{D^{1/2}x}\cdot\Norm{\bar{A}-\hat{A}}\leq \gamma\Snormt{D^{1/2}x}\leq \Normi{x}\cdot\gamma\cdot\Normo{D}. 
    \end{align*}
    We rewrite the first term as
    \begin{align*}
        \iprod{x,D^{1/2}\hat{A}D^{1/2}x} &=\iprod{x,(D^{1/2}-\hat{D}^{1/2}+\hat{D}^{1/2})\hat{A}(D^{1/2}-\hat{D}^{1/2}+\hat{D}^{1/2})x}\\
        &= \iprod{x, \hat{D}^{1/2}\hat{A}\hat{D}^{1/2}x} + 2\iprod{x,(D^{1/2}-\hat{D}^{1/2})\hat{A}\hat{D}^{1/2}x}+\iprod{x,(D^{1/2}-\hat{D}^{1/2})\hat{A}(D^{1/2}-\hat{D}^{1/2})x}.
    \end{align*}
    Again we bound each term separately:
    \begin{align*}
        \iprod{x,(D^{1/2}-\hat{D}^{1/2})\hat{A}\hat{D}^{1/2}x} &\leq\Norm{(D^{1/2}-\hat{D}^{1/2})x}\cdot \Norm{\hat{A}}\cdot \Norm{\hat{D}^{1/2}x}\\
        &\leq\rho\cdot\Norm{(D^{1/2}-\hat{D}^{1/2})x}\cdot \Norm{\hat{D}^{1/2}x}\\
        &\leq \rho\cdot\Norm{(D^{1/2}-\hat{D}^{1/2})x}\cdot\sqrt{\Normo{\hat{D}}}\\
        &\leq\rho\cdot\Normi{x}\cdot\Normf{D^{1/2}-\hat{D}^{1/2}}\cdot\sqrt{\Normo{\hat{D}}}\\
        &\leq \rho\cdot\Normi{x}\cdot\sqrt{\beta\cdot n\cdot \Normo{\hat{D}}}\\
        &\leq\rho\cdot \Normi{x}\cdot\sqrt{\beta\cdot n\cdot (\Normo{D}+\beta n)}\\
        &\leq \rho\cdot\Normi{x}\cdot\Paren{\sqrt{\beta\cdot n\cdot \Normo{D}}+\beta\cdot n},
    \end{align*}
    and
    \begin{align*}
        \iprod{x,(D^{1/2}-\hat{D}^{1/2})\hat{A}(D^{1/2}-\hat{D}^{1/2})x} &\leq \Snorm{(D^{1/2}-\hat{D}^{1/2})x}\Norm{\hat{A}}\\
        &\leq \rho\cdot\Normi{x}^2\cdot\Normf{D^{1/2}-\hat{D}^{1/2}}^2\\
        &\leq \rho\cdot\Normi{x}^2\cdot\beta\cdot n.
    \end{align*}
    Putting things together the statement follows.
\end{proof}

Next we prove \cref{thm:dp-max-bisection}.
\begin{proof}[Proof of \cref{thm:dp-max-bisection}]
    The argument proceeds as \cref{thm:dp-max-cut} so we only sketch the proof. We use the same algorithm with the exception that we run the procedure of \cref{thm:max-bisection} in place of \cref{thm:max-cut} in step (4). hence by \cref{fact:max-cut-alg-dp} the algorithm is $(\eps,\delta)$-DP. By the analysis of \cref{thm:dp-max-cut} we obtain the error guarantees and the running time. As in \cref{thm:max-bisection} with high probability we obtain a  nearly balanced partition.
\end{proof}
\section{Background}
\subsection{Differential privacy}\label{sec:dp-background}
We recall here differential privacy and several common privatization mechanisms.
\begin{definition}[Differential privacy]
    An algorithm $\cM\,:\cY\to \cO$ is $(\eps,\delta)$-differentially private for $\eps,\delta>0$ if and only if, for all events $\cE$ in the output space $\cO$ and every adjacent $A,A'\in\cY\,,$
    \begin{align*}
        \bbP\Paren{\cM(A)\in \cE}\leq \exp(\eps)\cdot\bbP\Paren{\cM(A')\in \cE}+\delta.
    \end{align*}
\end{definition}
\noindent When $\cY$ is a set of graphs, differential privacy with respect to \cref{def:graph-adjacency} is often called edge-DP.
Differential privacy is closed under post-processing and composition.

\begin{lemma}[Post-processing]\label{lem:dp-post-processing}
	If $\cM:\cY\rightarrow \cO$ is an $(\eps, \delta)$-differentially private algorithm and $\cM':\cY\rightarrow \cZ$ is any randomized function. Then the algorithm $\cM'\Paren{\cM(Q)}$ is $(\eps, \delta)$-differentially private.
\end{lemma}
\noindent In order to talk about composition it is convenient to also consider DP algorithms whose privacy guarantee holds only against subsets of inputs. 

\begin{definition}[Differential Privacy Under Condition]\label{def:dp-under-condition}
	An algorithm $\cM:\cY\rightarrow\cO$ is said to be $(\eps, \delta)$-differentially private under condition $\Psi$ (or $(\eps, \delta)$-DP under condition $\Psi$) for $\eps, \delta >0$ if and only if, for every event $\cE$ in the output space and every neighboring $A,A' \in \cY$ both satisfying $\Psi$ we have
	\begin{align*}
		\bbP \Brac{\cM(A)\in \cE}\leq e^\eps\cdot \bbP \Brac{\cM(A')\in \cE}+\delta\,.
	\end{align*}
\end{definition}

\noindent It is not hard to see that the following composition theorem holds for privacy under condition.

\begin{lemma}[Composition for Algorithm with Halting]\label{lem:dp-composition}
	Let $\cM_1:\cY\rightarrow\cO_1\cup \Set{\bot}\,, \cM_2:\cO_1\times \cY\rightarrow \cO_2\cup \Set{\bot}\,,\ldots\,, \cM_t:\cO_{t-1}\times \cY\rightarrow \cO_t\cup \Set{\bot}$ be algorithms. Furthermore, let $\cM$ denote the algorithm that proceeds as follows (with $\cO_0$ being empty): For $i=1\,\ldots, t$ compute $o_i=\cM_i(o_{i-1}, Y)$ and, if $o_i=\bot$, halt and output $\bot$. Finally, if the algorithm has not halted, then output $o_t$.
	Suppose that:
	\begin{itemize}
		\item For any $1\leq i\leq t$, we say that $Y$ satisfies the condition $\Psi_i$ if running the algorithm on $Y$ does not result in halting after applying $\cM_1,\ldots, \cM_i$.
		\item  $\cM_1$ is $(\eps_1,\delta_1)$-DP.
		\item  $\cM_i$ is $(\eps_i, \delta_i)$-DP (with respect to neighboring datasets in the second argument) under condition $\Psi_{i-1}$ for all $i=\Set{2,\ldots, t}\,.$
	\end{itemize}
	Then $\cM$ is $\Paren{\sum_i \eps_i, \sum_i \delta_i}$-DP. 
\end{lemma}

The following composition theorem is a variant of the Propose–Test–Release paradigm of \cite{dwork2009differential}.

\begin{lemma}\label{lem:dp-comp}[Two-step composition with halting and per-$y$ good outputs]
Let $\cM_1:\cY\to\cO_1\cup\{\bot\}$ be $(\varepsilon_1,\delta_1)$-DP.
For each $y\in\cY$, let $A_y\subseteq\cO_1$ satisfy
\[
\Pr[\cM_1(y)\in A_y \mid \cM_1(y)\neq \bot]\ \ge\ 1-p.
\]
Let $\cM_2:\cO_1\times\cY\to\cO_2\cup\{\bot\}$ be such that for all neighboring
$y,y'$ and all $a\in A_y$,
the map $\cM_2(a,\cdot)$ is $(\varepsilon_2,\delta_2)$-DP.
Define the composition $\cM$ that runs $a\sim\cM_1(y)$, halts with $\bot$ if $a=\bot$,
else outputs $\cM_2(a,y)$. Then $\cM$ is
$(\varepsilon_1+\varepsilon_2,\ \delta_1+\delta_2+p)$-DP.
\end{lemma}

\begin{proof}
Fix neighboring $y,y'$ and a set $S\subseteq\cO_2\cup\{\bot\}$.
For $u\in\{y,y'\}$ and $a\in\cO_1$, write
\[
g_u(a)\;\defeq\Pr[\cM_2(a,u)\in S]\in[0,1].
\]
By conditioning on the output $a$ of $\cM_1(y)$,
\begin{align*}
\Pr[\cM(y)\in S]
&= \E_{a\sim \cM_1(y)}\!\Big[\1[a=\bot]\1_{\{\bot\in S\}}
\;+\;\1[a\in\cO_1]\cdot g_y(a)\Big] \\
&= \E_{a\sim \cM_1(y)}\!\Big[\1[a=\bot]\1_{\{\bot\in S\}}
\;+\;\1[a\in A_y]\,g_y(a)
\;+\;\1[a\in \cO_1\setminus A_y]\,g_y(a)\Big].
\end{align*}
For $a\in A_y$, by the $(\varepsilon_2,\delta_2)$-DP of $\cM_2(a,\cdot)$,
\[
g_y(a)\ \le\ e^{\varepsilon_2}g_{y'}(a)+\delta_2.
\]
For $a\in \cO_1\setminus A_y$, we use the trivial bound $g_y(a)\le 1$.
Therefore,
\begin{align*}
\Pr[\cM(y)\in S]
&\le \E_{a\sim \cM_1(y)}\!\Big[\1[a=\bot]\1_{\{\bot\in S\}}
\;+\;\1[a\in A_y]\min\{1,\,e^{\varepsilon_2}g_{y'}(a)+\delta_2\}\Big]
\;+\;\Pr[\cM_1(y)\in \cO_1\setminus A_y] \\
&\le \E_{a\sim \cM_1(y)}\!\Big[\underbrace{\1[a=\bot]\1_{\{\bot\in S\}}
\;+\;\1[a\in A_y]\min\{1,\,e^{\varepsilon_2}g_{y'}(a)\}}_{=:~\psi(a)\in[0,1]}\Big]
\;+\;\delta_2\;+\;\Pr[\cM_1(y)\in \cO_1\setminus A_y].
\end{align*}
By the definition of $A_y$,
\[
\Pr[\cM_1(y)\in \cO_1\setminus A_y]
= \Pr[\cM_1(y)\neq \bot]\cdot
   \Pr[\cM_1(y)\in \cO_1\setminus A_y \mid \cM_1(y)\neq\bot]
\le p.
\]
Apply $(\varepsilon_1,\delta_1)$-DP of $\cM_1$ to the bounded test function
$\psi\in[0,1]$:
\[
\E_{a\sim \cM_1(y)}[\psi(a)]
\ \le\ e^{\varepsilon_1}\E_{a\sim \cM_1(y')}[\psi(a)]\ +\ \delta_1.
\]
For the expectation under $y'$,
\[
\E_{a\sim \cM_1(y')}[\psi(a)]
\ \le\ \Pr[\cM_1(y')=\bot]\1_{\{\bot\in S\}}
\ +\ e^{\varepsilon_2}\E_{a\sim \cM_1(y')}[g_{y'}(a)].
\]
Since $e^{\varepsilon_1}\le e^{\varepsilon_1+\varepsilon_2}$, we conclude
\begin{align*}
\E_{a\sim \cM_1(y)}[\psi(a)]
&\le e^{\varepsilon_1+\varepsilon_2}\E_{a\sim \cM_1(y')}[g_{y'}(a)]
\ +\ e^{\varepsilon_1}\Pr[\cM_1(y')=\bot]\1_{\{\bot\in S\}}
\ +\ \delta_1 \\
&\le e^{\varepsilon_1+\varepsilon_2}\Big(\E_{a\sim \cM_1(y')}[g_{y'}(a)]
\ +\ \Pr[\cM_1(y')=\bot]\1_{\{\bot\in S\}}\Big)\ +\ \delta_1 \\
&= e^{\varepsilon_1+\varepsilon_2}\Pr[\cM(y')\in S]\ +\ \delta_1.
\end{align*}
Combining the displays,
\[
\Pr[\cM(y)\in S]
\ \le\ e^{\varepsilon_1+\varepsilon_2}\Pr[\cM(y')\in S]\ +\ \delta_1+\delta_2+p,
\]
which is exactly $(\varepsilon_1+\varepsilon_2,\ \delta_1+\delta_2+p)$-DP.
\end{proof}

\noindent The Gaussian  mechanism is among the most widely used mechanisms in differential privacy. %and the Laplace

\begin{definition}[Sensitivity]
    Let $f:\cY\to\R^n$ be a function. Its $\ell_1$-sensitivity and $\ell_2$-sensitivity are
    \begin{align*}
       \Delta_{1,f}:=\max_{\substack{A,A'\in\cY\\A,A\text{ are adjacent}}}\Normo{f(A)-f(A')}\qquad \Delta_{2,f}:=\max_{\substack{A,A'\in\cY\\A,A\text{ are adjacent}}}\Normt{f(A)-f(A')}.
    \end{align*}
    For a real-valued function $f$ the log-sensitivity is $\Delta_{\ell_1,\log f}.$
\end{definition}
\noindent For functions with low $\ell_2$-sensitivity the tool of choice is the Gaussian mechanism.
\begin{theorem}[Gaussian Mechanism]\label{thm:gaussian-mechanism}
    Let $f:\cY\to\R^n$ be any function with $\ell_2$-sensitivity $\Delta_{2,f}.$ Let $0<\eps,\delta\leq 1.$ Then the algorithm that adds $N\Paren{0,\frac{\Delta^2_{2,f}\cdot \log(2/\delta)}{\eps^2}\cdot \Id_n}$ to $f$ is $(\eps,\delta)$-differentially private.
\end{theorem}
\noindent For functions with low $\ell_1$-sensitivity it is common to use the Laplace mechanism.

\begin{definition}[Laplace distribution]\label{def:laplace-distribution}
	The Laplace distribution with mean $q$ and parameter $b>0$, denoted by $\text{Lap}(q, b)$, has PDF $\frac{1}{2b}e^{-\Abs{x-q}/b}\,.$
	Let $\Lap(b)$ denote $\Lap(0,b)$.
\end{definition}

A standard tail bound concerning the Laplace distribution will be useful throughout the paper. 

\begin{fact}[Laplace tail bound]\label{fact:laplace-tail-bound}
	Let $\bm x \sim \text{Lap}(q, b)$. Then,
	\begin{align*}
		\bbP \Brac{\abs{\bm x -q} > t}\leq e^{-t/b}\,.
	\end{align*}
\end{fact}

The Laplace distribution is useful for the following mechanism.

\begin{lemma}[Laplace mechanism]\label{lemma:laplace-mechanism}
	Let $f:\cY\rightarrow \R^n$ be any function with $\ell_1$-sensitivity at most $\Delta_{f,1}$. Then the algorithm that adds $\text{Lap}\Paren{\frac{\Delta_{f, 1}}{\eps}}^{\otimes n}$ to $f$ is $(\eps,0)$-DP.
\end{lemma}

The following mechanism applies the Laplace mechanism to the logarithm of the given function.

\begin{lemma}\label{lem:privatization-for-log-sensitivity}
    Let $a>0$ and let $f:\R^{n\times n}\to\R_{\geq 0}$ be a  function such that, on adjacent inputs $M,M',$ satisfies $f(M)/f(M)'\leq \Brac{1/a,a}.$ 
    There exists an $(\eps,0)$-DP algorithm that, on any input $M,$ returns $\hat{\bm f}(M)$ satisfying, with probabilty at least $1-p$,
    \begin{align*}
        \Abs{\hat{\bm f}(M)-f(M)}\leq a^{\frac{\log\frac{1}{p}}{\eps}}.
    \end{align*}
    \begin{proof}
        By definition for any pair of adjacent inputs $M,M'$
        \begin{align*}
            \abs{\log f(M)-\log f(M')}=\abs{\log \Paren{f(M)/f(M')}}\leq \log a.
        \end{align*}
        Hence to obtain an $(\eps,0)$-DP estimate of $\log f(M)$ we may apply the Laplace mechanism to $\log f.$ Let $\log \hat{\bm f}(M)$ be the resulting output. Then for a large enough constant $C>0,$ by \cref{fact:laplace-tail-bound}
        \begin{align*}
            \bbP\Paren{\abs{\log \hat{\bm f}(M)-\log f(M)}>\frac{\log a}{\eps}\log \tfrac{1}{p}}\leq p.
        \end{align*}
        Exponentiating the functions
        \begin{align*}
            \bbP\Paren{\abs{\hat{\bm f}(M)- f(M)}> a^{\frac{\log \frac{1}{p}}{\eps}}}\leq p.
        \end{align*}
    \end{proof}
\end{lemma}
\subsection{Sum-of-squares}\label{sec:sos-background}

We present here necessary background about the sum-of-squares framework.
See \cite{fleming2019semialgebraic} for proofs and more details.

Let $x = (x_1, x_2, \ldots, x_n)$ be a tuple of $n$ indeterminates and let $\R[x]$ be the set of polynomials with real coefficients and indeterminates $x_1,\ldots,x_n$.
In a \emph{polynomial feasibility problem}, we are given a system of polynomial inequalities $\cA = \{f_1 \geq 0, \dots, f_m \geq 0\}$,
and we would like to know if there exists a point $x \in \R^n$ satisfying $f_i(x) \geq 0$ for all $i \in [m]$.
This task is easily seen to be NP-hard.

Given a polynomial system $\cA$, the \emph{sum-of-squares (sos) algorithm} computes a \emph{pseudo-distribution} of solutions to $\cA$ if one exists. Pseudo-distributions are generalizations of probability distributions, therefore the sos algorithm solves a relaxed version of the feasibility problem. The search for a pseudo-distribution can be forzetalated as a semidefinite program (SDP).

There is strong duality between \emph{pseudo-distributions} and \emph{sum-of-squares proofs}: the sos algorithm will either find a pseudo-distribution satisfying $\cA$, or a refutation of $\cA$ inside the sum-of-squares proof system.
When using sos for algorithm design as we do here, we work in the former case and our goal is to design a rounding algorithm that transforms a pseudo-distribution into an actual point $x$ that satisfies or nearly satisfies $\cA$.
% In the latter case, the algorithm has successfully found a proof that the system $\cA$ is not satisfiable.

The side of the sum-of-squares algorithm which computes a pseudo-distribution is summarized into the following theorem (we will not need the side that computes a sum-of-squares refutation). The full definitions of these objects will be presented momentarily.

\begin{theorem}
\label{fact:running-time-sos}
	Fix a parameter $\ell \in \N$. There exists an $(n+ m)^{O(\ell)} $-time algorithm that, given an explicitly bounded and satisfiable polynomial system $\cA = \{f_1 \geq0, \dots, f_m \geq 0\}$ in $n$ variables with bit complexity $(n+m)^{O(1)}$, outputs a degree-$\ell$ pseudo-distribution that satisfies $\cA$ approximately.
\end{theorem}

\paragraph{Pseudo-distributions}
We can represent a discrete (i.e., finitely supported) probability distribution over $\R^n$ by its probability mass function $\zeta\from \R^n \to \R$ such that $\zeta \geq 0$ and $\sum_{x \in \mathrm{supp}(\zeta)} \zeta(x) = 1$.
A pseudo-distribution relaxes the constraint $\zeta\ge 0$ and only requires that $\zeta$ passes certain low-degree non-negativity tests.

Concretely, a \emph{degree-$\ell$ pseudo-distribution} is a finitely-supported function $\zeta:\R^n \rightarrow \R$ such that $\sum_{x \in \supp(\zeta)} \zeta(x) = 1$ and $\sum_{x \in \supp(\zeta)} \zeta(x) f(x)^2 \geq 0$ for every polynomial $f$ of degree at most $\ell/2$.
A straightforward polynomial interpolation argument shows that every degree-$\infty$ pseudo-distribution satisfies $\zeta\ge 0$ and is thus an actual probability distribution.

A pseudo-distribution $\zeta$ can be equivalently represented through its \emph{pseudo-expectation operator} $\tilde \E_\zeta$.
For a function $f$ on $\R^n$ we define the pseudo-expectation $\tilde{\E}_{\zeta} f(x)$ as
\begin{equation*}
	\tilde{\E}_{\zeta} f(x) = \sum_{x \in \supp(\zeta)} \zeta(x) f(x) \,\mper
\end{equation*}

We are interested in pseudo-distributions which satisfy a given system of polynomials $\cA$.

\begin{definition}[Satisfying constraints]
\label{def:constrained-pd}
	Let $\zeta$ be a degree-$\ell$ pseudo-distribution over $\R^n$.
	Let $\cA = \{f_1\ge 0, f_2\ge 0, \ldots, f_m\ge 0\}$ be a system of polynomial inequalities.
	We say that \emph{$\zeta$ is consistent with $\cA$ at level $r$}, denoted $\zeta \sdtstile{r}{} \cA$, if for every $S\subseteq[m]$ and every polynomial $h$ with $2\deg h + \sum_{i\in S} \max\{\deg f_i,\, r\}\leq \ell$,
	\begin{displaymath}
		\tilde{\E}_{\zeta} h^2 \cdot \prod _{i\in S}f_i  \ge 0\,.
	\end{displaymath}
	We say $\zeta$ satisfies $\cA$ and write $\zeta\sdtstile{}{} \cA$ if the case $r = 0$ holds.
\end{definition}

We remark that $\zeta \sdtstile{}{} \{1 \geq 0\}$ is equivalent to $\zeta$ being a valid pseudo-distribution, and if $\zeta$ is an actual (discrete) probability distribution, then we have  $\zeta\sdtstile{}{}\cA$ if and only if $\zeta$ is supported on solutions to the constraints $\cA$.

The pseudo-expectations of all polynomials in the variables $x$
with degree at most $\ell$ can be packaged into the list of \emph{pseudo-moments} $\tilde\E_\zeta x^S$ for all monomials $x^S, \, |S| \leq \ell$.
Since we will be entirely concerned with polynomials up to degree $\ell$, as in \cref{def:constrained-pd}, we can treat a degree-$\ell$ pseudo-distribution as being equivalently specified by the list of pseudo-moments up to degree $\ell$.
Thus we will view the output of the degree-$\ell$ sos algorithm as being the list of all pseudo-moments up to degree $\ell$ which has size $O(n^{\ell})$.

To design an algorithm based on sos, our task is to utilize the pseudo-moments in order to find a solution point $x$.
The sos framework extends linear programming and semidefinite programming, which conceptually use only the degree-1 or degree-2 moments respectively.
Taking sos to higher degree enforces additional constraints on all of the moments, coming from higher-degree sum-of-squares proofs as we will see next.

\paragraph{Sum-of-squares proofs}
We say that a polynomial $p\in \R[x]$ is a \emph{sum-of-squares (sos)} if there are polynomials $q_1,\ldots,q_r \in \R[x]$ such that $p=q_1^2 + \cdots + q_r^2$.
Let $f_1, f_2, \ldots, f_m, g \in \R[x]$.
A \emph{sum-of-squares proof} that the constraints $\{f_1 \geq 0, \ldots, f_m \geq 0\}$ imply the constraint $\{g \geq 0\}$ consists of  sum-of-squares polynomials $(p_S)_{S \subseteq [m]}$ such that
\begin{equation*}
	g = \sum_{S \subseteq [m]} p_S \cdot \Pi_{i \in S} f_i
	\mper
\end{equation*}
We say that this proof has \emph{degree $\ell$} if for every set $S \subseteq [m]$, the polynomial $p_S \Pi_{i \in S} f_i$ has degree at most $\ell$.
When a set of inequalities $\cA$ implies $\{g \geq 0\}$ with a degree $\ell$ SoS proof, we write:
\begin{equation*}
	\cA \sststile{\ell}{}\{g \geq 0\}
	\mper
\end{equation*}
A sum-of-squares \emph{refutation} of $\cA$ is a proof $\cA \sststile{\ell}{} \{-1 \geq 0\}$.

\paragraph{Duality}
Degree-$\ell$ pseudo-distributions and degree-$\ell$ sum-of-squares proofs exhibit strong duality.
In proof theoretic terms, degree-$\ell$ sum-of-squares proofs are sound and complete when degree-$\ell$ pseudo-distributions are taken as models.

Soundness, or weak duality, states that every sum-of-squares proof enforces a constraint on every valid pseudo-distribution.

\begin{fact}[Weak duality/soundness]
	\label{fact:sos-soundness}
	If $\zeta \sdtstile{r}{} \cA$ for a degree-$\ell$ pseudo-distribution $\zeta$ and there exists a sum-of-squares proof $\cA \sststile{r'}{} \cB$, then $\zeta \sdtstile{r\cdot r'+r'}{} \cB$.
\end{fact}

Although we will not need it in our analysis, strong duality a.k.a (refutational) completeness conversely shows that for a given set of axioms, there always exists either a degree-$\ell$ pseudo-distribution or a degree-$\ell$ sos refutation. %every property of degree-$\ell$ pseudo-distributions can be derived by degree-$\ell$ sum-of-squares proofs.
% \jiyu{actually I don't quite understand what does strong duality mean here, I'll ask you when I finish reading sec 6}

\begin{fact}[Strong duality/refutational completeness]
	\label{fact:sos-completeness}
	Suppose $\cA$ is a collection of polynomial constraints such that $\cA \sststile{\ell - r}{} \{ \sum_{i = 1}^n x_i^2 \leq B\}$ for some finite $B$.
    If there is no degree-$\ell$ pseudo-distribution $\zeta$ such that $\zeta \sdtstile{r}{} \cA\,$, then there is a sum-of-squares refutation $\cA \sststile{\ell - r}{}\{-1 \geq 0\}$.
\end{fact}

\paragraph{Implementation of sos}
The sum-of-squares algorithm can be implemented as a semidefinite program (SDP) which can then be solved using, for example, the ellipsoid method.
Associated with a degree-$\ell$ pseudo-distribution $\zeta$ is the \emph{moment tensor} which is the tensor $\tilde{\E}_{\zeta} (1,x_1, x_2,\ldots, x_n)^{\otimes \ell}$.
When $\ell$ is even, this tensor can be flattened into the \emph{moment matrix}, which has rows and columns indexed by zetaltisets of $[n]$ with size at most $\ell/2$ and whose $(I,J)$ entry is $\tilde \E_\zeta x^I x^J$.
Moment matrices can now be characterized as positive semidefinite matrices with simple symmetry constraints from flattening.

\begin{fact}
    A matrix $\Lambda$ with rows and columns indexed by zetaltisets of $[n]$ with size at most $\ell$ is a moment matrix of a degree-$2\ell$ pseudo-distribution if and only if:
    \begin{enumerate}[(i)]
        \item $\Lambda \sge 0$
        \item $\Lambda_{I,J} = \Lambda_{I', J'}$ whenever $I \cup J = I' \cup J'$ as zetaltisets
        \item $\Lambda_{\{\}, \{\}} = 1$
    \end{enumerate}
\end{fact}

The above characterization of pseudo-distributions in terms of the cone of positive semidefinite matrices is a forzetalation of the sos algorithm as an SDP.

We can deduce \cref{fact:running-time-sos} from the general theory of convex optimization \cite{grotschel2012geometric}.
The above fact leads to an $n^{O(\ell)}$-time weak separation oracle
for the convex set of all moment tensors of degree-$\ell$ pseudo-distributions over $\R^n$.
By the results of \cite{grotschel1981ellipsoid}, we can optimize over the set of pseudo-distributions in time $n^{O(\ell)}$, assuming numerical conditions.

The first numerical condition is that the bit complexity of the input to the sos algorithm is polynomial.
The second numerical condition is that we assume an upper bound on the norm of feasible solutions. This is guaranteed
if the input polynomial system $\cA$ is \emph{explicitly bounded},
meaning that it contains a constraint of the form $\|x\|^2 \leq M$ for some $M \geq 0$ with polynomial bit length,
or if $\cA \sststile{\ell}{} \{\|x\|^2 \leq M\}$. For example, Boolean constraints satisfy this since $\{x_i^2 = x_i\}_{i \in [n]} \sststile{2}{} \{\|x\|^2 \leq n\}$.

Due to finite numerical precision, the output of the sos algorithm can only be computed approximately, not exactly. For a pseudo-distribution $\zeta\,$, we say that $\zeta\sdtstile{r}{}\cA$ holds \emph{approximately} if the inequalities in \cref{def:constrained-pd} are satisfied up to an error of $2^{-n^\ell}\cdot \norm{h}\cdot\prod_{i\in S}\norm{f_i}$, where $\norm{\cdot}$ denotes the Euclidean norm of the coefficients of a polynomial in the monomial basis.\footnote{The choice of norm is not important here because the factor $2^{-n^\ell}$ swamps the effect of choosing another norm.} In our analysis, the approximation error is so minuscule that it can be ignored and we will simply assume that the pseudo-distribution $\zeta$ computed by the sos algorithm satisfies $\cA$ without error.

\subsection{Matrix perturbation theory}\label{sec:perturbation}

\begin{theorem}[Wedin's Theorem, \cite{stewart1990matrix}]\label{thm:wedin}
Let $M, M'\in\R^{m\times n}$ and  let \( M = U \Lambda V^\top + U_\perp \Lambda_\perp V_\perp^\top \) and \( M' = \tilde{U} \tilde{\Lambda} \tilde{V}^\top + \tilde{U}_\perp \tilde{\Lambda}_\perp \tilde{V}_\perp^\top \) 
be their singular value decompositions such that $U,\tilde{U}\in \R^{m\times r}$, $V,\tilde{V}\in \R^{n\times r}$. 
If \( \sigma_{\min}(\Lambda) > \alpha + \gamma\) and 
\( \sigma_{\max}(\tilde{\Lambda}_\perp ) \le \alpha\)
for some \( \alpha, \gamma > 0 \), then
\[
\normf{\Paren{\Id_m - \tilde{U}\tilde{U}^\top}U}^2 + \normf{\Paren{\Id_n - \tilde{V}\tilde{V}^\top}V}^2\le \frac{\normf{\Paren{M'-M}V}^2 + \normf{U^\top\Paren{M'-M}}^2}{\gamma^2}\,,
\]
and
\[
\max\Set{\norm{\Paren{\Id_m - \tilde{U}\tilde{U}^\top}U}, \norm{\Paren{\Id_n - \tilde{V}\tilde{V}^\top}V}}
\le
\frac{\max\Set{\norm{\Paren{M'-M}V}, \norm{U^\top\Paren{M'-M}}}}{\gamma}\,.
\]
\end{theorem}

\begin{theorem}[Weyl's inequality for singular values, \cite{stewart1990matrix}]\label{thm:weyl}
    Let $M, M'\in\R^{m\times n}$, then for all $1\le k\le \min\set{\rank(M), \rank(M')}$,
    \[
    \abs{\sigma_k(M') - \sigma_k(M)} \le \sigma_1(M'-M)\,.
    \]
\end{theorem}

\begin{fact}\label{fact:perturbation}
Let $M, M'\in\R^{n\times n}$ and let \( M = U \Lambda V^\top + U_\perp \Lambda_\perp V_\perp^\top \) and \( M' = \tilde{U} \tilde{\Lambda} \tilde{V}^\top + \tilde{U}_\perp \tilde{\Lambda}_\perp \tilde{V}_\perp^\top \) 
be their singular value decompositions such that $U,\tilde{U}\in \R^{n\times r}$, $V,\tilde{V}\in \R^{n\times r}$. 
Suppose in addition that $M$ is symmetric.
If \( \sigma_{\min}(\Lambda) - \sigma_{\max}(\Lambda_\perp) > \sigma_{\max}(M'-M) + \gamma\), then
\[
\normf{UU^\top- \tilde{U}\tilde{U}^\top}^2 = 2\cdot \normf{\Paren{\Id_n - \tilde{U}\tilde{U}^\top}U}^2 \le 2\cdot \frac{\normf{\Paren{M'-M}U}^2 + \normf{\Paren{M'-M}^\top U}^2}{\gamma^2}\,,
\]
and
\[
\norm{UU^\top- \tilde{U}\tilde{U}^\top}=\norm{\Paren{\Id_n - \tilde{U}\tilde{U}^\top}U} \le \frac{\max\Set{\norm{\Paren{M'-M}U},\norm{\Paren{M'-M}^\top U}}}{\gamma}\,.
\]
\end{fact}
\begin{proof}
    Let $\alpha =  \sigma_{\max}(\Lambda_\perp) + \sigma_{\max}(M'-M) $. By \cref{thm:weyl},  $\sigma_{\max}(\tilde{\Lambda}_\perp ) \le \alpha$, and hence we can apply \cref{thm:wedin}.
    Since $M$ is symmetric, each column $V_i$ of $V$ is either $U_i$ or $-U_i$. 
    Hence the entries of $\Paren{M'-M}V$ have the same absolute values as the entries of $\Paren{M'-M}U$, and we get the desired bounds on the norms of $\Paren{\Id_n - \tilde{U}\tilde{U}^\top}U$. The equalities follow from Theorem I.5.5 from \cite{stewart1990matrix}.
\end{proof}

\section{Relations between notions of matrix adjacency}\label{sec:adjacency}

A recent work \cite{nicolas2024differentially} used the following notion of adjacency: $A,A'\in\R^{n\times n}$ are adjacent, if $\sqrt{\sum_{k=1}^n \Paren{\sum_{l=1}^n \abs{E_{kl}}}^2}$, where $E=A'-A$. The next proposition shows that our adjacency notion (\cref{def:matrix-adjacency}) is strictly more general:

\begin{fact}\label{fact:adjacency-general}
For all symmetric matrices $E\in\R^{n\times n}$,
    \[
     \sqrt{\sum_{1\le i,j\le n} \abs{\paren{EE^\top}_{ij}}}\le \sqrt{\sum_{k=1}^n \Paren{\sum_{l=1}^n \abs{E_{kl}}}^2}\,.
    \]
Furthermore, for each $n\in \N$, there exists a matrix $E\in \R^{2n\times 2n}$ such that 
\[
\sqrt{\sum_{1\le i,j\le n} \abs{\paren{EE^\top}_{ij}}} \le \frac{C}{\sqrt{n}}\cdot 
\sqrt{\sum_{k=1}^n \Paren{\sum_{l=1}^n \abs{E_{kl}}}^2}\,,
\]
where $C$ is some absolute constant.
\end{fact}

\begin{proof}
    \[
    \sum_{1\le i,j\le n} \abs{\paren{EE^\top}_{ij}} = \sum_{1\le i,j\le n} \abs{\sum_{k=1}^n E_{ik}E_{jk}}
    \le \sum_{1\le i,j\le n} {\sum_{k=1}^n \abs{E_{ik}}\abs{E_{jk}}} = \sum_{a=1}^n\sum_{1\le i,j\le n} { \abs{E_{ki}}\abs{E_{kj}}} = \sum_{k=1}^n \Paren{\sum_{l=1}^n \abs{E_{kl}}}^2\,.
    \]
    The example is the following matrix $E\in \R^{2n\times 2n}$ 
    \[
E = \begin{bmatrix}
0 & \mathbf R \\
\mathbf R^\top & 0
\end{bmatrix} \in \mathbb{R}^{2n\times 2n}\,,
\]
    where $\mathbf R \in \R^{n\times n}$ is a random rotation. Then $EE^\top = \Id_{2n}$, so $\sqrt{\sum_{1\le i,j\le n} \abs{\paren{EE^\top}_{ij}}} = \sqrt{2n}$, and each row of $E$ has $\ell_1$ norm $\Omega(\sqrt{n})$ with overwhelming probability, so 
    $\sqrt{\sum_{k=1}^n \Paren{\sum_{l=1}^n \abs{E_{kl}}}^2} \ge \Omega(n)$. 
\end{proof}

The following fact shows that our notion of adjacency is more general then the $\ell_1$ adjacency:

\begin{fact}\label{fact:adjacency-l1}
    Let $E\in \R^{n\times n}$ be a symmetric matrix. Then
    \[
    \sqrt{\sum_{1\le i,j\le n} \abs{\paren{EE^\top}_{ij}}} = \sqrt{\normo{EE^\top}} \le \normo{E} = \sum_{1\le i,j \le n} \abs{E_{ij}}\,.
    \]
\end{fact}
\begin{proof}
Note that
\[
\sum_{k=1}^n \Paren{\sum_{l=1}^n \abs{E_{kl}}}^2\le \Paren{\sum_{1\le i,j \le n} \abs{E_{ij}}}^2\,,
\]
hence the desired bound follows from \cref{fact:adjacency-l1}.
    % Let $a, b\in[n]$ be the row and the column of $E$ with nonzero entries. Then $\Paren{E^\top E}_{bb} = \sum_{k=1}^n E_{kb}^2 \le \normo{E}^2$, and if $i\neq b$ or $j\neq b$, $\Paren{E^\top E}_{ij} = E_{ai}E_{aj}$. Hence $\normo{E^\top E} \le \normo{E}^2 + \sum_{1\le i,j\le n} E_{ai}E_{aj} \le 2\normo{E}^2$. Since $\normi{P} = \frac{r}{n}\coh{r}{M}$, we get the desired bound.
\end{proof}

\begin{fact}\label{fact:adjacency-frob}
    Let $E\in \R^{n\times n}$ be a symmetric matrix. Then
    \[
    \normf{E}\le \sqrt{\sum_{1\le i,j\le n} \abs{\paren{EE^\top}_{ij}}}\,.
    \]
\end{fact}
\begin{proof}
\[
\normf{E}^2 = \normn{EE^\top} \le \normo{EE^\top} = \sum_{1\le i,j\le n} \abs{\paren{EE^\top}_{ij}}\,.
\]
\end{proof}
\section{Single spike principal component analysis}\label{sec:wishart}
\tom{@Gleb link this section where you make the claim in the introduction! Also can you give a more informative title to the section?}
In this section we prove the statements about PCA in Wishart model claimed in \cref{sec:introduction}. 
Recall that
\[
M = \sqrt{\beta} \cdot u \mathbf{g}^\top + \mathbf{W},
\]
where $u \in \mathbb{R}^n$ is a delocalized unit signal vector (i.e., the entries of $u$ are at most $\tilde{O}(\sqrt{1/n})$), $\mathbf{g} \sim N(0,1)^m$, and $\mathbf{W} \sim N(0,1)^{n \times m}$ are independent.

It is known that in the large-sample regime $m \gg n$, if $\beta = C\sqrt{n/m}$ with a sufficiently large constant $C$, the top left singular vector of $M$ is highly correlated with $u$ with high probability \cite{johnstone2001distribution, berthet2013optimal, d2020sparse}. 
Let us show that the spectral gap is $\Theta\Paren{\beta\sqrt{n}}$. Consider $MM^\top$:
\[
\beta m uu^\top + \sqrt{\beta} \cdot u \mathbf{g}^\top\mathbf{W}^\top + \sqrt{\beta}\mathbf{W}\mathbf{g}\cdot u^\top
+ \mathbf{W}\mathbf{W}^\top\,.
\]
Let $v$ be the top eigenvector of $M$. Since it has correlation at least $0.99$ with $u$,
\[
v^\top MM^\top v \ge 0.99\beta m - O\Paren{\sqrt{\beta m n}} + m - O\Paren{\sqrt{m n}}\,.
\]
For sufficiently large $C$, this value is at least $m + 0.9\beta m$. For every vector $v_\perp$ orthogonal to $v$, its correlation with $u$ is at most $0.1$, hence
\[
v_\perp^\top MM^\top v_\perp \le 0.1\beta m  - O\Paren{\sqrt{\beta m n}} + m - O\Paren{\sqrt{m n}}\le 0.2\beta n\,.
\]
Hence $\sigma_1^2 - \sigma_2^2 \ge 0.7\beta m$. Note that for each vector $x$, 
\[
m - O\Paren{\sqrt{mn}} \le x^\top MM^\top x \le \beta m + m - O\Paren{\sqrt{mn}}\,,
\]
Hence $\sigma_1^2 - \sigma_2^2 \le O\Paren{\beta m}$. Finally, since $\sigma_1 + \sigma_2 = \Theta\Paren{\sqrt{m}}$,
\[
\sigma_1 - \sigma_2 = \frac{\sigma_1^2 - \sigma_2^2}{\sigma_1 + \sigma_2} = \Theta(\beta\sqrt{m})\,.
\]

The bound $\coh{1}{M} \le O\Paren{\log(n+m)}$ follows from \cref{thm:coherence-gaussian}.

\end{document}